\documentclass[twoside,11pt]{article}

\usepackage{blindtext}

\usepackage{jmlr2e}

\usepackage{amsmath}
\usepackage{placeins}
\usepackage{amsfonts}
\usepackage{graphicx}
\usepackage{float}
\usepackage{enumerate}
\usepackage{natbib}
\usepackage{adjustbox}
\usepackage{mathtools}
\usepackage{algorithmic}
\usepackage{array}
\usepackage[caption=false,font=normalsize,labelfont=sf,textfont=sf]{subfig}
\usepackage{textcomp}
\usepackage{booktabs} 
\usepackage{stfloats}
\usepackage{multirow}
\usepackage{capt-of}
\usepackage{arydshln}
\usepackage{verbatim}
\usepackage{hyperref}
\hypersetup{
    colorlinks=true,
    linkcolor=black,
    filecolor=black,      
    urlcolor=black,
    citecolor=black,
}

\usepackage{xcolor}
\usepackage{url}

\usepackage{lastpage}

\ShortHeadings{Deep BNP Framework for Robust MI Estimation}{Fazeliasl, Zhang, Jiang, and Kong}
\firstpageno{1}

\begin{document}

\title{A Deep Bayesian Nonparametric Framework for Robust Mutual Information Estimation}

\author{\name Forough Fazeliasl \email fazelias@ualberta.ca \\
       \addr Department of Mathematical and Statistical Sciences\\
       University of Alberta\\
       Edmonton, Canada
       \AND
       \name Michael Minyi Zhang \email mzhang18@hku.hk \\
       \addr Department of Statistics and Actuarial Science\\
       University of Hong Kong\\
       Hong Kong
       \AND
       \name Bei Jiang \email bei1@ualberta.ca \\
       \addr Department of Mathematical and Statistical Sciences\\
       University of Alberta\\
       Edmonton, Canada
       \AND
       \name Linglong Kong  \email lkong@ualberta.ca \\
       \addr Department of Mathematical and Statistical Sciences\\
       University of Alberta\\
       Edmonton, Canada}


\maketitle

\begin{abstract}
Mutual Information (MI) is a crucial measure for capturing dependencies between variables, yet exact computation is challenging in high dimensions with intractable likelihoods, impacting accuracy and robustness. One idea is to use an auxiliary neural network to train an MI estimator; however, methods based on the empirical distribution function (EDF) can introduce sharp fluctuations in the MI loss due to poor out-of-sample performance, thereby destabilizing convergence.
We present a Bayesian nonparametric (BNP) solution for training an MI estimator by constructing the MI loss with a finite representation of the Dirichlet process posterior to incorporate regularization in the training process. With this regularization, the MI loss integrates both prior knowledge and empirical data to reduce the loss sensitivity to fluctuations and outliers in the sample data, particularly in small sample settings like mini-batches. This approach addresses the challenge of balancing accuracy and low variance by effectively reducing variance, leading to stabilized and robust MI loss gradients during training and enhancing the convergence of the MI approximation while offering stronger theoretical guarantees for convergence. We explore the application of our estimator in maximizing MI between the data space and the latent space of a variational autoencoder. Experimental results demonstrate significant improvements in convergence over EDF-based methods, with applications across synthetic and real datasets, notably in 3D CT image generation—yielding enhanced structure discovery and reduced overfitting in data synthesis. While this paper focuses on generative models in application, the proposed estimator is not restricted to this setting and can be applied more broadly in various BNP learning procedures.

\end{abstract}


\section{Introduction}
\label{sec:intro}
Mutual information inference is a statistical measure used
to quantify the dependency between random variables with wide applications in machine learning. It plays a crucial role in tasks such as feature selection, representation learning, clustering, and generative modeling. By capturing both linear and nonlinear dependencies, mutual information serves as a versatile metric in understanding and leveraging relationships in complex datasets.  However, computing MI in these applications faces significant challenges, particularly in the presence of high-dimensional or complex data structures. Mutual information is defined using the joint distribution of two random variables and the product of their marginal distributions, making its estimation especially difficult when these distributions are analytically intractable. 
While neural network-based estimators have shown promise in addressing some of these issues \citep{belghazi2018mutual,hjelm2018learning}, they remain prone to instability and sensitivity to sample variability and outliers, highlighting the need for more flexible and reliable approaches for mutual information estimation in machine learning tasks.

Recently, Bayesian nonparametric (BNP) techniques have proven effective in introducing uncertainty into modeling data distributions with minimal assumptions, leading to enhanced flexibility, robustness, and reduced sensitivity to sample variability in the machine learning community \citep{fazeli2023semi,dellaporta2022robust,fazeli2023bayesian,bariletto2024bayesian}. However, BNP approaches for estimating mutual information have received comparatively little attention. 
In this paper, we introduce a Dirichlet process (DP) prior on the data distribution and propose a deep BNP mutual information neural estimation (DPMINE) between two variables. This estimator serves as the Bayesian counterpart to the Kullback-Leibler (KL)-based and Jensen-Shannon (JS)-based MINE methods introduced in \cite{belghazi2018mutual} and \cite{hjelm2018learning}, respectively. The MINE is computed by training an auxiliary function, parameterized by a neural network, that approximates the intractable likelihoods of the joint distribution of the two variables and the product of their marginal distributions. 

Our DP-based approach in the proposed estimator employs a finite DP posterior representation in the loss function of the auxiliary function, making our estimator more robust to sample variability. It introduces distributional regularization through the incorporation of a prior distribution, thereby improving the efficiency and stability of the optimization process. This approach improves estimation stability by enhancing out-of-sample performance, addressing the curse of dimensionality and sensitivity to sample size, and providing robustness guarantees in high-dimensional settings. This makes it well-suited for integration into deep learning models, where a procedure with faster and robust convergence is needed for training the relevant cost functions regularized by this estimator.
In contrast, frequentist methods, which rely on the empirical distribution, 
may exhibit greater sensitivity to sample variability, potentially leading to less stable gradient computations during training, particularly with small batch sizes
or when dealing with highly heterogeneous data. This sensitivity can be interpreted as introducing variance in the estimation process. We present experiments comparing our DP approach to EDF-based methods (frequentist approaches) in mutual information estimation, demonstrating improved performance. Furthermore, our experiments highlight the robustness of the DPMINE in high-dimensional settings, showing better convergence and reduced estimation fluctuations.
In our results, we demonstrate that the DPMINE is theoretically stronger than MINE counterpart, particularly for the KL-based estimator. 

To further motivate the need for our KL-based Bayesian counterpart, we briefly highlight the limitations of existing MI estimators. JS-based estimators ensure stable convergence but often lack the precision needed to capture complex dependencies, while traditional KL-based estimators provide greater accuracy but suffer from poor convergence, limiting their practical use \citep{song2020understanding,poole2019variational}. Our KL-based Bayesian counterpart addresses these issues by combining the accuracy of KL with improved convergence properties. While its broader impact on tasks such as information bottleneck methods is beyond the scope of this paper, we focus specifically on its effect in improving the performance of a particular family of generative models where representation learning is essential.

We apply our proposed DPMINE to enhance the performance of a BNP VAE-GAN model, a fusion of Generative Adversarial Networks (GANs) and Variational Autoencoders (VAEs) \citep{fazeli2023bayesian}. 
A VAE-GAN model effectively addresses the persistent challenges of mode collapse in GANs and the generation of blurry images in VAEs, striking a balance between diversity and quality in the generated data \citep{larsen2016autoencoding, donahue2016adversarial, dumoulin2016adversarially, rosca2017variational, kwon2019generation}. 
This integration combines the high-quality outputs that GANs offer and the diversity of representations that VAEs offer, resulting in a superior generative model. However, to fully unlock the potential of VAE-GAN models, further exploration of the information encoded in the latent space is needed \citep{belghazi2018mutual,hjelm2018learning}. We aim to maximize the information between original and encoded samples, as well as between generated and encoded samples using our DPMINE.  
This incorporation of DPMINE in BNP generative modeling improves representation in the code space of the generative model and also helps preserve the information in the code space for generating new samples. 
Although this paper focuses on DPMINE's application in generative models, we emphasize that its use is not limited to such models and opens new opportunities for researchers interested in applying this estimator to various optimization problems within the BNP learning (BNPL) framework \citep{lyddon2018nonparametric,lyddon2019general,fong2019scalable}.

The paper is structured as follows: Section \ref{sec:background} reviews the basic concepts in mutual information inference, as well as its applications in generative modeling.
In Section \ref{sec:DPMINE}, we present the fundamental concepts of the DP and introduce the DPMINE, along with theoretical results on the properties of our estimator.  
We then demonstrate the integration of DPMINE into the 3D convolutional architecture of a BNP VAE-GAN model.
Section \ref{sec:experiment} explores the impact of the DPMINE, first demonstrating its improved performance over EDF-based methods in mutual information estimation, with greater robustness in high-dimensional settings. We then apply the DPMINE to train the VAE-GAN model through various experiments, showcasing its ability to mitigate mode collapse. In particular, we evaluate the proposed procedure using a collection of 3-dimensional (3D) chest CT scans of COVID-19-infected patients.
Finally, we conclude the paper in Section \ref{sec:conclusion},
where we discuss potential future research directions. All technical proofs, additional experimental details, limitation broad impacts, and safeguards are provided in the Appendix.

\section{Fundamentals of Mutual Information Inference and Applications in Generative Modeling}\label{sec:background}

In this section, we introduce the background information on MI and its applications in generative modeling.

 
\subsection{Mutual information}
Let $(\mathbf{X},\mathbf{Y})$ be a pair of random variables with a joint cumulative distribution function (CDF) $F_{\mathbf{X},\mathbf{Y}}$ and marginal CDFs $F_{\mathbf{X}}$ and $F_{\mathbf{Y}}$, defined over the product space $\mathcal{X}\otimes\mathcal{Y}$. The mutual information between $\mathbf{X}$ and $\mathbf{Y}$ is given by:
\begin{align}
    \text{MI}(\mathbf{X},\mathbf{Y})
    &=\text{D}_{\text{KL}}(F_{\mathbf{X}\otimes \mathbf{Y}},F_{\mathbf{X}}\otimes F_{\mathbf{Y}})=H(\mathbf{X})-H(\mathbf{X}|\mathbf{Y})\label{mi-eq2},
\end{align}
where $\text{D}_{\text{KL}}(\cdot,\cdot)$ denotes the KL divergence, and $H(\cdot)$ and $H(\cdot|\cdot)$ represent marginal and conditional Shannon entropies, respectively.


Computing Eq.~\eqref{mi-eq2} can be challenging in many practical applications due to the intractable likelihood of the relevant densities. There are approximate techniques to calculate mutual information. 
However, many of them suffer from the curse of dimensionality, like in examples where $k$-nearest neighbors are used to estimate MI. \citep{al2022test,berrett2019nonparametric}.
Moreover, optimizing $k$ to balance the bias-variance trade-off adds to the complexity \citep{sugiyama2012machine}.

An effective approach for overcoming these obstacles involves
the establishment of a variational lower bound (VLB) approximation of the MI through the MINE. For instance, \cite{belghazi2018mutual} uses the Donsker-Varadhan (DV) representation \citep{donsker1983asymptotic} of the KL divergence to form a lower bound of \eqref{mi-eq2} as:
\begin{align}\label{lower-DV}
    \mathcal{L}_{\boldsymbol{\gamma}}^{\text{DV}}(\mathbf{X},\mathbf{Y})=
    \mathbb{E}_{F_{\mathbf{XY}}}[T_{\boldsymbol{\gamma}}(\mathbf{X},\mathbf{Y})]-\ln \mathbb{E}_{F_{\mathbf{X}}\otimes F_{\mathbf{Y}}}[e^{T_{\boldsymbol{\gamma}}(\mathbf{X},\mathbf{Y})}],
\end{align}
where $\{T_{\boldsymbol{\gamma}}\}_{\boldsymbol{\gamma}\in\boldsymbol{\Gamma}}$ be a set of continuous functions parameterized by a neural network on a compact domain $\Gamma$ that maps $\mathcal{X}\otimes\mathcal{Y}$ to $\mathbb{R}$ and $\ln(\cdot)$ denotes the natural logarithm. 
The DV-MINE is then defined as $\mathcal{L}_{\widehat{\boldsymbol{\gamma}}}^{\text{DV}}(\mathbf{X},\mathbf{Y})$, where $\widehat{\boldsymbol{\gamma}}$ is the optimal parameter that maximizes the VLB \eqref{lower-DV}. 
$T_{\boldsymbol{\gamma}}$, which is an auxiliary statistic, serves as the fundamental component in this procedure, which learns to differentiate between a sample from $F_{\mathbf{XY}}$ and $F_{\mathbf{X}}\otimes F_{\mathbf{Y}}$. In practice, $\mathcal{L}_{\boldsymbol{\gamma}}^{\text{DV}}(\mathbf{X},\mathbf{Y})$ in Eq.~\eqref{lower-DV} is approximated using the empirical cumulative distribution functions  (ECDFs), $F_{\mathbf{X}_{1:n}}:=\frac{1}{n}\sum_{i=1}^{n}\delta_{\mathbf{X}_{i}}$ and $F_{\mathbf{Y}_{1:n}}:=\frac{1}{n}\sum_{i=1}^{n}\delta_{\mathbf{Y}_{i}}$:
\begin{align}\label{DV-ECDF}
\mathcal{L}_{\boldsymbol{\gamma}}^{\text{DV}}(\mathbf{X}_{1:n},\mathbf{Y}_{1:n})=\sum_{\ell=1}^{n}\dfrac{1}{n}T_{\boldsymbol{\gamma}}(\mathbf{X}_{\ell},\mathbf{Y}_{\ell})
-\ln\sum_{\ell=1}^{n}\dfrac{1}{n}e^{T_{\boldsymbol{\gamma}}(\mathbf{X}_{\ell},\mathbf{Y}_{\pi(\ell)})},
\end{align}
where $\lbrace \pi(1:n) \rbrace$ denotes a random permutation of $\lbrace 1:n\rbrace$, a standard technique to empirically approximate $\mathbb{E}_{F_{\mathbf{X}}\otimes F_{\mathbf{Y}}}(\cdot)$ in Eq.~\eqref{lower-DV} \citep{belghazi2018mutual, hjelm2018learning}.

 \cite{belghazi2018mutual} also considered a divergence representation given in \cite{nguyen2010estimating} to form another lower bound for Eq.~\eqref{mi-eq2}. However, they discovered that Eq.~\eqref{lower-DV} yields a tighter bound than this alternative lower bound and demonstrated numerically that DV-MINE outperforms MINE based on the alternative lower bound. Another MINE, based on providing a lower bound for the JS divergence \citep{nowozin2016f} between $F_{\mathbf{XY}}$ and $F_{\mathbf{X}}\otimes F_{\mathbf{Y}}$, can also be found in the literature \citep{hjelm2018learning,jones2023information}. It is defined as
\begin{small}
\begin{align*}
    \mathcal{L}_{\boldsymbol{\gamma}}^{\text{JS}}(\mathbf{X},\mathbf{Y})&=
    \mathbb{E}_{F_{\mathbf{XY}}}[-\zeta(-T_{\boldsymbol{\gamma}}(\mathbf{X},\mathbf{Y}))]- \mathbb{E}_{F_{\mathbf{X}}\otimes F_{\mathbf{Y}}}[\zeta(T_{\boldsymbol{\gamma}}(\mathbf{X},\mathbf{Y}))],
\end{align*}
\end{small}
where $\zeta(\cdot)=\ln(1+\exp{(\cdot)})$ denotes the softplus function.
However, \cite{jones2023information} remarked that JS-MINE may be more appropriate for problems involving the maximization of MI, rather than a precise approximation.

A detailed exploration and comparative assessment of different variational bounds for MI is provided in \cite{poole2019variational,tsai2021self}, with much of this work driven by the goal of optimizing MI in deep learning problems \citep{belghazi2018mutual,hjelm2018learning,barber2004algorithm,lin2022mutual,chen2016infogan}. However, recent studies have shown that existing neural MI estimators, whether JS-based (providing stable but imprecise estimates) or KL-based (offering accuracy but poor convergence), struggle to provide accurate and low-variance estimates when the true MI is high, highlighting that stability alone in JS-based estimators and accuracy alone in KL-based estimators are insufficient for effectively capturing information \citep{song2020understanding}. This limitation can impair the regularization of the objective function in deep learning models, including the encoding-decoding process in VAE-based models. Attempts to stabilize KL-based estimators through various modifications of the DV lower bound to reduce variance while preserving accuracy have been proposed \citep{mroueh2021improved,Wen2020Mutual,guo2022tight}. However, these modifications often increase computational complexity and make the training process more difficult to control.

\subsection{Interactions between generative models and mutual information}
Generative models have been extensively explored in data augmentation and synthesis \citep{mendes2023lung,menon2023ccs,toda2021synthetic,wolterink2017generative,bu20213d}. 
However, these models are expanded beyond the standard GAN loss function \citep{Goodfellow}, which could potentially introduce challenges such as mode collapse--memorizing certain modes of data distribution while overlooking other diversities--and training instability. The standard GAN is constructed based on two sets of functions: \(\{G_{\boldsymbol{\omega}}\}_{\boldsymbol{\omega}\in\boldsymbol{\Omega}}\) and $\{D_{\boldsymbol{\theta}}\}_{\boldsymbol{\theta}\in\boldsymbol{\Theta}}$, which are parameterized by neural networks and referred to as the generator and discriminator, respectively \citep{Goodfellow}. The generator engages in a game with the discriminator, aiming to deceive it into being unable to distinguish between generated samples and real samples. It is achieved by updating parameters as follows:
\begin{small}
\begin{align*}
    (\widehat{\boldsymbol{\omega}},\widehat{\boldsymbol{\theta}})&=\arg\min\limits_{\boldsymbol{\Omega}}\max\limits_{\boldsymbol{\Theta}} \mathbb{E}_{F_{\boldsymbol{\xi}}}[\ln(1-D_{\boldsymbol{\theta}}(G_{\boldsymbol{\omega}}(\boldsymbol{\xi})))]+\mathbb{E}_{F}[\ln(D_{\boldsymbol{\theta}}(\mathbf{X}))].
\end{align*}
\end{small}
Here, $\mathbf{X}\in\mathbb{R}^{d}\sim F$ represents a real dataset, and $\boldsymbol{\xi}\in\mathbb{R}^p\sim F_{\boldsymbol{\xi}}$ represents a noise vector, where $d>p$. 
There are various approaches to address the mode collapse issue, 
including modifying the loss function \citep{arjovsky2017wasserstein,nowozin2016f,Li,dellaporta2022robust,fazeli2023semi} and exploring different architectures \citep{radford2015unsupervised,zhang2019self}. 

In contrast, an effective strategy is to 
employ VAE-GAN models \citep{larsen2016autoencoding}. 
The VAE-GAN employs an encoder network in the VAE component to learn a probabilistic 
encoding of $\mathbf{X}$ into a compact 
representation $\boldsymbol{c}$. 
Subsequently, a decoder is used to generate realistic data from $\boldsymbol{c}$ \citep{kingma2013auto}. The probabilistic approach involves modeling the variational distribution $F_{E_{\boldsymbol{\eta}}}(\boldsymbol{c}|\mathbf{X})$ using the encoder network $\{E_{\boldsymbol{\eta}}\}_{\boldsymbol{\eta}\in\boldsymbol{\mathcal{H}}}$, which approximates the true posterior distribution of $\boldsymbol{c}$ given $\mathbf{X}$. The VAE aims to minimize the error in reconstructing the original input from the code space, as well as the regularization error--the difference between $F_{E_{\boldsymbol{\eta}}}(\boldsymbol{c}|\mathbf{X})$ and $F_{\mathbf{z}}$.

The traditional VAE-GAN model comprises of two interconnected networks $E_{\boldsymbol{\eta}}$, $G_{\boldsymbol{\omega}}$, and $D_{\boldsymbol{\theta}}$. The VAE is connected to the GAN through $G_{\boldsymbol{\omega}}$, which acts as a decoder for the VAE. The generator is fed with $(\boldsymbol{\xi},\boldsymbol{c})$ to generate samples that aim to fool $D_{\boldsymbol{\theta}}$ into distinguishing them from $\mathbf{X}$.  Some notable works that employ this strategy include the bi-directional GAN (BiGAN) \citep{donahue2016adversarial,dumoulin2016adversarially}, the $\alpha$-GAN \citep{rosca2017variational}, the $\alpha$-WGAN+Gradient penalty ($\alpha$-WGAN+GP) \citep{kwon2019generation}, and the $\mathrm{BNPWMMD}$-GAN \citep{fazeli2023bayesian}. 
The main differences between these models lie in their loss functions and network architectures.

Particularly, the $\alpha$-WGAN+GP serves as the primary frequentist nonparametric (FNP) counterpart to the BNPWMMD, though BNPWMMD has recently proven to be a stronger competitor. It combines $\alpha$-GAN with WGAN+GP to improve training stability in generating 3D brain MRI data, including sagittal, coronal, and axial representations.  
\cite{jafari2023improved} enhanced the $\alpha$-WGAN+GP to generate connected 3D volumes, while  \cite{ferreira2022generation} proposed another enhanced model for generating 3D rat brain MRI using convolutional data manipulation techniques including zoom, rotation, Gaussian noise, flip, translation and scaling intensity. 
However, the model may learn to expect these transformations, resulting in generated data that reflects the augmented relationships rather than true variability (correlated outputs).
Additionally, the provided results did not show a significant visual improvement. 


Alongside advancements in VAE-GAN models for generating samples, it is essential to explore the information encoded in the latent space to fully unlock the potential of these models.
Maximizing the MI between the input and output of the encoder can significantly reduce the reconstruction error in VAE-based models \citep{belghazi2018mutual, hjelm2018learning}. This approach encourages the encoder to retain important details and structure in the input data, ultimately leading to a more informative code representation for the decoder to use. \cite{belghazi2018mutual} incorporated $\max\limits_{\boldsymbol{\Gamma},\boldsymbol{\mathcal{H}}}\mathcal{L^{\text{DV}}_{\boldsymbol{\gamma}}}(\mathbf{X},E_{\boldsymbol{\eta}}(\mathbf{X}))$ into the loss function of a BiGAN model to estimate and maximize the 
MI between $\mathbf{X}$ and $E_{\boldsymbol{\eta}}(\mathbf{X})$. This approach, referred to as BiGAN+MINE model, demonstrated stronger coverage on the training dataset than the original BiGAN, displaying the ability of the MMI procedure to reduce mode collapse.

Moreover, preserving the information contained in codes during the generation process is extremely important. 
For instance,  \cite{chen2016infogan} attempted to maximize the MI between code $E_{\boldsymbol{\eta}}(\mathbf{X})$ and generated sample $G_{\boldsymbol{\omega}}(\boldsymbol{\xi},\boldsymbol{E_{\boldsymbol{\eta}}(\mathbf{X})})$ by providing another VLB for standard GANs under appropriate regularity conditions. However, no results were provided to declare the effectiveness of their method in mitigating mode collapse. 
Alternatively, \cite{belghazi2018mutual} employed
DV representation to maximize 
$MI(E_{\boldsymbol{\eta}}(\mathbf{X}),G_{\boldsymbol{\omega}}(\boldsymbol{\xi},E_{\boldsymbol{\eta}}(\mathbf{X})))$ for standard GANs
to demonstrate the effectiveness of the MINE-based method in covering the training data for this scenario. 


\section{A DP-based accurate-stable MI estimator and its application in generative model regularization}\label{sec:DPMINE}
Our proposed method of deep mutual information estimation partially relies on the DP as a technique to enhance the robustness of the training process, which we will first introduce here. The proposed approach is easy to implement and, particularly for KL-based estimators, provides a tighter lower bound, ultimately improving the stability of the estimator, which, as demonstrated in the numerical results section, leads to variance reduction.
\subsection{Dirichlet process}\label{sec:DP}
The DP is an infinite generalization of the Dirichlet distribution 
that is considered on the sample space denoted as $\mathfrak{X}$, which possesses a $\sigma$-algebra $\mathcal{A}$ comprising subsets of $\mathfrak{X}$ \citep{Ferguson}. 
$F$ follows a DP with parameters $(a, H)$ with the notation $F^{\text{Pri}}:=(F \sim DP(a, H))$, if for any measurable partition $A_{1}, \ldots, A_{k}$ of $\mathfrak{X}$ with $k \geq 2$, the joint distribution of the vector $(F(A_{1}), \ldots, F(A_{k}))$ follows a Dirichlet distribution characterized by parameters $(aH(A_{1}), \ldots, aH(A_{k}))$. Moreover, it is assumed that $H(A_{j})=0$ implies $F(A_{j})=0$ with probability one. The base measure $H$ captures the prior knowledge regarding the data distribution, while $a$ signifies the strength or intensity of this knowledge. 

As a conjugate prior, the posterior distribution of $F$ also follows a DP, denoted by $F^{\text{Pos}}:=(F|\mathbf{X}_{1:n}\sim \text{DP}(a+n,H^{\ast}))$, for $n$ independent and identically distributed (IID) draws, 
$\left(\mathbf{X}_{1:n}\in\mathbb{R}^d\right)$, from the random probability measure $F$ where $H^{\ast}=a(a+n)^{-1}H+n(a+n)^{-1}F_{\mathbf{X}_{1:n}}$, and $F_{\mathbf{X}_{1:n}}$ represents the empirical cumulative distribution function of the sample $\mathbf{X}_{1:n}$. 

Although the stick-breaking representation is a commonly employed series representation for DP inference \citep{sethuraman1994constructive}, it lacks the necessary normalization terms to convert it into a probability measure \citep{zarepour2012rapid}. Additionally, simulating from an infinite series is only feasible through using a random truncation approach to handle the terms within the series. To address these limitations, \cite{Ishwaran} introduced an approximation of the DP in the form of a finite series, which allows for convenient simulation. In the context of posterior inference, this approximation is given by
\begin{align}\label{approx of DP}
F^{\text{Pos}}_{N}:=\sum_{i=1}^{N}J^{\text{Pos}}_{i,N}\delta_{\mathbf{X}^{\text{Pos}}_{i}},
\end{align}
where $\left(J^{\text{Pos}}_{1:N,N}\right)\sim \mbox{Dirichlet}((a+n)/N,\ldots,(a+n)/N)$, $\left(\mathbf{X}^{\text{Pos}}_{1:N}\right)\overset{\text{IID}}{\sim}H^{\ast}$, and $\delta_{\mathbf{X}^{\text{Pos}}}$ is the Dirac delta measure. In this study, the variables $J^{\text{Pos}}_{i,N}$ and $\mathbf{X}^{\text{Pos}}_{i}$ represent the DP's weight and location, respectively. The sequence $(F^{\text{Pos}}_{N})_{N\geq 1}$ converges in distribution to $F^{\text{Pos}}$, where $F^{\text{Pos}}_{N}$ and $F^{\text{Pos}}$ are random values in $M_{1}(\mathbb{R}^d)$, the space of probability measures on $\mathbb{R}^d$ endowed with the topology of weak convergence \citep{Ishwaran}. 
In the subsequent sections, we 
investigate the efficacy of this approximation within a regularization method in a BNP generative model.
\begin{remark}\label{remark1}
    The prior $H$ in the DP introduces a regularizing effect by smoothing the resulting distribution. Unlike the empirical distribution, which represents the data exactly as observed, the DP with a prior creates a distribution that accounts for both observed data and the prior knowledge. This makes the model less sensitive to fluctuations and variability in the sample data, which can be particularly beneficial for optimization processes in models that rely on loss functions. It does so by anchoring the sampled distribution around $H$, especially when data is sparse or when there are regions of high variability. Choosing an effective prior $H$ for regularization depends on the specific research context and the statistician’s knowledge of the subject area. For instance, in genetics, $H$ might reflect biological baselines, while in financial modeling, a prior could incorporate historical market trends. Each choice aligns with the research goals, capturing nuances relevant to the field. A commonly effective choice for general purposes is a multivariate normal distribution with parameters reflecting the sample's mean and covariance. This prior helps anchor the distribution near expected values, but with flexibility that accounts for sample variability, thus providing a stabilizing regularization effect across diverse data points.
\end{remark}

\begin{remark}\label{remark2}
    The DP framework balances between adhering to $H$ and adapting to the data through concentration parameter $a$.
This balance prevents overfitting to the specific data sample and provides a more stable and generalized distribution, especially in small sample scenarios. The DP framework strikes a balance between adhering to \( H \) and adapting to the data through the concentration parameter \( a \). This balance mitigates overfitting to specific data samples and ensures a more stable and generalized distribution, particularly in small sample scenarios. In this paper, we follow \cite{fazeli2023bayesian} to employ a Maximum A Posteriori (MAP) estimate for selecting the optimal value of the concentration parameter \( a \). This is achieved by maximizing the log-likelihood of \( F^{\mathrm{pos}} \), fitted to the given dataset, over a range of \( a \) values.
\end{remark}

\begin{remark}\label{remark3}
    In approximating the DP in Eq.~\eqref{approx of DP}, it is crucial to determine an optimal number of terms to ensure that the truncated series remains a close representation of the true Dirichlet Process. We address this requirement by applying a random stopping rule as outlined in \cite{zarepour2012rapid}. This adaptive rule terminates the approximation when the marginal contribution of the current term falls below a specified significance level relative to the cumulative sum of previous terms. Given a predefined threshold $\epsilon \in (0,1)$, the stopping criterion is defined by:
$N = \inf\left\lbrace j :, \frac{\Gamma_{j,j}}{\sum_{i=1}^{j}\Gamma_{i,j}} < \epsilon\right\rbrace. $
This approach effectively balances computational efficiency with approximation accuracy.
\end{remark}
 
\subsection{DPMINE: The Dirichlet process mutual information neural estimator}\label{subsec:DPMINE rep}

In this section, the DP approximation \eqref{approx of DP} is leveraged to develop two lower bounds for MI, leading to the formulation of DPMINE.
\paragraph{KL-based representation:}
Let $\left(\mathbf{X}_{1:n} \right)$ be $n$ IID  random variables for $\mathbf{X}\in \mathbb{R}^d\sim F$. 
For continuous functions $f_i: \mathbb{R}^{d}\rightarrow \mathbb{R}^{d^{\prime}}$, $i=1,2$, let  $\mathbf{X}^{\prime}_{i}=f_i(\mathbf{X})$.
Then, 
for a fixed value of $a$ and a chosen probability measure $H$
used in the DP approximation \eqref{approx of DP}, 
we propose to estimate the DV lower bound of $\text{MI}(\mathbf{X}_{1}^{\prime},\mathbf{X}_{2}^{\prime})$ as:
\begin{small}
\begin{multline}\label{DPDV-lower}
    \mathcal{L}_{\boldsymbol{\gamma}}^{\text{DPDV}}(f_1(\mathbf{X}^{\text{Pos}}_{1:N}),f_2(\mathbf{X}^{\text{Pos}}_{1:N})):=\sum_{\ell=1}^{N}J_{\ell,N}^{\text{Pos}}T_{\boldsymbol{\gamma}}(f_1(\mathbf{X}^{\text{Pos}}_{\ell}),f_2(\mathbf{X}^{\text{Pos}}_{\ell}))-\ln\sum_{\ell=1}^{N}J_{\ell,N}^{\text{Pos}}e^{T_{\boldsymbol{\gamma}}(f_1(\mathbf{X}^{\text{Pos}}_{\ell}),f_2(\mathbf{X}^{\text{Pos}}_{\pi(\ell)}))}.
\end{multline}
\end{small}
Here, $\lbrace\pi(1:N)\rbrace$ denotes a random permutation of $\lbrace 1:N\rbrace$, which is a common technique used to empirically approximate $\mathbb{E}_{F_{\mathbf{X}}\otimes F_{\mathbf{Y}}}(\cdot)$ in \eqref{lower-DV} \citep{belghazi2018mutual,hjelm2018learning}.

\paragraph{JS-based representation:} The JS-based estimation is similarly proposed as:
\begin{small}
\begin{align}\label{DPJS}
    &\mathcal{L}_{\boldsymbol{\gamma}}^{\text{DPJS}}(f_1(\mathbf{X}^{\text{Pos}}_{1:N}),f_2(\mathbf{X}^{\text{Pos}}_{1:N})):=\sum_{\ell=1}^{N}J_{\ell,N}^{\text{Pos}}\Big[-\zeta(-T_{\boldsymbol{\gamma}}(f_1(\mathbf{X}^{\text{Pos}}_{\ell}),f_2(\mathbf{X}^{\text{Pos}}_{\ell})))-\zeta(T_{\boldsymbol{\gamma}}(f_1(\mathbf{X}^{\text{Pos}}_{\ell}),f_2(\mathbf{X}^{\text{Pos}}_{\pi(\ell)})))
    \Big].
\end{align}
\end{small}
 Let $\text{i}$ be a label in $\lbrace \text{DV}, \text{JS}\rbrace$ indicating the type of the VLB. The BNP MINEs are consequently considered as: 
 \begin{small}
\begin{align}\label{DP-MINE}
\text{MI}^{\text{DPi}}(f_1(\mathbf{X}^{\text{Pos}}_{1:N}),f_2(\mathbf{X}^{\text{Pos}}_{1:N}))&=\max\limits_{\boldsymbol{\gamma}\in\boldsymbol{\Gamma}}\mathcal{L}_{\boldsymbol{\boldsymbol{\gamma}}}^{\text{DPi}}(f_1(\mathbf{X}^{\text{Pos}}_{1:N}),f_2(\mathbf{X}^{\text{Pos}}_{1:N}))
\end{align}
\end{small}
The following theorem states that, at least in the case of KL-based MINE, $\mathcal{L}_{\boldsymbol{\gamma}}^{\text{DPDV}}(\mathbf{X}_{1}^{\text{Pos}^{\prime}},\mathbf{X}_{2}^{\text{Pos}^{\prime}})$ is asymptotically larger than $\mathcal{L}_{\boldsymbol{\gamma}}^{\text{DV}}(\mathbf{X}_{1}^{\prime},\mathbf{X}_{2}^{\prime})$ on average. This results in a tighter VLB on the true MI and allows for a more accurate estimation of MI by maximizing the tighter bound over a compact domain $\boldsymbol{\Gamma}$.
\begin{theorem}[Limiting expectation]\label{thm-asmp-dpdv}
    Considering DP posterior representations defined in \eqref{DPDV-lower} and \eqref{DPJS}. Given the DP posterior approximation in \eqref{approx of DP}, we have,  
    \begin{itemize}
        \item[$i.$] 
        $\lim_{n,N \to \infty}\mathbb{E}_{F_{N}^{\text{Pos}}}\left(\mathcal{L}_{\boldsymbol{\gamma}}^{\text{DPDV}}(f_1(\mathbf{X}^{\text{Pos}}_{1:N}),f_2(\mathbf{X}^{\text{Pos}}_{1:N})) \right) \geq \mathcal{L}_{\boldsymbol{\gamma}}^{\text{DV}}(\mathbf{X}_{1}^{\prime},\mathbf{X}_{2}^{\prime})$, a.s.,
        \item[$ii.$] $\mathbb{E}_{F_{N}^{\text{Pos}}}\left(\mathcal{L}_{\boldsymbol{\gamma}}^{\text{DPJS}}(f_1(\mathbf{X}^{\text{Pos}}_{1:N}),f_2(\mathbf{X}^{\text{Pos}}_{1:N})) \right)$ converges a.s. to $\mathcal{L}_{\boldsymbol{\gamma}}^{\text{JS}}(\mathbf{X}_{1}^{\prime},\mathbf{X}_{2}^{\prime})$, as $n,N\rightarrow\infty,$
    \end{itemize}
    where ``a.s.'' stands for ``almost surely'', denoting that the statements hold with probability 1.
\end{theorem}

\noindent The strong consistency of the proposed estimators is also investigated through the next theorem.
\begin{theorem}[Consistency]\label{thmConsistency} 
    Considering BNP MINEs given in \eqref{DP-MINE}. Then,  for any label $\text{i}$ in $\lbrace \text{DV}, \text{JS}\rbrace$, as $n,N\rightarrow\infty$:
    \begin{itemize}
        \item[$i.$] $\text{MI}^{\text{DPi}}(f_1(\mathbf{X}^{\text{Pos}}_{1:N}),f_2(\mathbf{X}^{\text{Pos}}_{1:N}))\xrightarrow{a.s.}\text{MI}^{\text{i}}(\mathbf{X}_{1}^{\prime},\mathbf{X}_{2}^{\prime})$,
        \item[$ii.$] There exists a set of neural network $\{T_{\boldsymbol{\gamma}}\}_{\boldsymbol{\gamma}\in\boldsymbol{\Gamma}}$ on some compact domain $\boldsymbol{\Gamma}$ such that
        \begin{align*}
            \text{MI}^{\text{DPi}}(f_1(\mathbf{X}^{\text{Pos}}_{1:N}),f_2(\mathbf{X}^{\text{Pos}}_{1:N}))\xrightarrow{a.s.}\text{MI}(\mathbf{X}_{1}^{\prime},\mathbf{X}_{2}^{\prime}).
        \end{align*}
    \end{itemize}
\end{theorem}

\subsection{Embedding DPMINE in generative models}
The BNPWMMD-GAN  \citep{fazeli2023bayesian} is a VAE-GAN model that places a DP prior on the data distribution. 
This approach helps prevent overfitting, while also leveraging the advantages of GANs and VAEs in data augmentation. In this study, we aim to refine the BNPWMMD-GAN by using the proposed BNP MINE to further improve its performance. 
\begin{figure*}[!t]
\centering
\includegraphics[width=1\textwidth,height=4cm]{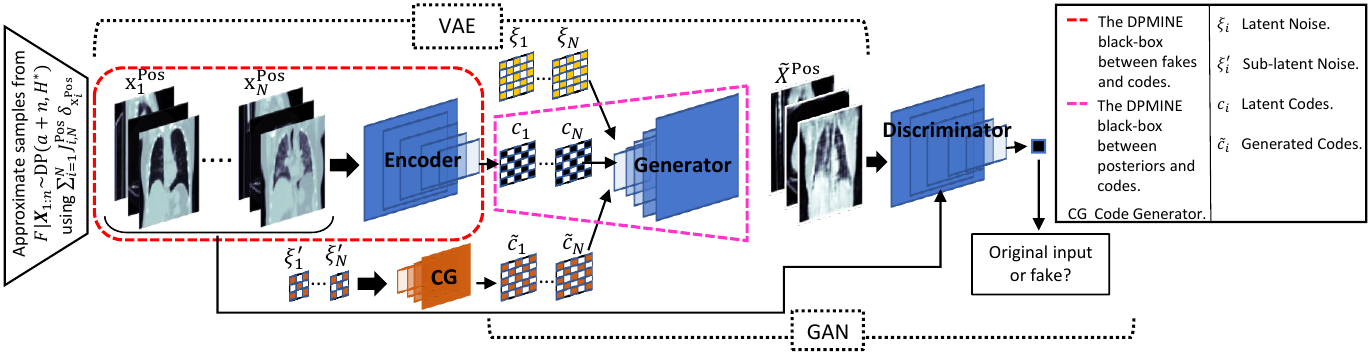}
\caption{A general diagram of the BNPWMMD model refined by DPMINE in generating 3D images.}
\label{BNPWMMD-diagram}
\end{figure*}
The BNPWMMD-GAN architecture features four networks: The encoder--$\{E_{\boldsymbol{\eta}}\}_{\boldsymbol{\eta}\in\boldsymbol{\mathcal{H}}}$, the generator--$\{G_{\boldsymbol{\omega}}\}_{\boldsymbol{\omega}\in\boldsymbol{\Omega}}$, the code generator--$\{CG_{\boldsymbol{\omega}^{\prime}}\}_{\boldsymbol{\omega}^{\prime}\in\boldsymbol{\Omega}^{\prime}}$, and the discriminator--$\{D_{\boldsymbol{\theta}}\}_{\boldsymbol{\theta}\in\boldsymbol{\Theta}}$. The addition of  
the code generator network sets the BNPWMMD-GAN apart from traditional VAE-GANs. It takes noise from a sub-latent space--the latent space of the code generator--
and generates code samples adversarially by using encoder outputs as its training dataset, as described in Eq. \eqref{BNPWMMD:CG}. The code generator is responsible for exploring uncharted regions of the code space that may have been overlooked by the encoder, thereby enhancing the overall diversity of the generated data. 

Given the codes $\left(\boldsymbol{c}_{1:N}:= E{\boldsymbol{\eta}}\left( \mathbf{X}^{\text{Pos}}_{1:N}\right)\right)$ and generated codes $\left(\widetilde{\boldsymbol{c}}_{1:N}:= CG_{\boldsymbol{\omega}^{\prime}}\left( \boldsymbol{\xi}^{\prime}_{1:N}\right)\right)$ along with the latent noise $\left(\boldsymbol{\xi}_{1:N} \in \mathbb{R}^p\right) \sim F_{\boldsymbol{\xi}}$ and sub-latent noise $\left(\boldsymbol{\xi}^{\prime}_{1:N} \in \mathbb{R}^q\right) \sim F_{\boldsymbol{\xi}^{\prime}}$, where $q < p$, the BNPWMMD-GAN is trained by updating the networks' parameters according to the following hybrid objective function\footnote{Each $F$ indexed by a sample vector in \eqref{BNPWMMD:loss} indicates the corresponding empirical distribution.}: 
\begin{subequations}\label{BNPWMMD:loss}
\begin{align}
&(\widehat{\boldsymbol{\omega}},\widehat{\boldsymbol{\eta}})\footnotemark =
\arg\min\limits_{\boldsymbol{\boldsymbol{\Omega},\boldsymbol{\mathcal{H}}}}\bigg\lbrace\clipbox{-2 0 327 0}{$\underbrace{
\begin{small}
-\frac{1}{N}\sum_{i=1}^{N}\big[ D_{\boldsymbol{\theta}}(G_{\boldsymbol{\omega}}(\boldsymbol{\xi}_{i}))
    +D_{\boldsymbol{\theta}}(G_{\boldsymbol{\omega}}(\boldsymbol{c}_{i}))
    +D_{\boldsymbol{\theta}}(G_{\boldsymbol{\omega}}(\widetilde{\boldsymbol{c}}_{i}))\big]
    +\text{MMD}(F^{\text{Pos}}_{N},F_{G_{\boldsymbol{\omega}}(\boldsymbol{\xi}_{1:N})})\hspace{11.5cm}
\end{small}     
    }$}\nonumber\\
    &~~~~~~~~~~~~~~~~~~~\clipbox{5 0 -1 0}{$\underbrace{\hspace{.16cm}+
\begin{small}    
    \text{MMD}(F^{\text{Pos}}_{N},F_{G_{\boldsymbol{\omega}}(\widetilde{\boldsymbol{c}}_{1:N})})
    +\text{MMD}(F^{\text{Pos}}_{N},F_{G_{\boldsymbol{\omega}}(\boldsymbol{c}_{1:N})})
\end{small}     
    }_{\mathcal{I}_1(\boldsymbol{\omega},\boldsymbol{\eta})}$}
    +
\underbrace{
\begin{small}
    \text{MMD}(F_{\boldsymbol{\xi}_{1:N}},F_{\boldsymbol{c}_{1:N}})
\end{small}
}_{\mathcal{I}_2(\boldsymbol{\omega},\boldsymbol{\eta})}
    \bigg\rbrace,\label{BNPWMMD:EG}\\
    &\widehat{\boldsymbol{\theta}}=\arg\min\limits_{\boldsymbol{\boldsymbol{\Theta}}}\bigg\lbrace
    \underbrace{
\begin{small}    
    \dfrac{1}{N}\sum_{i=1}^{N}\Big[D_{\boldsymbol{\theta}}(G_{\boldsymbol{\omega}}(\boldsymbol{\xi}_{i}))
    +D_{\boldsymbol{\theta}}(G_{\boldsymbol{\omega}}(\boldsymbol{c}_{i}))
\end{small}    
\begin{small}    
    +D_{\boldsymbol{\theta}}(G_{\boldsymbol{\omega}}(\widetilde{\boldsymbol{c}}_{i}))
    -3J^{\text{Pos}}_{i,N}D_{\boldsymbol{\theta}}(\mathbf{X}_{i}^{\text{Pos}})\Big]
\end{small}    
    }_{\mathcal{J}_1(\boldsymbol{\theta})}
    +
    \underbrace{\begin{small}
        \lambda L_{\text{GP-D}}
    \end{small}}_{\mathcal{J}_2(\boldsymbol{\theta})}\bigg\rbrace,\label{BNPWMMD:D}\\
&\widehat{\boldsymbol{\omega}}^{\prime}=\arg\min\limits_{\boldsymbol{\Omega}^{\prime}}\text{MMD}(F_{\boldsymbol{c}_{1:N}},F_{\widetilde{\boldsymbol{c}}_{1:N}}).\label{BNPWMMD:CG}
\end{align}
\end{subequations}
\footnotetext{Simultaneous updates of parameters $\boldsymbol{\omega}$ and $\boldsymbol{\eta}$ are facilitated by the essential role played by the generator function, which serves as a decoder during VAE training.}
Here, the generator is fed with $(\boldsymbol{\xi}_{1:N},\boldsymbol{c}_{1:N},\widetilde{\boldsymbol{c}}_{1:N})$, and terms $\mathcal{I}_1(\boldsymbol{\omega},\boldsymbol{\eta})$ and $\mathcal{J}_1(\boldsymbol{\theta})$ refer to minimizing the combined distance given by\footnote{The details for computing the Wasserstein distance (WS) and the maximum mean discrepancy (MMD) are provided in Appendix  \ref{app:distance}.}:
\begin{align}
    d_{\text{WMMD}}(F^{\text{Pos}},F_{G_{\boldsymbol{\omega}}})=\text{WS}(F^{\text{Pos}},F_{G_{\boldsymbol{\omega}}})+\text{MMD}(F^{\text{Pos}},F_{G_{\boldsymbol{\omega}}}).\label{w-mmd}
\end{align}

The inclusion of distance measurement \eqref{w-mmd} will improve the model training outcomes as it incorporates both the overall distribution comparison and the feature-matching techniques. The idea of minimizing the distance between $F^{\text{Pos}}$ and the generator distribution $F_{G_{\boldsymbol{\omega}}}$ was initially introduced independently in the BNP learning approach outlined in \cite{dellaporta2022robust,fazeli2023semi}. This approach suggests that such minimization induces a posterior distribution on the generator's parameter space, which, in turn, leads to a posterior approximation for the network parameters. 
This BNPL strategy effectively handles model misspecification by using a non-parametric prior \citep{fong2019scalable,lyddon2018nonparametric,lyddon2019general,lee2024enhancing}, which promotes stable learning and improves generalization in generative tasks. Additionally, it mitigates sensitivity to outliers and sample variability through prior regularization, leading to more stable out-of-sample performance compared to conventional learning approaches that depend solely on the empirical data distribution for loss computation \citep{bariletto2024bayesian}.
A depiction of this strategy is provided by \citet[Figure 1]{dellaporta2022robust}.

Since the MMD measure is an $L^{2}$-norm distance, the MMD-based terms in $\mathcal{I}_1(\boldsymbol{\omega},\boldsymbol{\eta})$ serves as the posterior reconstruction errors, while $\mathcal{I}_2(\boldsymbol{\omega},\boldsymbol{\eta})$ serves as the regularization error for approximating the variational distribution $F_{E_{\boldsymbol{\eta}}}(\boldsymbol{c}_{1:N}|\mathbf{X}^{\text{Pos}}_{1:N})$. Furthermore, the gradient penalty $\mathcal{J}_2(\boldsymbol{\theta})$ with a positive coefficient $\lambda$ was used to ensure training stability by forcing the $1$-Lipschitz constraint on the discriminator \citep{gulrajani2017improved}. Specifically, this penalty is given by 
\[
L_{\mathrm{GP-D}}=\frac{1}{N}\sum_{i=1}^{N}(\|\nabla_{\widehat{\mathbf{X}}^{\mathrm{Pos}}_i}D_{\boldsymbol{\theta}}(\widehat{\mathbf{X}}^{\mathrm{Pos}}_i) \|_2-1)^2,
\] 
where \( \widehat{\mathbf{X}}^{\mathrm{Pos}}_i = u\widetilde{\mathbf{X}}_{i}^{\mathrm{Pos}} + (1-u)\mathbf{X}^{\mathrm{Pos}}_i \) and \( \widetilde{\mathbf{X}}_{i}^{\mathrm{Pos}} \) represents any posterior fake sample generated by the generator.

Now, to maximize the extraction of information during the encoding process and ensure the preservation of meaningful information in the decoding process, we propose refining Eq.~\eqref{BNPWMMD:EG} using the optimization problem in Eq. \eqref{EG-refinment} along with incorporating objective functions in Eqs.~\eqref{gamma:est1} and \eqref{gamma:est2}. 
\begin{subequations}\label{gamma:est}
\begin{align}
(\widehat{\boldsymbol{\omega}},\widehat{\boldsymbol{\eta}})&=
\arg\min\limits_{\boldsymbol{\boldsymbol{\Omega},\boldsymbol{\mathcal{H}}}}\lbrace\mathcal{I}_1(\boldsymbol{\omega},\boldsymbol{\eta})+\mathcal{I}_2(\boldsymbol{\omega},\boldsymbol{\eta})-\mathcal{L}_{\boldsymbol{\gamma}_1}^{\text{DPi}}(\scalebox{1}{$\mathbf{X}^{\text{Pos}}_{1:N},\boldsymbol{c}_{1:N}$})-\mathcal{L}_{\boldsymbol{\gamma}_2}^{\text{DPi}}(\scalebox{1}{$G_{\boldsymbol{\omega}}(\boldsymbol{c}_{1:N}),\boldsymbol{c}_{1:N}$})
\nonumber\\
&~~~-\mathcal{L}_{\boldsymbol{\gamma}_2}^{\text{DPi}}(\scalebox{1}{$G_{\boldsymbol{\omega}}(\widetilde{\boldsymbol{c}}_{1:N})\footnotemark,\boldsymbol{c}_{1:N}$})-\mathcal{L}_{\boldsymbol{\gamma}_2}^{\text{DPi}}(\scalebox{1}{$G_{\boldsymbol{\omega}}(\boldsymbol{\xi}_{1:N}),\boldsymbol{c}_{1:N}$})\big\rbrace,\label{EG-refinment}\\
\widehat{\boldsymbol{\gamma}}_1&=
\arg\min\limits_{\boldsymbol{\Gamma}_1}-\mathcal{L}_{\boldsymbol{\gamma}_1}^{\text{DPi}}(\scalebox{1}{$\mathbf{X}^{\text{Pos}}_{1:N},\boldsymbol{c}_{1:N}$}),\label{gamma:est1}\\
\widehat{\boldsymbol{\gamma}}_2&=
\arg\min\limits_{\boldsymbol{\Gamma}_2}\big\lbrace
-\mathcal{L}_{\boldsymbol{\gamma}_2}^{\text{DPi}}(\scalebox{1}{$G_{\boldsymbol{\omega}}(\boldsymbol{c}_{1:N}),\boldsymbol{c}_{1:N}$})-\mathcal{L}_{\boldsymbol{\gamma}_2}^{\text{DPi}}(\scalebox{1}{$G_{\boldsymbol{\omega}}(\widetilde{\boldsymbol{c}}_{1:N}),\boldsymbol{c}_{1:N}$})-\mathcal{L}_{\boldsymbol{\gamma}_2}^{\text{DPi}}(\scalebox{1}{$G_{\boldsymbol{\omega}}(\boldsymbol{\xi}_{1:N}),\boldsymbol{c}_{1:N}$})\big\rbrace.\label{gamma:est2}
\end{align}
\end{subequations}
\footnotetext{$G_{\boldsymbol{\omega}}(\widetilde{\boldsymbol{c}}_{1:N})$ can be expressed as $f_1(\cdot)=G_{\boldsymbol{\omega}}(\widetilde{\boldsymbol{c}}_{1:N})-\mathbf{X}_{1:N}+\cdot$ in definitions \eqref{DPDV-lower} and \eqref{DPJS}. Therefore, it is well-defined.}
Note that we use the neural network $\{T_{\boldsymbol{\gamma}_{1}}\}_{\boldsymbol{\gamma}_{1}\in\boldsymbol{\Gamma}_{1}}$ to estimate $\text{MI}(\mathbf{X},E_{\boldsymbol{\eta}}(\mathbf{X}))$, while we use $\{T_{\boldsymbol{\gamma}_{2}}\}_{\boldsymbol{\gamma}_{2}\in\boldsymbol{\Gamma}_{2}}$ to estimate $\text{MI}(\widetilde{\mathbf{X}},E_{\boldsymbol{\eta}}(\mathbf{X}))$. Here, $\widetilde{\mathbf{X}}$ denotes any fake data generated by $G_{\boldsymbol{\omega}}(\cdot)$. The general diagram of the model is visually shown in Figure \ref{BNPWMMD-diagram}.
A detailed description of the model's architecture, accompanied by a flowchart illustrating its operational process, is provided in Appendix  \ref{app:imp-details}.

\section{Experimental results}\label{sec:experiment}

In this section, we evaluate the performance of the procedure using simulation and real datasets. Initially, we need to set the hyperparameters of the DP before placing it on the data distribution $F$. 
Following \cite{fazeli2023semi}, we consider the base measure to be a multivariate normal distribution, with the concentration parameter estimated using the MAP procedure described in Remark \ref{remark1} and \ref{remark2} of Section \ref{sec:DP}.
Additionally, the value of \( N \) in the DP approximation of Eq.~\eqref{approx of DP} is chosen according to the guidance provided in Remark \ref{remark3} of Section \ref{sec:DP}



We initiate this section by presenting several simulation examples to assess the performance of both KL-based and JS-based DPMINEs in estimating the MI. Subsequently, we will explore the application of the better-performing estimator in refining the training procedure of VAE-GANs.
\subsection{MINE: BNP versus FNP}\label{sec:BNP vs FNP}
\subsubsection{Estimation accuracy} In our experimental evaluation, we must evaluate the accuracy of both the KL-based and JS-based DPMINEs in estimating the MI as it directly impacts the accuracy of refining the desired BNP VAE-GAN model. Additionally, we need to compare the BNP estimators with their frequentist counterparts. To achieve this, we present estimation values of both BNP and FNP procedures through Figure \ref{MINE-kl-js} for two cases--when the two random variables $X$ and $Y$ are independent, and when they are dependent. The figure shows that the BNP estimators exhibit better convergence than the FNP estimators as well as lower variance, which stems from the improved out-of-sample performance of the BNP procedure that reduces sensitivity to sample size and outliers. However, in non-independent cases, the JS-based estimator shows a gap between its estimated value and the true value, unlike the KL-based estimator. Therefore, we decided to continue our investigation in the next experimental part using KL-based estimators.
\begin{figure*}[htbp]
\centering\begin{tabular}[h!]{|cc||cc|}
\hline
\multicolumn{2}{|c||}{$\text{MI}^{\text{True}}(X,Y)=0$}&\multicolumn{2}{c|}{$\text{MI}^{\text{True}}(X,Y)=0.69$}\\
\hline
\includegraphics[width=.22\textwidth,height=2cm]{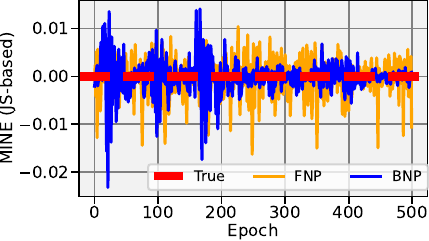}& \includegraphics[width=.22\textwidth,height=2cm]{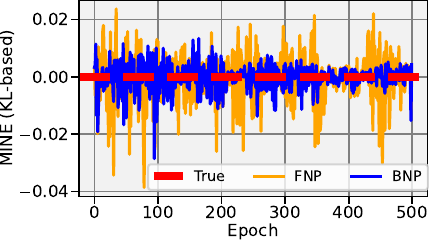}&\includegraphics[width=.22\textwidth,height=2cm]{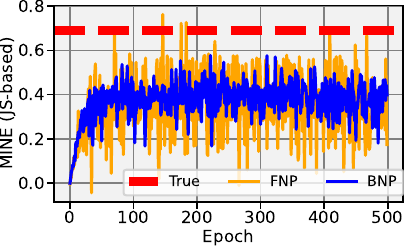}& \includegraphics[width=.22\textwidth,height=2cm]{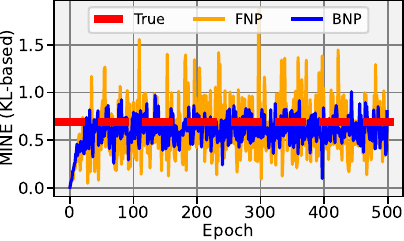}\\
\hdashline
\multicolumn{2}{|c||}{{\scriptsize $X,Y\overset{\text{IID}}{\sim}U(-1,1)$}}&\multicolumn{2}{|c|}{{\scriptsize $X=\text{sign}(Z),\,Z\sim N(0,1)$; $Y=X+\epsilon,\, \epsilon\sim N(0,0.2)$}}\\
\hline
\end{tabular}
\caption{MINE estimation of the MI between two random variables $X$ and $Y$ using both BNP and FNP frameworks, given a sample size of 16 over 500 epochs. The red dashed line represents the true value of MI. The blue line represents the BNP estimation of MI (our method), while the yellow line represents the FNP estimation. The left-hand figure in each experiment represents the JS-based estimator, and the right-hand figure represents the KL-based estimator. 
}
\label{MINE-kl-js}
\end{figure*}

\subsubsection{Robustness of the DPMINE to high dimensionality}
To assess the robustness of our estimator to high dimensionality, we report the results of two previous examples with varying dimensions of $d=2,10,100,1000$ (see Figures~\ref{robust-Independence} and \ref{robust-Dependence}). The blue color represents the BNP estimation of MI, while the yellow one represents the FNP estimation. These figures demonstrate the robustness of the DPMINE estimator to high dimensionality, indicating better convergence and reduced fluctuations in MI estimation, particularly in cases where the two variables are not independent. This highlights the estimator's ability to control variance, which is a key objective of this paper. In Figure~\ref{robust-Dependence}(e), we provide results for 1500 epochs, as higher dimensions require more epochs for convergence based on the chosen learning rate, in addition to the results already shown for 500 epochs in Figure \ref{robust-Dependence}(d).

\begin{figure}[h!]
\centering
\subfloat[$d=2$]{
  \begin{minipage}[b]{1\linewidth}
    \centering
    \includegraphics[width=.4\textwidth,height=3cm]{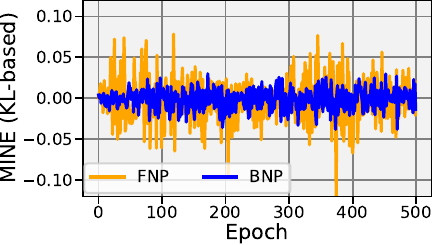}
    \includegraphics[width=.4\textwidth,height=3cm]{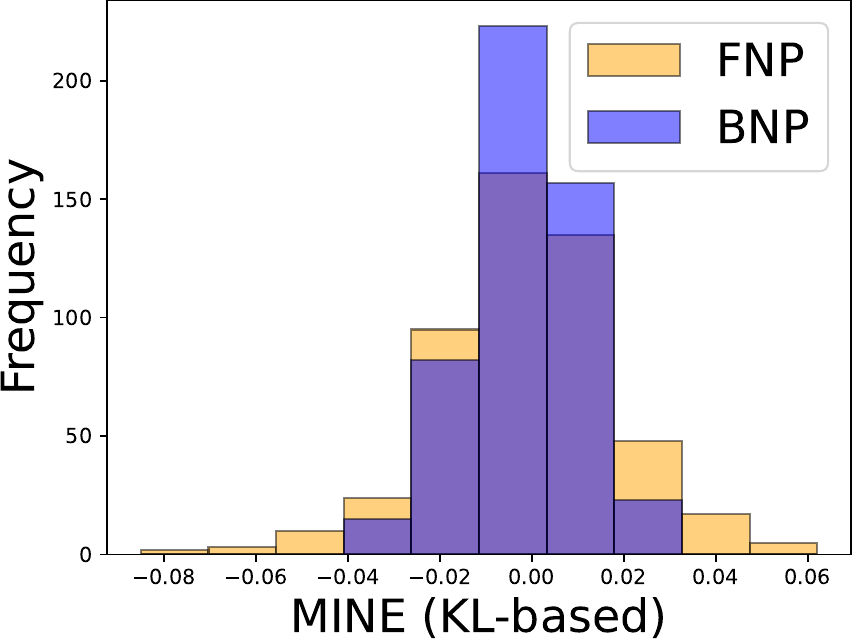}
  \end{minipage}
}
\\
\subfloat[ $d=10$]{
  \begin{minipage}[b]{1\linewidth}
    \centering
    \includegraphics[width=.4\textwidth,height=3cm]{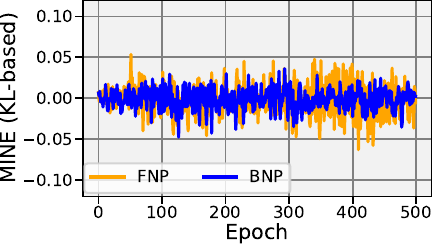}
    \includegraphics[width=.4\textwidth,height=3cm]{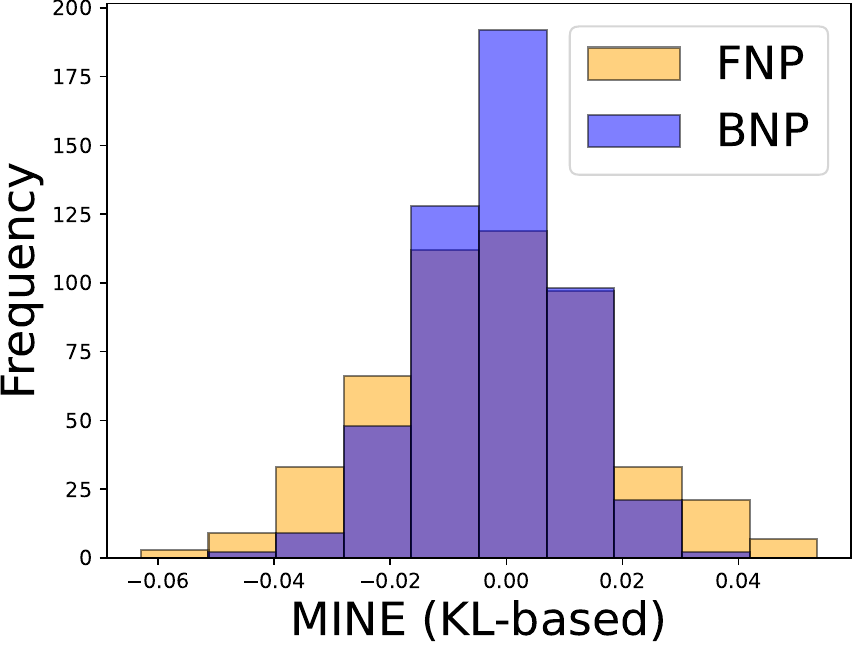}
  \end{minipage}
}
\\
\subfloat[ $d=100$]{
  \begin{minipage}[b]{1\linewidth}
    \centering
    \includegraphics[width=.4\textwidth,height=3cm]{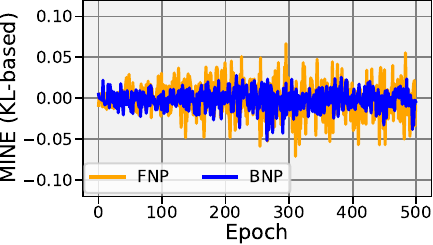}
    \includegraphics[width=.4\textwidth,height=3cm]{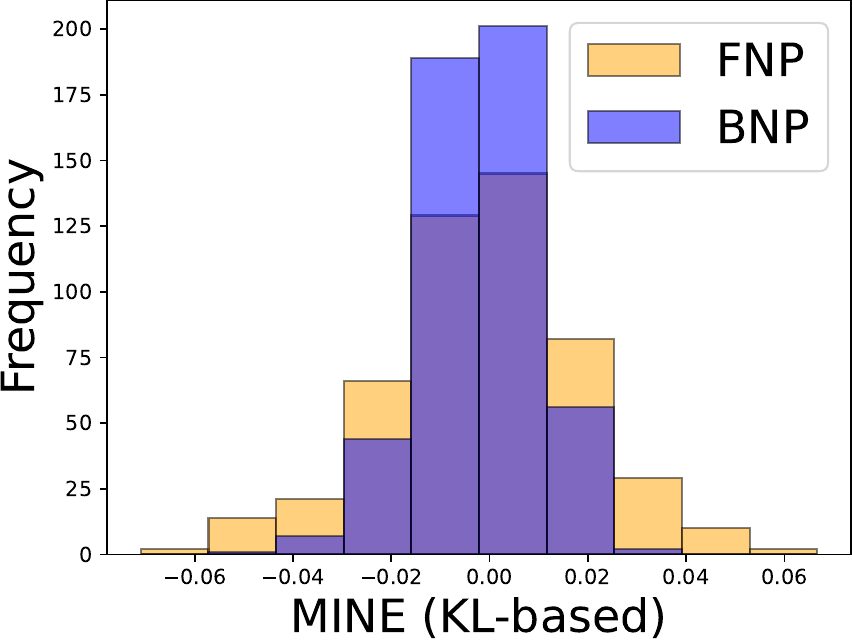}
  \end{minipage}
}
\\
\subfloat[ $d=1000$]{
  \begin{minipage}[b]{1\linewidth}
    \centering
    \includegraphics[width=.4\textwidth,height=3cm]{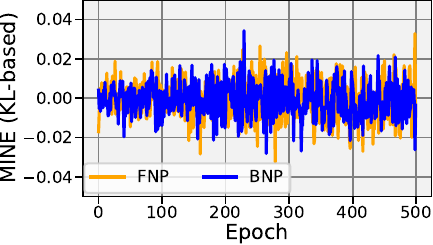}
    \includegraphics[width=.4\textwidth,height=3cm]{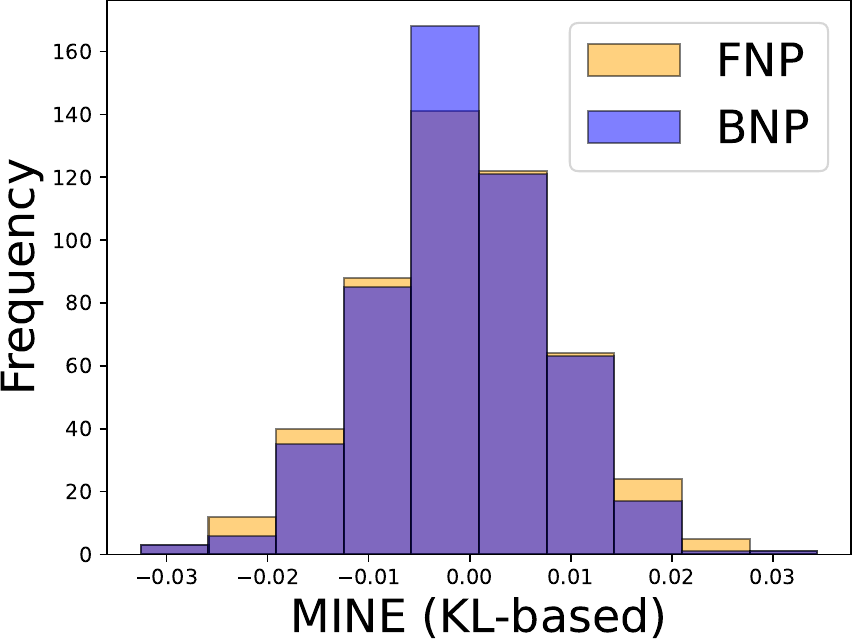}
  \end{minipage}
}
\caption{MI estimations between two random variables $\mathbf{X},\mathbf{Y}\overset{\text{IID}}{\sim}U(-\mathbf{1},\mathbf{1})$, $\mathbf{X},\mathbf{Y}\in\mathbb{R}^d$ for various dimension $d$.}
\label{robust-Independence}
\end{figure}

\begin{figure}[htbp]
\centering
\subfloat[$d=2$]{
  \begin{minipage}[b]{1\linewidth}
    \centering
    \includegraphics[width=.4\textwidth,height=3cm]{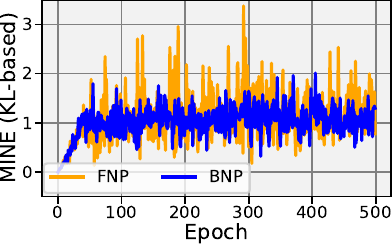}
    \includegraphics[width=.4\textwidth,height=3cm]{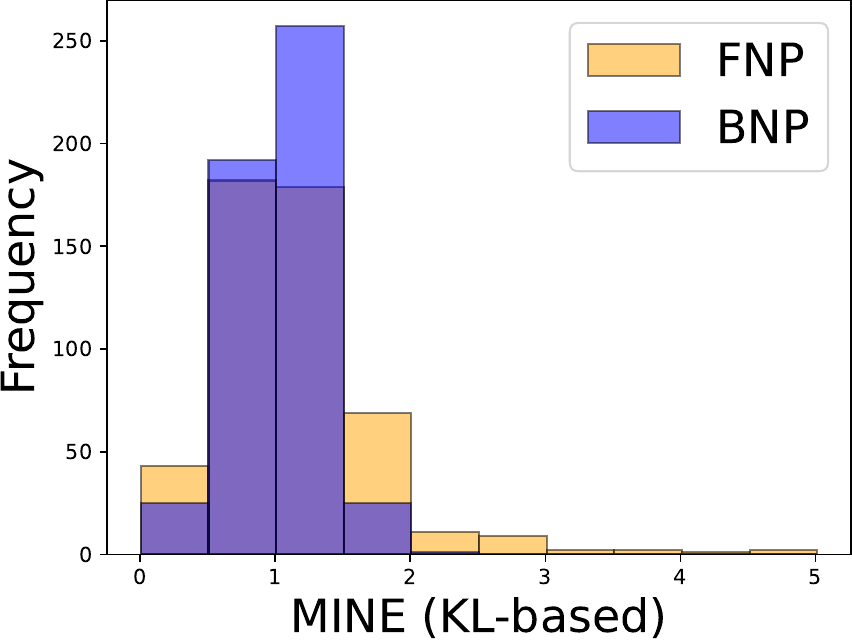}
  \end{minipage}
}
\\
\subfloat[ $d=10$]{
  \begin{minipage}[b]{1\linewidth}
    \centering
    \includegraphics[width=.4\textwidth,height=3cm]{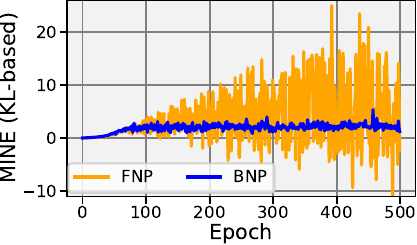}
    \includegraphics[width=.4\textwidth,height=3cm]{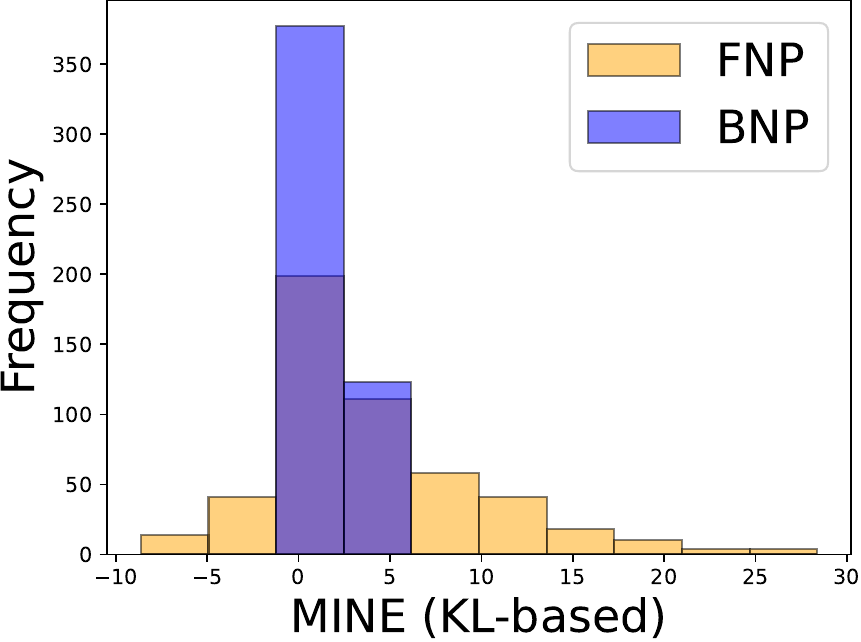}
  \end{minipage}
}
\\
\subfloat[ $d=100$]{
  \begin{minipage}[b]{1\linewidth}
    \centering
    \includegraphics[width=.4\textwidth,height=3cm]{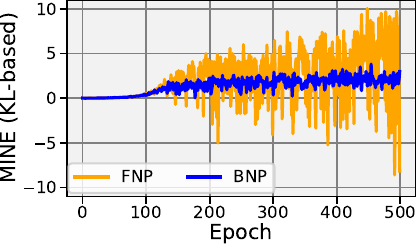}
    \includegraphics[width=.4\textwidth,height=3cm]{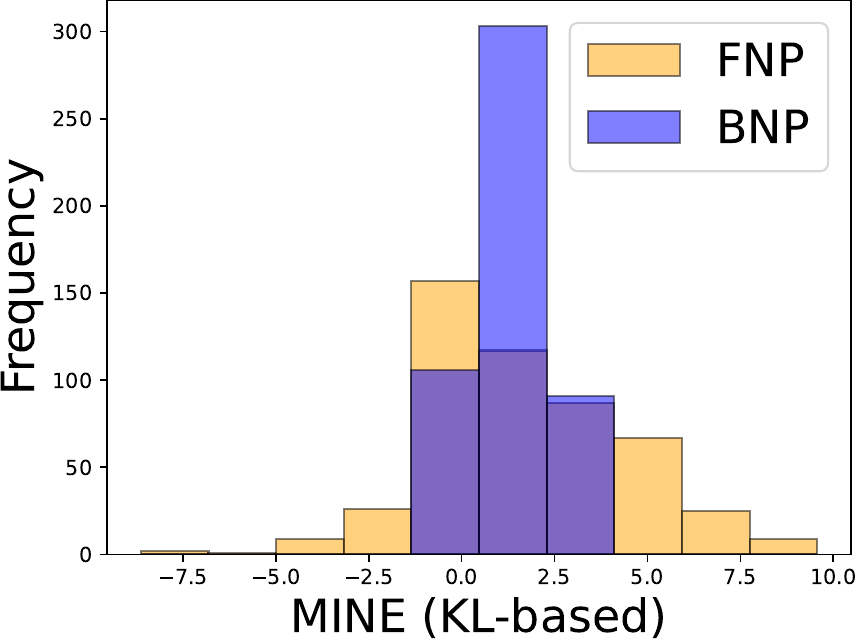}
  \end{minipage}
}
\\
\subfloat[ $d=1000$]{
  \begin{minipage}[b]{1\linewidth}
    \centering
    \includegraphics[width=.4\textwidth,height=3cm]{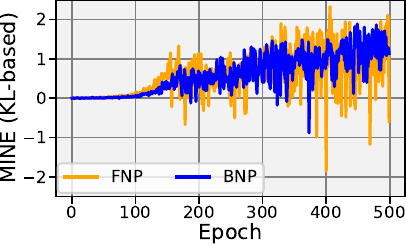}
    \includegraphics[width=.4\textwidth,height=3cm]{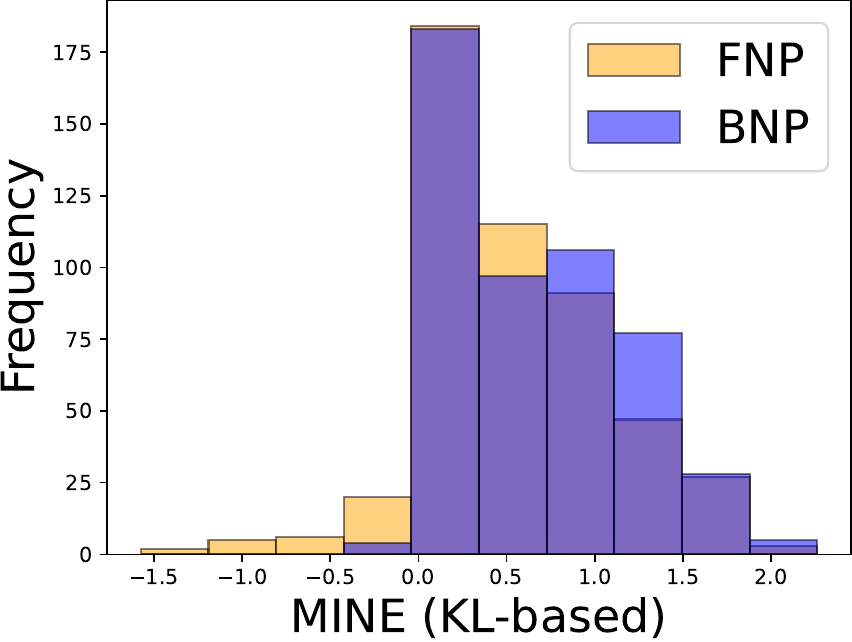}
  \end{minipage}
}
\\
\subfloat[ $d=1000$ (epoch=1500)]{
  \begin{minipage}[b]{1\linewidth}
    \centering
    \includegraphics[width=.4\textwidth,height=3cm]{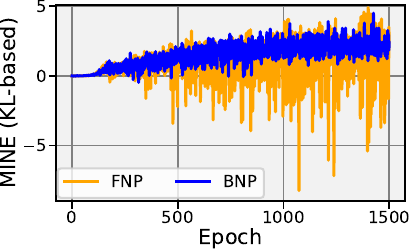}
    \includegraphics[width=.4\textwidth,height=3cm]{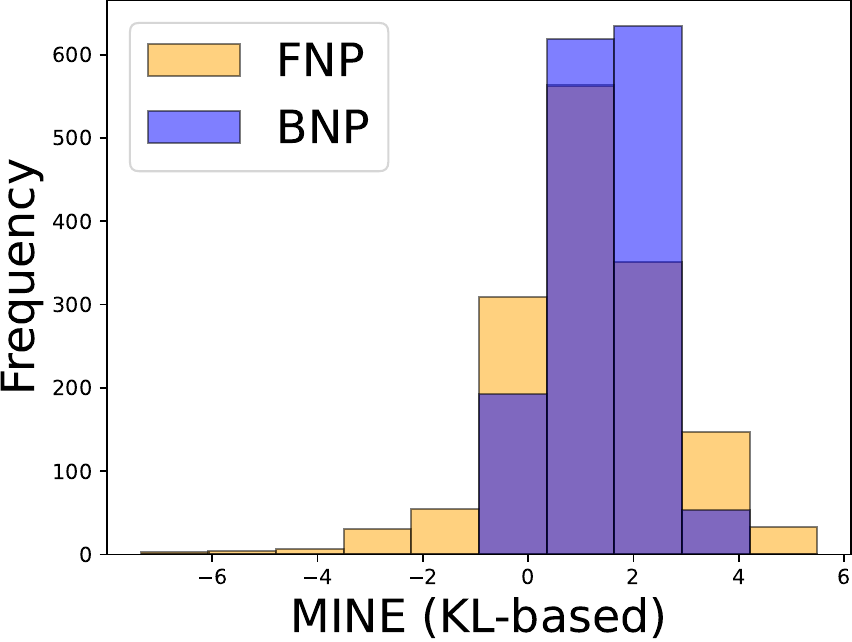}
  \end{minipage}
}
\caption{MI estimations between two random variables $\mathbf{X}=\text{sign}(\mathbf{Z}),\,(\mathbf{Z}\in\mathbb{R}^d)\sim N(\mathbf{0},I_d)$ and $\mathbf{Y}=\mathbf{X}+\boldsymbol{\epsilon},\, (\boldsymbol{\epsilon}\in\mathbb{R}^d)\sim N(\mathbf{0},0.2I_d)$ for various dimension $d$.}
\label{robust-Dependence}
\end{figure}

\begin{table}[!t]
\centering
\begin{tabular}{|c|c|c|c|c|c|c|}
\hline
\multirow{2}{*}{\textbf{Estimator}} & \multirow{2}{*}{\textbf{Epochs}} & \multicolumn{4}{c|}{\textbf{Dimensions}} \\
\cline{3-6}
 & & \textbf{2} & \textbf{10} & \textbf{100} & \textbf{1000} \\
\hline
\multirow{2}{*}{MINE} & 500 & 5 sec & 5 sec & 5 sec & 5 sec \\
& 1500 & 12 sec & 12 sec & 12 sec & 12 sec \\
\hline
\multirow{2}{*}{DPMINE} & 500 & 42 sec & 42 sec & 42 sec & 42 sec \\
& 1500 & 1 min 25 sec & 1 min 25 sec & 1 min 25 sec & 1 min 25 sec \\
\hline
\end{tabular}
\caption{Run time of the MINE and DPMINE on different dimensions of the data and epochs.}
\label{table:runtime}
\end{table}
\subsection{Refining VAE-GAN training via MINE}
To implement the BNPWMMD+DPMINE, we consider the Gaussian kernel function, defined as $k_{G_{\sigma}}(\mathbf{X},\mathbf{Y})=\exp(\frac{-||\mathbf{X}-\mathbf{Y}||^{2}}{2\sigma^2}) $ with bandwidth parameter $ \sigma $, in the MMD distance given in Appendix \ref{app:distance}. We search for  the appropriate bandwidth parameter $\sigma$ over a fixed grid of values, 
$\sigma\in\lbrace 2, 5, 10, 20, 40, 80 \rbrace$. We then compute the mixture of Gaussian kernels, denoted as $k(\cdot, \cdot) = \sum_{\sigma} k_{G_{\sigma}}(\cdot, \cdot)$. This selection of kernel function and bandwidth has been shown to yield satisfactory performance in training MMD-based GANs, as mentioned in \cite{Li,fazeli2023semi,fazeli2023bayesian}. Then, we consider next examples to investigate the proposed approach.

\subsubsection{Synthetic Example}
\paragraph{Coil Dataset:} 
We first look at a synthetic example that showcases the effectiveness of the BNPWMMD+DPMINE in mitigating mode collapse. In this example, we simulate $5000$ samples in 3D space, denoted as $(X(t), Y(t), Z(t))$, where $X(t) = 6\cos{t}$, $Y(t) = 6\sin{t}$, and $Z(t) = t$, with $t$ ranging from $-2\pi$ to $4\pi$.  We then normalize all datasets to a range between $-1$ and $1$. 
This normalization ensures compatibility with the hyperbolic tangent activation function, which is used in the generator's last layer in all compared models. We also used a latent dimension of 100 with a sub-latent dimension of 10 in this example.
Additionally, we provide the results of the BNPWMMD 
to display the basic model's performance
in covering data space in the absence of DPMINE. 

For a comprehensive comparison, we include results from a set of experiments involving three well-known VAE-GAN models commonly used in the literature: BiGAN+MINE, $\alpha$-WGAN+MINE, and $\alpha$-WGAN.\footnote{In this section, the GP notation is omitted for brevity, as GP is incorporated into the loss functions of both the BNPWMMD and $\alpha$-WGAN models.} As demonstrated by \cite{belghazi2018mutual}, adding MINE improves the performance of BiGAN, as evidenced by their comparison between BiGAN and BiGAN+MINE. Thus, we exclude the standalone BiGAN from our experiments.

Notably, BNPWMMD leverages posterior samples generated from \( F_{N}^{\text{Pos}} \), as defined in Eq.~\eqref{approx of DP}, to optimize its objective functions, rather than relying on direct samples from the empirical distribution used by its FNP counterpart. Consequently, it is essential to employ the BNP version of MINE to regularize BNPWMMD's cost function during network parameter optimization. In contrast, FNP models such as \( \alpha \)-WGAN and BiGAN use the empirical distribution to compute sample means in their cost functions, which are compatible with the calculation of MINE in Eq.~\eqref{DV-ECDF} (see Appendix \ref{app:baseline} for further details on the baselines' cost functions).
Given these structural differences, applying the MINE designed for FNP models to BNPWMMD (i.e., BNPWMMD+MINE) or using DPMINE with FNP models like \( \alpha \)-WGAN and BiGAN (i.e., \( \alpha \)-WGAN+DPMINE or BiGAN+DPMINE) is inappropriate. Such combinations are incompatible with the input structures and cost function formulations of these models.


We show 1,000 samples in Figure \ref{coil-exm}, generated from noise input (random samples) and reconstructed from encoded and decoded real input. 
This figure clearly demonstrates the impact that the DPMINE has in generative modeling, as it allows the BNPWMMD to cover the entire data space. Although applying the MINE to $\alpha$-WGAN improves the convergence of random samples compared to the absence of these refinements, it still struggles to effectively cover the data space even with these modifications. 

Moreover, the BNPWMMD+DPMINE maintains a good balance in coverage between random and reconstructed samples while the $\alpha$-WGAN model cannot provide a reasonable balance between these two cases. The BiGAN+MINE performs poorly, in comparison, after $5000$ epochs and requires a significantly larger number of epochs to converge. However, we found that even with these extended epochs, it cannot achieve better quality than the BNP counterpart provided in this paper. An additional synthetic example in Appendix \ref{app:additional-results-bunny} further demonstrates the effectiveness of our model.

\begin{figure*}[htbp]
\centering\begin{tabular}[h!]{|c|c||c|c|c|}
\hline
\multirow{2}{*}[10pt]{\rotatebox{90}{Real Dataset}}&\multirow{2}{*}[30pt]{\includegraphics[width=0.17\textwidth,height=3.5cm]{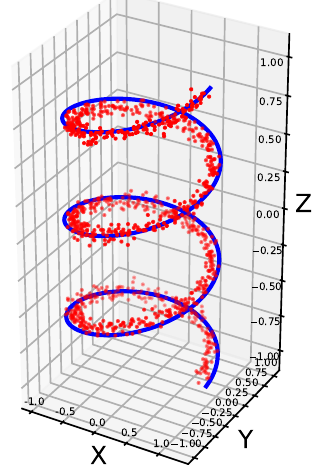}}
&\multirow{2}{*}[25pt]{\rotatebox{90}{Generated Samples}} &\multirow{1}{*}[40pt]{\rotatebox{90}{\scriptsize Random}}& \includegraphics[width=.63\textwidth,height=2.7cm]{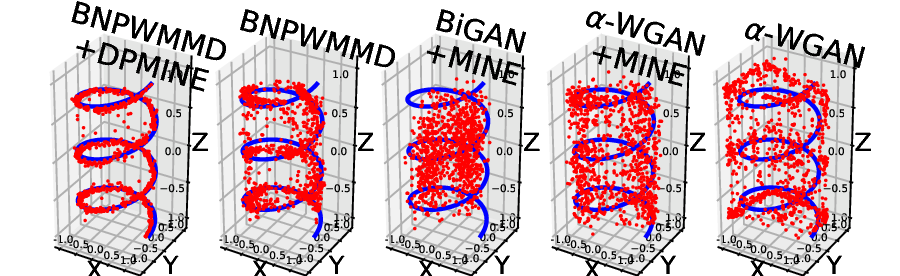}\\
\cline{4-5}
&& &\multirow{1}{*}[53.5pt]{\rotatebox{90}{\scriptsize Reconstruction}}& \includegraphics[width=.63\textwidth,height=2.7cm]{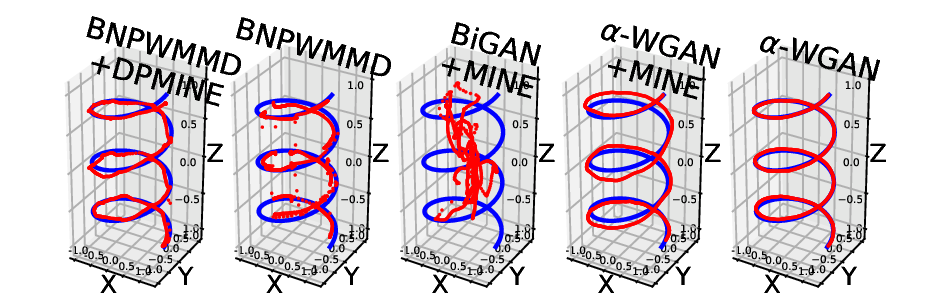}\\
\hline
\end{tabular}
\caption{1000 randomly generated and reconstructed samples for the coil example. 
}
\label{coil-exm}
\end{figure*}

\subsubsection{Real Example}
\paragraph{COVID-19 Dataset:} This dataset comprises of 3D chest CT images of 1000 subjects diagnosed with lung infections after testing positive for COVID-19\footnote{The dataset is freely available online at \url{https://doi.org/10.7910/DVN/6ACUZJ} (license: CC0 1.0).}. All images in this dataset are stored in the DICOM format and have a resolution of 16 bits per pixel, with dimensions of $512\times512$ pixels in grayscale. 
We randomly selected 200 patients, resulting in a total of 91,960 images. As part of the preprocessing step, we first stored each patient's data in a Neuroimaging Informatics Technology Initiative (NIFTI) format file. Then, we converted each NIFTI file into a 3D image with a dimension of $64\times64\times64$. Each dimension represents the axial, sagittal, and coronal views of the lungs. After normalizing the dataset, we used a mini-batch size of 16 and trained all compared models for 7500 epochs. We also used a latent dimension of 1000 as suggested in \cite{kwon2019generation} for high-dimensional cases, with a sub-latent dimension of 100 as used in \cite{fazeli2023bayesian}. 

To showcase the image generation ability of BNPWMMD+DPMINE compared to other models, we provide a 3D visualization of real and random samples of chest CT images in Figure \ref{fig-cube}. This indicates the exceptional capability of the model to generate sharp images. For a more comprehensive view of the generated samples, we included additional slices from each dimension representation in Appendix \ref{app:additional-results-covid} 
We also compared the reconstruction capability of different models by displaying slices of a reconstructed sample from each model in the Appendix. The BNPWMMD+DPMINE excels at reconstructing the high-dimensional dataset, further validating the effectiveness of our proposed model. 
\begin{figure}[ht]
\centering
\subfloat[{\tiny Real Dataset}]{\label{fig1}\includegraphics[width=.22\linewidth]{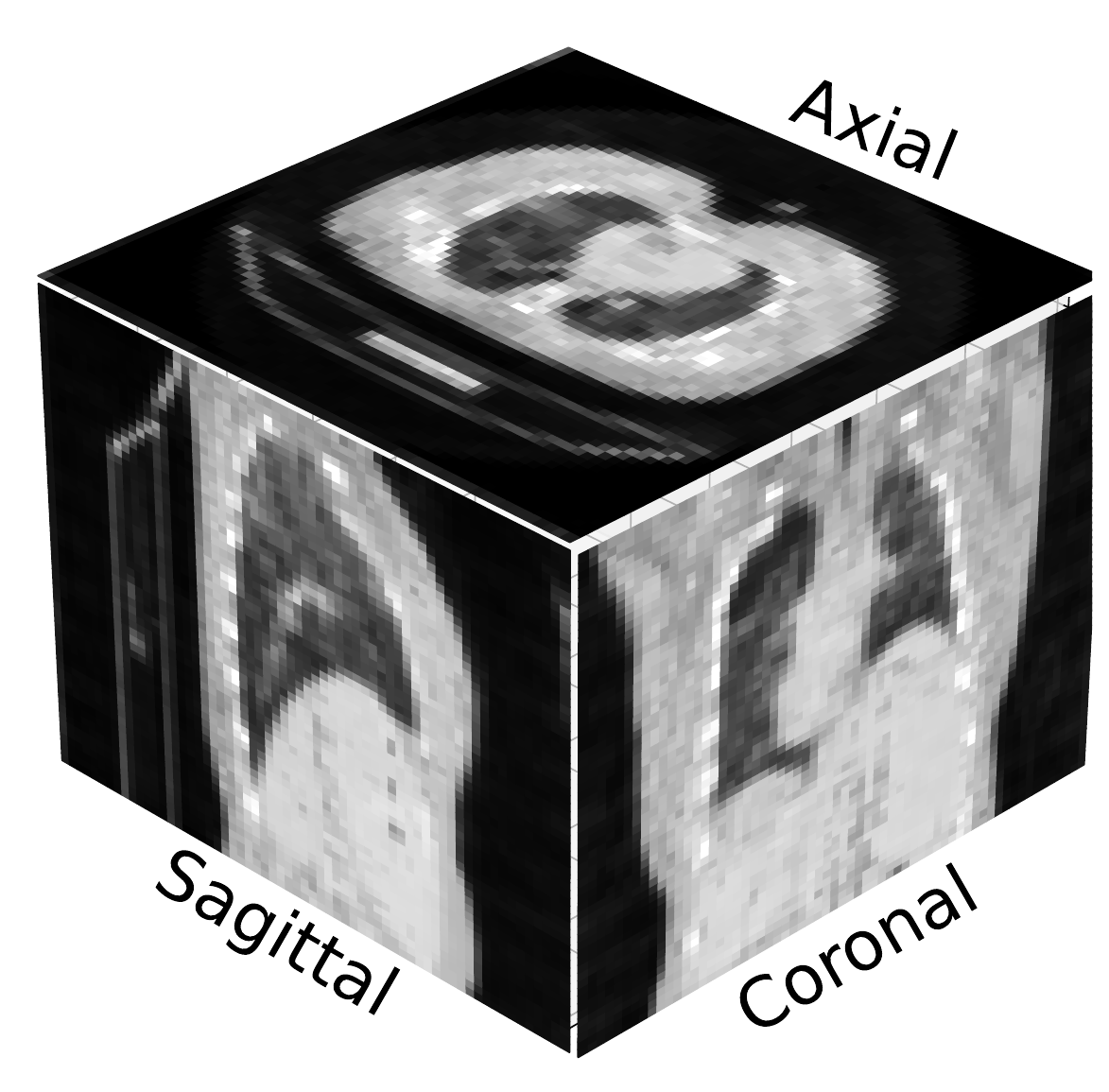}}\quad
\subfloat[{\tiny BNPWMMD+DPMINE}]{\label{fig2}\includegraphics[width=.22\linewidth]{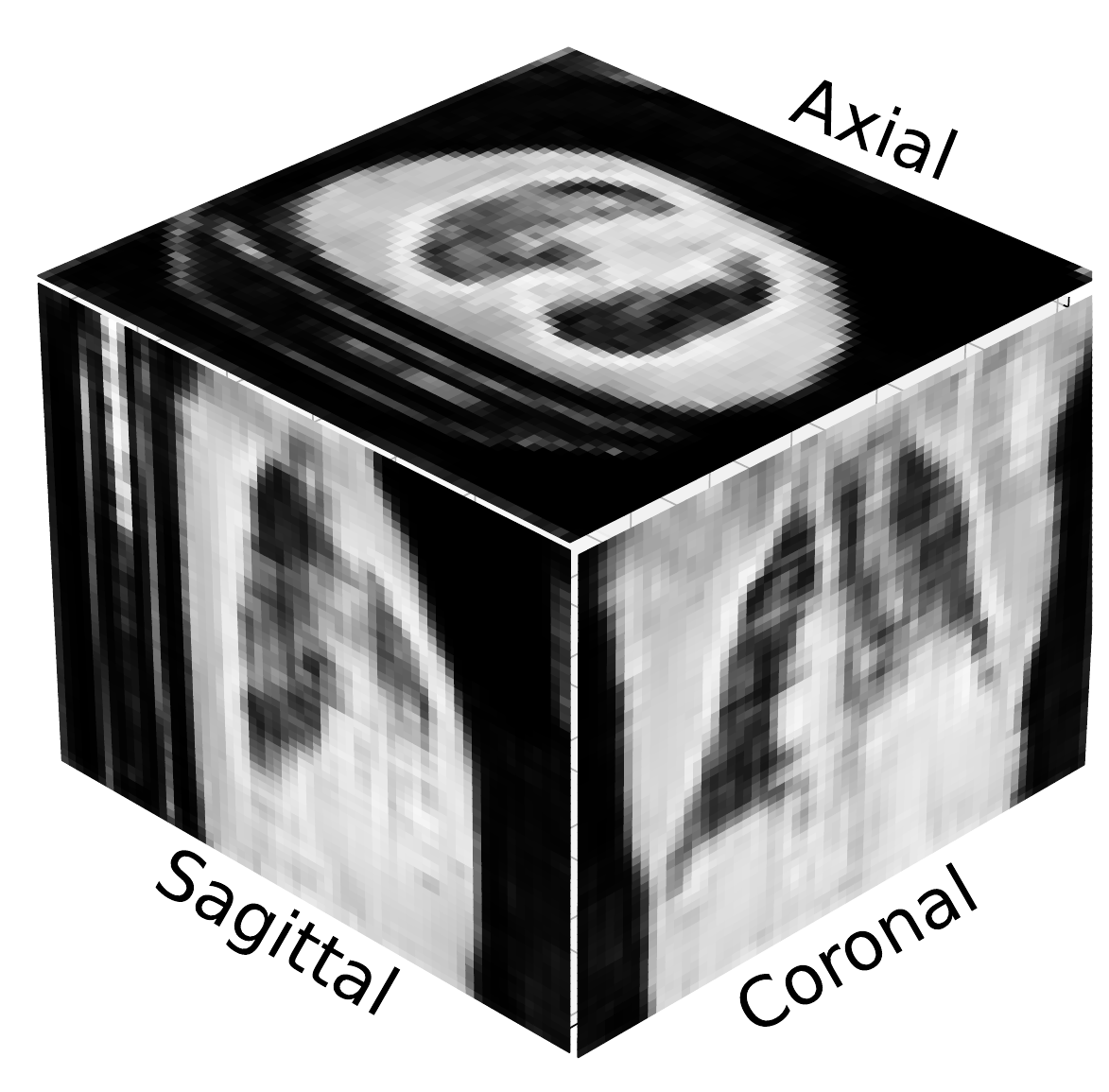}}\quad
\subfloat[{\tiny $\alpha$-WGAN+MINE}]{\label{fig3}\includegraphics[width=.22\linewidth]{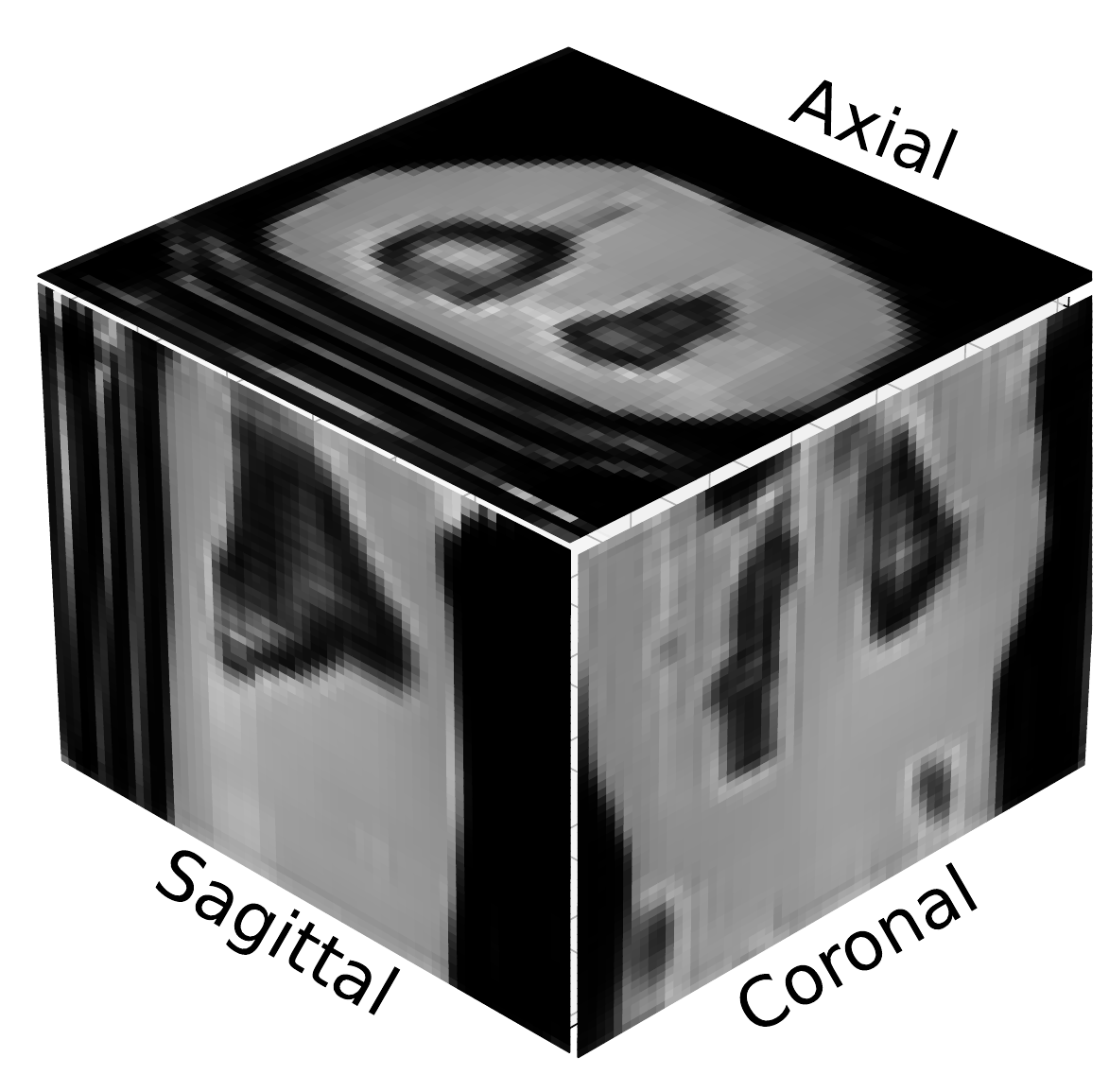}}\quad
\subfloat[{\tiny BiGAN+MINE}]{\label{fig4}\includegraphics[width=.22\linewidth]{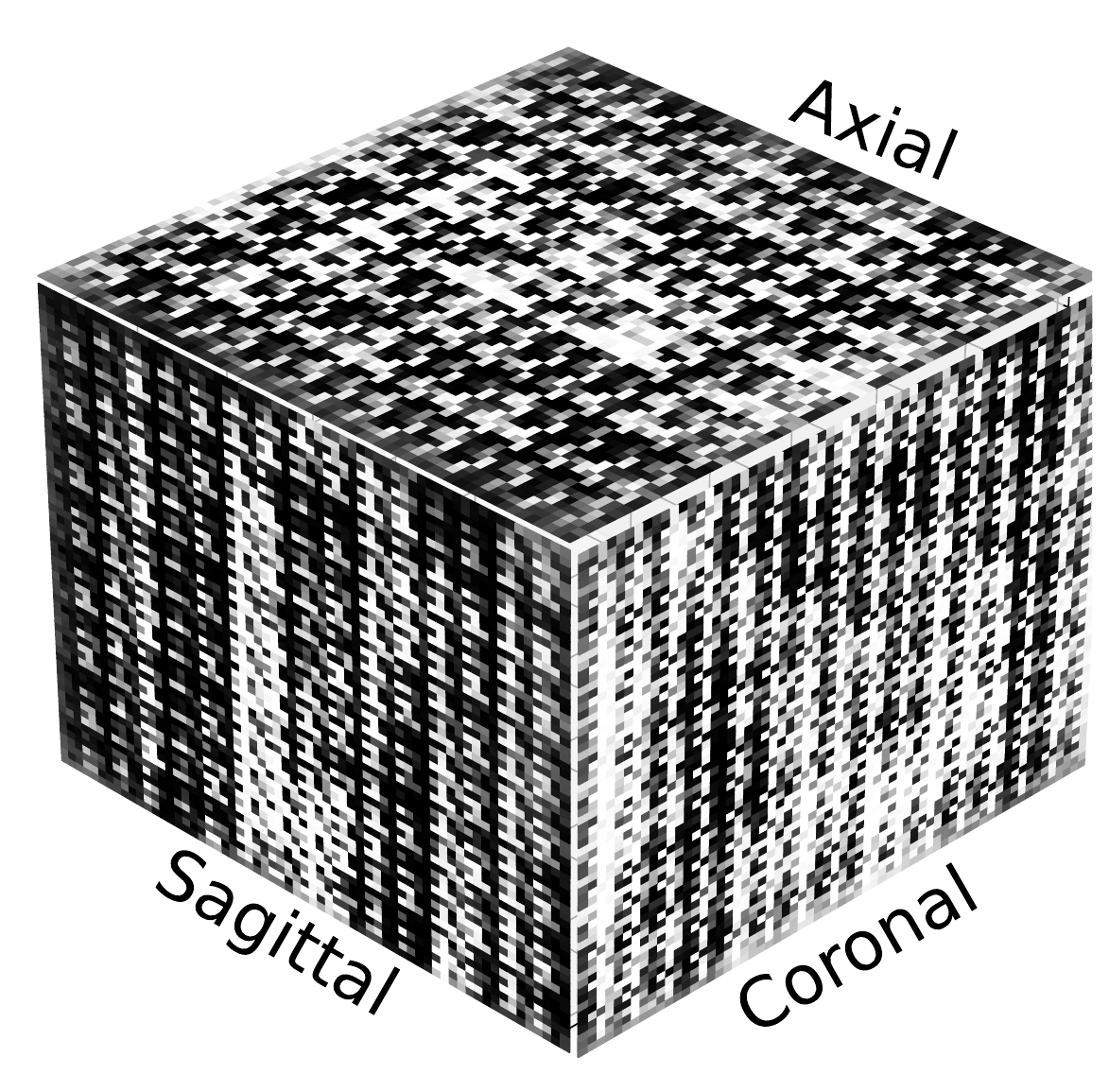}}
\caption{A 3D visualization of a real sample and a randomly generated sample in COVID-19 example.}
\label{fig-cube}
\end{figure}

However, additional quantitative tools are essential to assess the similarity between real and fake samples. 
Reducing the data dimensionality allows for better visualization to understand the similarities between data points. 
In Figure \ref{fig1}(a), we present the $t$-SNE \citep{van2008visualizing} mapping of the real dataset and 200 random samples to 2D points which clearly shows the extensive coverage of samples generated from the BNPWMMD+DPMINE on the real dataset. Conversely, the $\alpha$-WGAN+MINE exhibits mode collapse, and BiGAN+MINE demonstrates poor performance.  We also use a custom encoder, a 2-layer linear neural network, to map samples to the 2D space displayed by Figure \ref{fig1}(b), which supports a similar conclusion.  

To quantify the dissimilarity in the 2D space between real and generated features, denoted as $\boldsymbol{f}_r:=(\text{Feature1}_{1:200,r},\text{Feature2}_{1:200,r})$ and $\boldsymbol{f}_g:=(\text{Feature1}_{1:200,g},\text{Feature2}_{1:200,g})$, we calculate the average of Fr\'{e}chet inception distance (FID) and the Kernel inception distance (KID) metrics over 100 replications. 
Results for the empirical MMD metric between real and generated scans in the 2D feature space, computed with the same kernel and bandwidth parameters used in the BNPWMMD model, are also presented in Table \ref{table1}. In this table, we deliberately avoid computing MMD in the data space, as our proposed model is already optimized by minimizing MMD between real and fake scans in the data space as a part of training process. As a result, our model is expected to achieve better MMD scores in the data space compared to others, rendering such a comparison unfair. Smaller values in the reported metrics indicate lower dissimilarity and better performance, thereby confirming the significant role of the proposed model in avoiding mode collapse. Additionally, we use the Multi-Scale Structural Similarity Index (MS-SSIM) \citep{wang2003multiscale} to evaluate the perceived visual quality of images. MS-SSIM captures both structural information and visual quality by performing a multi-scale decomposition and comparing luminance, contrast, and structure across various scales between real and synthetic samples. As an instance-level similarity metric, MS-SSIM assumes a one-to-one mapping between real and generated samples. While the methods being compared are unconditional generators, using MS-SSIM is justified because the axial, sagittal, and coronal views of lung scans are generally well-aligned across patients. This alignment facilitates meaningful comparisons between synthetic and corresponding real views, enabling an evaluation that effectively captures structural fidelity.
A higher MS-SSIM score suggests better quality, while a lower score indicates lower quality.
For further details on the calculations, refer to Appendix \ref{app:imp-details}. Additionally, Appendix  \ref{app:additional-results-covid} presents another real-world example using a brain MRI dataset of patients with tumors, extending the evaluation to a broader range of datasets from different domains. This inclusion provides a more comprehensive validation of the model’s effectiveness and highlights its versatility across various types of data.

\begin{figure}[htbp]
  \centering
  \subfloat[{\small t-SNE}]{\includegraphics[width=.45\linewidth]{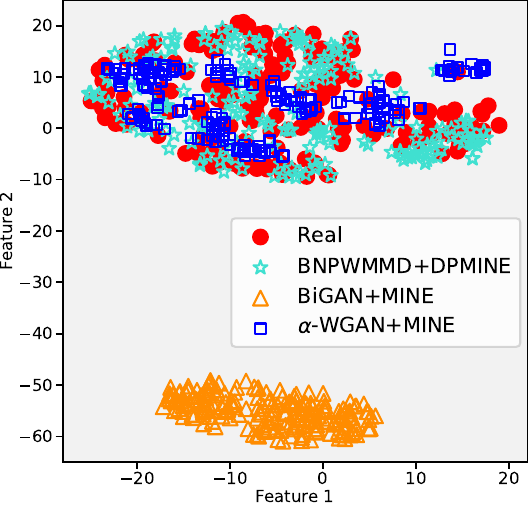}}\quad
  \subfloat[{\small Custom encoder}]{\includegraphics[width=.45\linewidth]{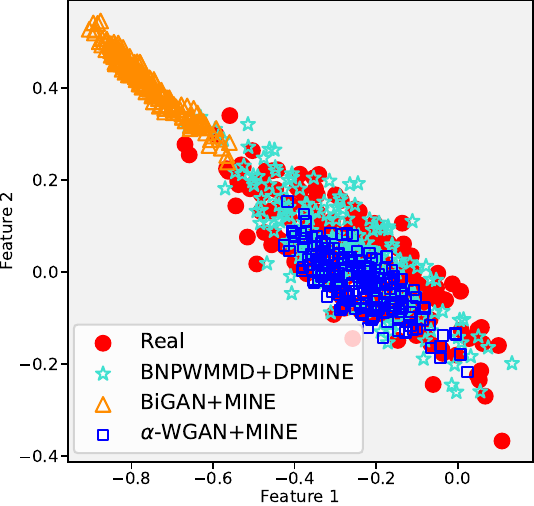}}
  \caption{Scatter plot of 2D features in COVID-19 example.}
  \label{fig1}
\end{figure}

\begin{table}[ht]
\centering
\small
\captionof{table}{Comparison of statistical scores for COVID-19 example.}
\scalebox{0.85}{
\begin{tabular}{ l l| c |c|c }
\hline
\toprule
\multicolumn{2}{l|}{\multirow{2}{*}{Evaluator}} & BNPWMMD&$\alpha$-WGAN & BiGAN \\
&&+DPMINE&+MINE&\hspace{.25 cm}+MINE\\
\hline
\multirow{1}{*}{FID} &Custom Encoder&$\mathbf{0.00063}$ &$0.02177$ &$0.25729$\\
&t-SNE&$\mathbf{9.48447}$&$13.3682$&$6298.69$\\
\hline
\multirow{1}{*}{KID}&Custom Encoder&$\mathbf{0.00069}$ &$0.02281$ &$0.46684$\\ 
&t-SNE&$\mathbf{0.99858}$&$2.31426$&$3107.53$\\
\hline
\multirow{1}{*}{MMD} &Custom Encoder&$\mathbf{0.00902}$ &$0.05131$ &$0.92849$\\
&t-SNE&$\mathbf{0.00074}$&$0.05742$&$0.34447$\\
\hline
 \multirow{1}{*}{MS-SSIM}&&$\mathbf{0.43465}$ &$0.39756$ &$0.02251$\\
\bottomrule
\end{tabular}
}
\label{table1}
\end{table}

\section{Concluding remarks}\label{sec:conclusion}
We introduced a Bayesian nonparametric (BNP) framework for mutual information estimation using a Dirichlet process-based neural estimator (DPMINE). The core theoretical contribution of this work lies in the construction of a tighter lower bound for mutual information using a finite representation of the Dirichlet process posterior, which was established specifically for KL-based estimators. This tighter bound enhances the informativeness of the MI loss function and introduces a form of regularization that reduces sensitivity to sample variability and outliers, leading to improved convergence during training. The combination of prior regularization and empirical data allows DPMINE to address the trade-off between variance and accuracy, which is a known limitation of existing KL-based estimators. Our experimental results demonstrate that DPMINE outperforms empirical distribution function-based methods (EDF-based) in terms of convergence and variance reduction, particularly in complex high-dimensional generative modeling tasks such as 3D CT image generation.

The proposed estimator effectively enhances the training of deep generative models by improving the reliability and informativeness of the learned representations. Notably, the demonstrated improvements in convergence and variance control make DPMINE highly suitable for handling challenging high-dimensional datasets where accurate mutual information estimation is critical. Although this paper focuses on the application of DPMINE in generative models, the underlying BNP framework extends beyond this setting and opens new opportunities for optimizing mutual information in other complex learning tasks, including representation learning, reinforcement learning, and Bayesian decision-making.  

An important future direction involves adapting the DPMINE framework to address emerging challenges in training large language models (LLMs). Mutual information estimation is increasingly recognized as a key tool for improving representation quality and guiding information flow in transformer-based architectures. The ability of DPMINE to provide stable and accurate estimates of information content could be leveraged to refine LLM pretraining objectives, enhance token embedding strategies, and improve generalization in downstream tasks. Furthermore, by reducing variance and stabilizing training, the BNP framework offers a promising solution for mitigating the sensitivity of LLM training to noise and sample heterogeneity, particularly in few-shot and domain-adaptive settings.  

Additionally, the regularization effect introduced by the BNP framework provides a principled approach for reducing overfitting and improving generalization in settings with limited data, such as low-resource languages and domain-specific corpora. Incorporating DPMINE into federated learning contexts could further strengthen privacy-preserving data generation and cross-device model training by enhancing the robustness of information flow across non-IID datasets. However, this will pose significant challenges since the procedure developed in this paper is based on the IID assumption, and extending it to non-IID. settings will require additional methodological adjustments.  






\newpage

\appendix

\label{app:theorem}



\section{Notations}
\begin{itemize}
    \item $F$: Real data distribution.
    \item $(\mathbf{X}_{1:n})$: A sample of $n$ independent and identically distributed random variables generated from $F$.
    \item $\text{DP}(a+n,H^{\ast})$: The Dirichlet process (DP) posterior with concentration parameter $a+n$, $a>0$, and base measure $H^{\ast}$.
    \item $H^{\ast}$: A mixture probability measure given by $\frac{a}{a+n}H+\frac{n}{a+n}F_{\mathbf{x}_{1:n}}$.
    \item $H$: The base measure of the DP prior $\text{DP}(a,H)$.
    \item $F_{\mathbf{x}_{1:n}}$: The empirical probability measure based on the sample data defined by $\frac{1}{n}\sum_{i=1}^{n}\delta_{\mathbf{X}_{i}}$.
    \item $F_{N}^{\text{Pos}}:$ The probability measure of the DP posterior approximation defined by $\sum_{i=1}^{N}J^{\text{Pos}}_{i,N}\delta_{\mathbf{X}^{\text{Pos}}_{i}}$ with
    $\left(\mathbf{X}^{\text{Pos}}_{1:N}\right)\sim H^{\ast}$ and $\left( J^{\text{Pos}}_{1:N,N}\right)\sim \mbox{Dirichlet}((a+n)/N,\ldots,(a+n)/N)$ \cite{Ishwaran}.
    \item $f_{i}(\mathbf{X})$: A continuous function of a random variable $\mathbf{X}\sim F$, denoted as $\mathbf{X}^{\prime}_{i}$, for $i=1,2$.
    \item $\{T_{\boldsymbol{\gamma}}\}_{\boldsymbol{\gamma}\in\boldsymbol{\Gamma}}$: A set of continuous functions parameterized by a neural network on a compact domain $\boldsymbol{\Gamma}$.
    \item $I_d$: Identity matrix of size $d \times d$.
    \item ``a.s.'': Standing for ``almost surely'' and indicates that the statements hold with probability 1.
\end{itemize}

\section{Definition of MMD and Wasserstein distances: DP and empirical representations}\label{app:distance}
Selecting an appropriate statistical distance is crucial for effective generative model training. Here, we focus on the DP representation of two popular distances used in BNP deep learning and we will briefly mention their frequentist counterparts.

\subsection{Maximum mean discrepancy distance (feature-matching comparison)}\label{app:MMD}
The MMD distance was initially introduced in \cite{Gretton} for frequentist two-sample comparisons. Recently, a DP-based version of this distance has been proposed for BNP hypothesis testing \cite{fazeli2023semi}. Consider a set of functions $\{G_{\boldsymbol{\omega}}\}_{\boldsymbol{\omega}\in\boldsymbol{\Omega}}$ parameterized by a neural network that can generate $n$ IID random variables $\left(\mathbf{Y}_{1:n} \right)$, where the likelihood function is intractable and not accessible. Let $k(\cdot,\cdot)$ be a continuous kernel function with a feature space corresponding to a universal reproducing kernel Hilbert space \cite{Gretton} defined on a compact sample space $\mathfrak{X}$ \cite{Gretton}. Given a sample $\left(\mathbf{X}_{1:n} \right)\overset{\text{IID}}{\sim}F$, the MMD distance between $F^{\text{Pos}}$ and $F_{\mathbf{Y}_{1:n}}$ is approximated as:
\begin{small}
\begin{multline}\label{BNP-pos-MMD}
\text{MMD}^2(F_{N}^{\text{Pos}},F_{\mathbf{Y}_{1:n}}):= \sum_{\ell,t=1}^{N} J^{\text{Pos}}_{\ell,N}J^{\text{Pos}}_{t,N}k(\mathbf{X}^{\text{Pos}}_{\ell},\mathbf{X}^{\text{Pos}}_{t}) \\
- \dfrac{2}{n}\sum_{\ell=1}^{N}\sum_{t=1}^{n} J^{\text{Pos}}_{\ell,N}k(\mathbf{X}^{\text{Pos}}_{\ell},\mathbf{Y}_{t})
+ \dfrac{1}{n^2}\sum_{\ell,t=1}^{n} k(\mathbf{Y}_{\ell},\mathbf{Y}_{t}).
\end{multline}
\end{small}
In the frequentist version of the MMD distance, as defined in \cite{Gretton}, the empirical distribution $F_{\mathbf{X}_{1:n}}$ is considered. This is denoted as $\text{MMD}^2(F_{\mathbf{X}_{1:n}},F_{\mathbf{Y}_{1:n}})$, which is obtained by replacing $N$, $J_{1:N,N}^{\text{Pos}}$, and $\mathbf{X}^{\text{Pos}}_{1:N}$ with $n$, $1/n$, and $\mathbf{X}_{1:n}$, respectively, in Eq. \eqref{BNP-pos-MMD}.

\subsection{Wasserstein distance (overall distribution comparison)}\label{app:Wasserstein}
The frequentist version of the Wasserstein distance is completely discussed in \cite[Part I6]{villani2008optimal}. 
Fazeli-Asl et al. \cite{fazeli2023bayesian} proposed a BNP version of this distance through its Kantorovich-Rubinstein dual representation. 
Let $\lbrace D_{\boldsymbol{\theta}}\rbrace_{{\boldsymbol{\theta}}\in \Theta}$ be a parametrized family of continuous functions that all are 1-Lipschitz. Then the Wasserstein distance between $F^{\text{Pos}}$ and $F_{\mathbf{Y}_{1:n}}$ is approximated as:
\begin{small}
    \begin{align}\label{W-BNP}
    \text{WS}(F_{N}^{\text{Pos}},F_{\mathbf{Y}_{1:n}}):=\max\limits_{\Theta} 
    \sum_{i=1}^{N}\left( J^{\text{Pos}}_{i,N}D_{\theta}(\mathbf{X}_{i}^{\text{Pos}})
    -\dfrac{D_{\boldsymbol{\theta}}(\mathbf{Y}_i)}{n}\right).
\end{align}
\end{small}
By utilizing modifications in \eqref{W-BNP} similar to those described for the MMD distance, the empirical representation of the Wasserstein distance can also be obtained.
In this section, we propose two novel representations for the MINE using the DP. These representations are based on the KL and JS divergences will be used in our BNP learning framework to maximize information during the training process.

\section{Baseline Models}\label{app:baseline}
In this section, we present details of the two baseline models used in our experiments. Their cost functions are defined in terms of the general expectation of random variables. However, in practice, all expectations appearing in their objective functions are approximated using empirical distributions.
\subsection{$\alpha$-WGAN+Gradient penalty}
The \(\alpha\)-WGAN+GP, introduced by \cite{kwon2019generation}, is a VAE-GAN--a hybrid generative model that combines Variational Autoencoders (VAEs) and Generative Adversarial Networks (GANs)--specifically designed for data synthesis.
 It incorporates the encoder--$\{E_{\boldsymbol{\eta}}\}_{\boldsymbol{\eta}\in\boldsymbol{\mathcal{H}}}$, the generator--$\{G_{\boldsymbol{\omega}}\}_{\boldsymbol{\omega}\in\boldsymbol{\Omega}}$, the code discriminator--$\{CD_{\boldsymbol{\theta}^{\prime}}\}_{\boldsymbol{\theta}^{\prime}\in\boldsymbol{\Theta}^{\prime}}$, the discriminator--$\{D_{\boldsymbol{\theta}}\}_{\boldsymbol{\theta}\in\boldsymbol{\Theta}}$, and the Wasserstein GAN with gradient penalty framework. For training dataset $\mathbf{X}\sim F$, the model is trained by updating the networks’ parameters according to the
following hybrid objective function:
	\begin{align}
		(\widehat{\boldsymbol{\omega}}, \widehat{\boldsymbol{\eta}})&=\arg\min\limits_{\boldsymbol{\boldsymbol{\Omega},\boldsymbol{\mathcal{H}}}}\lbrace-\mathbb{E}_{F_{\boldsymbol{c}}}[D_{\boldsymbol{\theta}}(G_{\boldsymbol{\omega}}(\boldsymbol{c}))]
		-\mathbb{E}_{F_{\boldsymbol{\xi}}}[D_{\boldsymbol{\theta}}(G_{\boldsymbol{\omega}}(\boldsymbol{\xi}))]+\lambda_1\mathbb{E}_{F_{\boldsymbol{c}}}\left \| \mathbf{X}-G_{\boldsymbol{\omega}}(\boldsymbol{c}) \right \|_{1}\rbrace,\nonumber\\
		\widehat{\boldsymbol{\theta}}&=\arg\min\limits_{\boldsymbol{\Theta}}\lbrace\mathbb{E}_{F_{\boldsymbol{c}}}[D_{\boldsymbol{\theta}}(G_{\boldsymbol{\omega}}(\boldsymbol{c}))]
		+\mathbb{E}_{F_{\boldsymbol{\xi}}}[D_{\boldsymbol{\theta}}(G_{\boldsymbol{\omega}}(\boldsymbol{\xi}))]
		-2\mathbb{E}_F[D_{\boldsymbol{\theta}}(\mathbf{x})]+\lambda_2L_{\text{GP-D}}\rbrace,\label{Dis_loss}\\
		\widehat{\boldsymbol{\theta}^{\prime}}&=\arg\min\limits_{\boldsymbol{\Theta}^{\prime}}\lbrace\mathbb{E}_{F_{\boldsymbol{c}}}[CD_{\boldsymbol{\theta}^{\prime}}(\boldsymbol{c})]
		-\mathbb{E}_{F_{\boldsymbol{\xi}}}[CD_{\boldsymbol{\theta}^{\prime}}(\boldsymbol{\xi})]
		+\lambda_2L_{\text{GP-CD}}\rbrace,\label{CDis}
	\end{align}
	where \(\boldsymbol{\xi}\) is a noise vector sampled from the distribution \(F_{\boldsymbol{\xi}}\), and \(\boldsymbol{c} := E_{\boldsymbol{\eta}}(\mathbf{X})\) represents the latent representation of \(\mathbf{X}\). The term \(\lambda_1 \mathbb{E}_{F_{\boldsymbol{c}}} \left\| \mathbf{X} - G_{\boldsymbol{\omega}}(\boldsymbol{c}) \right\|_1\) denotes the reconstruction loss, inferred by modeling the data distribution with a Laplace distribution. Additionally, the gradient penalty terms \(L_{\text{GP-D}}\) and \(L_{\text{GP-CD}}\), scaled by the coefficient \(\lambda_2\), are added to Eqs. \eqref{Dis_loss} and \eqref{CDis}, respectively, to enforce the 1-Lipschitz constraint on the discriminators.

\subsection{BiGAN}
BiGAN, independently introduced by \cite{donahue2016adversarial} and \cite{dumoulin2016adversarially}, is a variant of the VAE-GAN framework designed for data generation. In BiGAN, the discriminator distinguishes between pairs of data and latent codes—either real data paired with its encoded representation or generated data paired with its sampled latent vector. This bidirectional structure enables the model to simultaneously learn a mapping from data to the latent space and from the latent space back to the data. The training of BiGAN involves optimizing the following objectives:
\begin{align*}
(\widehat{\boldsymbol{\omega}}, \widehat{\boldsymbol{\eta}})&= \arg\min\limits_{\boldsymbol{\Omega},\boldsymbol{\mathcal{H}}}\lbrace -\mathbb{E}_{F_{\boldsymbol{\xi}}}[\ln D_{\boldsymbol{\theta}}(G_{\boldsymbol{\omega}}(\boldsymbol{\xi}), \boldsymbol{\xi})] - \mathbb{E}_{F}[\mathbb{E}_{F_{\boldsymbol{c}}}[1-\ln D_{\boldsymbol{\theta}}(\mathbf{X}, \boldsymbol{c})]]\rbrace, \\
\widehat{\boldsymbol{\theta}} &= \arg\min\limits_{\boldsymbol{\Theta}}\lbrace-\mathbb{E}_{F}[\mathbb{E}_{F_{\boldsymbol{c}}}[\ln D_{\boldsymbol{\theta}}(\mathbf{X}, \boldsymbol{c})]] - \mathbb{E}_{F_{\boldsymbol{\xi}}}[\ln \left( 1 - D_{\boldsymbol{\theta}}(G_{\boldsymbol{\omega}}(\boldsymbol{\xi}), \boldsymbol{\xi}) \right)]\rbrace.
\end{align*}

\section{Theoretical proofs}\label{app:technical proofs}
\subsection{\textbf{Proof of Theorem 1}}
\begin{theorem}[Limiting expectation]\label{thm-asmp-dpdv}
    Considering DP posterior representations of MINE presented in the main paper. Given the DP posterior approximation $F^{\text{Pos}}_N$, we have,  
    \begin{itemize}
        \item[$i.$] 
        $\lim_{n,N \to \infty}\mathbb{E}_{F_{N}^{\text{Pos}}}\left(\mathcal{L}_{\boldsymbol{\gamma}}^{\text{DPDV}}(f_1(\mathbf{X}^{\text{Pos}}_{1:N}),f_2(\mathbf{X}^{\text{Pos}}_{1:N})) \right) \geq \mathcal{L}_{\boldsymbol{\gamma}}^{\text{DV}}(\mathbf{X}_{1}^{\prime},\mathbf{X}_{2}^{\prime})$, a.s.,
        \item[$ii.$] $\mathbb{E}_{F_{N}^{\text{Pos}}}\left(\mathcal{L}_{\boldsymbol{\gamma}}^{\text{DPJS}}(f_1(\mathbf{X}^{\text{Pos}}_{1:N}),f_2(\mathbf{X}^{\text{Pos}}_{1:N})) \right)$ converges a.s. to $\mathcal{L}_{\boldsymbol{\gamma}}^{\text{JS}}(\mathbf{X}_{1}^{\prime},\mathbf{X}_{2}^{\prime})$, as $n,N\rightarrow\infty,$
    \end{itemize}
\end{theorem}
\begin{proof}
Recall that
\begin{small}
\begin{multline}\label{DPDV-lower}
    \mathcal{L}_{\boldsymbol{\gamma}}^{\text{DPDV}}(f_1(\mathbf{X}^{\text{Pos}}_{1:N}),f_2(\mathbf{X}^{\text{Pos}}_{1:N})):=\sum_{\ell=1}^{N}J_{\ell,N}^{\text{Pos}}T_{\boldsymbol{\gamma}}(f_1(\mathbf{X}^{\text{Pos}}_{\ell}),f_2(\mathbf{X}^{\text{Pos}}_{\ell})) 
    -\ln\sum_{\ell=1}^{N}J_{\ell,N}^{\text{Pos}}e^{T_{\boldsymbol{\gamma}}(f_1(\mathbf{X}^{\text{Pos}}_{\ell}),f_2(\mathbf{X}^{\text{Pos}}_{\pi(\ell)}))}.
\end{multline}
\end{small}

    Considering the property of Dirichlet distribution, $\mathbb{E}_{F^{\text{Pos}}}(J^{\text{Pos}}_{\ell,N})=1/N$, and then applying Jensen's inequality in the above equation implies 
\begin{small}    
\begin{multline}\label{Exp-0dpdv}
\mathbb{E}_{F^{\text{Pos}}}\left(\mathcal{L}_{\boldsymbol{\gamma}}^{\text{DPDV}}(f_1(\mathbf{X}^{\text{Pos}}_{1:N}),f_2(\mathbf{X}^{\text{Pos}}_{1:N})) \right)
\geq\sum_{\ell=1}^{N}\dfrac{1}{N}T_{\boldsymbol{\gamma}}(f_1(\mathbf{X}^{\text{Pos}}_{\ell}),f_2(\mathbf{X}^{\text{Pos}}_{\ell}))-\ln\sum_{\ell=1}^{N}\dfrac{1}{N}e^{T_{\boldsymbol{\gamma}}(f_1(\mathbf{X}^{\text{Pos}}_{\ell}),f_2(\mathbf{X}^{\text{Pos}}_{\pi(\ell)}))}\\
=I.
\end{multline}
\end{small}
As $n$ approaches infinity, the Glivenko-Cantelli theorem implies that $F_{\mathbf{x}_{1:n}}\xrightarrow{a.s.}F$, and subsequently, $H^{\ast}\xrightarrow{a.s.}F$. This implies that $\left(\mathbf{X}^{\text{Pos}}_{1:N}\right)$ converges to $\left(\mathbf{X}_{1:N}\right)$, a sample of $N$ random variables generated from $F$. Therefore, using the continuous mapping theorem as $n\rightarrow\infty$, we have the convergence:
\begin{small}
\begin{multline*}
I\xrightarrow{a.s.}\sum_{\ell=1}^{N}\dfrac{1}{N}T_{\boldsymbol{\gamma}}(f_1(\mathbf{X}_{\ell}),f_2(\mathbf{X}_{\ell}))
-\ln\sum_{\ell=1}^{N}\dfrac{1}{N}e^{T_{\boldsymbol{\gamma}}(f_1(\mathbf{X}_{\ell}),f_2(\mathbf{X}_{\pi(\ell)}))}
= \mathcal{L}_{\boldsymbol{\gamma}}^{\text{DV}}(f_1(\mathbf{X}_{1:N}),f_2(\mathbf{X}_{1:N})).
\end{multline*}
\end{small}
Here, we use $\mathcal{L}_{\boldsymbol{\gamma}}^{DV}(f_1(\mathbf{X}_{1:N}),f_2(\mathbf{X}_{1:N}))$ to indicate the empirical representation of the DV lower bound based on the $N$ samples $\left(\mathbf{X}_{1:N}\right)$. On the other hand, the law of the large number implies
\begin{align}\label{LLN}
    \mathcal{L}_{\boldsymbol{\gamma}}^{\text{DV}}(f_1(\mathbf{X}_{1:N}),f_2(\mathbf{X}_{1:N}))\xrightarrow{a.s.} \mathcal{L}_{\boldsymbol{\gamma}}^{\text{DV}}(\mathbf{X}_{1}^{\prime},\mathbf{X}_{2}^{\prime}),
\end{align}
as $N\rightarrow\infty$.

Now, subtract $\mathcal{L}{\boldsymbol{\gamma}}^{\text{DV}}(\mathbf{X}_{1}^{\prime},\mathbf{X}_{2}^{\prime})$ from both sides of \eqref{Exp-0dpdv} and take the limit as $n$ and $N$ approach infinity:
\begin{small}
\begin{align*}
\lim_{n, N \to \infty} \left(\mathbb{E}_{F^{\text{Pos}}}\left(\mathcal{L}{\boldsymbol{\gamma}}^{\text{DPDV}}(f_1(\mathbf{X}^{\text{Pos}}_{1:N}),f_2(\mathbf{X}^{\text{Pos}}_{1:N})) \right) - \mathcal{L}{\boldsymbol{\gamma}}^{\text{DV}}(\mathbf{X}_{1}^{\prime},\mathbf{X}_{2}^{\prime})\right)
\geq 0,
\end{align*}
\end{small}

which completes the proof of (i). A similar method is used to prove (ii) and it is then omitted.
\end{proof}

\subsection{\textbf{Proof of Theorem 2}}
\begin{theorem}[Consistency] 
    Considering BNP MINEs given in the main paper. Then,  for any label $\text{i}$ in $\lbrace \text{DV}, \text{JS}\rbrace$, as $n,N\rightarrow\infty$:
    \begin{itemize}
        \item[$i.$] $\text{MI}^{\text{DPi}}(f_1(\mathbf{X}^{\text{Pos}}_{1:N}),f_2(\mathbf{X}^{\text{Pos}}_{1:N}))\xrightarrow{a.s.}\text{MI}^{\text{i}}(\mathbf{X}_{1}^{\prime},\mathbf{X}_{2}^{\prime})$,
        \item[$ii.$] There exists a set of neural network $\{T_{\boldsymbol{\gamma}}\}_{\boldsymbol{\gamma}\in\boldsymbol{\Gamma}}$ on some compact domain $\boldsymbol{\Gamma}$ such that
        \begin{align}
            \text{MI}^{\text{DPi}}(f_1(\mathbf{X}^{\text{Pos}}_{1:N}),f_2(\mathbf{X}^{\text{Pos}}_{1:N}))\xrightarrow{a.s.}\text{MI}(\mathbf{X}_{1}^{\prime},\mathbf{X}_{2}^{\prime}).
        \end{align}
    \end{itemize}
\end{theorem}
\begin{proof}
We will only provide the proof for the DPDV estimator of the MI. The proof for the DPJS estimator is similar and therefore omitted.

To prove (i), following the proof of Theorem 1, we have $\left(\mathbf{X}^{\text{Pos}}_{1:N}\right)\xrightarrow{a.s.}\left(\mathbf{X}_{1:N}\right)$ as $n\rightarrow\infty$. Then, for any $\ell\in\lbrace 1,\ldots,N\rbrace$, the continuous mapping theorem implies:
\begin{align}
    T_{\boldsymbol{\gamma}}(f_1(\mathbf{X}^{\text{Pos}}_{\ell}),f_2(\mathbf{X}^{\text{Pos}}_{\ell}))&\xrightarrow{a.s.} T_{\boldsymbol{\gamma}}(f_1(\mathbf{X}_{\ell}),f_2(\mathbf{X}_{\ell}))\label{thm2:T},\\
     T_{\boldsymbol{\gamma}}(f_1(\mathbf{X}^{\text{Pos}}_{\ell}),f_2(\mathbf{X}^{\text{Pos}}_{\pi(\ell)}))&\xrightarrow{a.s.} T_{\boldsymbol{\gamma}}(f_1(\mathbf{X}_{\ell}),f_2(\mathbf{X}_{\pi(\ell)}))\label{thm2:T2},
\end{align}
as $n\rightarrow\infty$. 
On the other hand, following Fazeli-Asl et al. \cite[Theorem 1]{fazeli2023bayesian}, as $n\rightarrow\infty$, we have:
\begin{align}\label{thm2:weight}
J^{\text{Pos}}_{\ell,N}\xrightarrow{a.s.}\dfrac{1}{N}.
\end{align}
Now, considering results \eqref{thm2:T}, \eqref{thm2:T2}, and \eqref{thm2:weight} in the DP representation of the DV lower bound given in \eqref{DPDV-lower} when $n\rightarrow\infty$, we can imply:
\begin{align}\label{Thm2:dpdvConv}
\mathcal{L}_{\boldsymbol{\gamma}}^{\text{DPDV}}(f_1(\mathbf{X}^{\text{Pos}}_{1:N}),f_2(\mathbf{X}^{\text{Pos}}_{1:N}))\xrightarrow{a.s.} \mathcal{L}_{\boldsymbol{\gamma}}^{\text{DV}}(f_1(\mathbf{X}_{1:N}),f_2(\mathbf{X}_{1:N})).
\end{align}
Finally, considering the convergence in \eqref{LLN} on the right-hand side of Eq. \eqref{Thm2:dpdvConv} when $N\rightarrow\infty$, since $\max(\cdot)$ is a continuous function, the continuous mapping theorem concludes the proof of (i).

To prove (ii), the triangular inequality implies:
\begin{multline}\label{thm2:eps12}
    \left|\text{MI}^{\text{DPDV}}(f_1(\mathbf{X}^{\text{Pos}}_{1:N}),f_2(\mathbf{X}^{\text{Pos}}_{1:N}))-\text{MI}(\mathbf{X}_{1}^{\prime},\mathbf{X}_{2}^{\prime})\right|
    \\
    \leq
    \left| \text{MI}^{\text{DPDV}}(f_1(\mathbf{X}^{\text{Pos}}_{1:N}),f_2(\mathbf{X}^{\text{Pos}}_{1:N}))-\text{MI}^{\text{DV}}(\mathbf{X}_{1}^{\prime},\mathbf{X}_{2}^{\prime})\right|\\+
    \left| \text{MI}^{\text{DV}}(\mathbf{X}_{1}^{\prime},\mathbf{X}_{2}^{\prime})-\text{MI}(\mathbf{X}_{1}^{\prime},\mathbf{X}_{2}^{\prime})\right|.
\end{multline}
Regarding almost surely convergence in part (i), given $\epsilon_1>0$, there exists $N_0\in\mathbb{N}$ such that for all $n,N>N_0$,
\begin{align}\label{Thm2:eps1}
    \left| \text{MI}^{\text{DPDV}}(f_1(\mathbf{X}^{\text{Pos}}_{1:N}),f_2(\mathbf{X}^{\text{Pos}}_{1:N}))-\text{MI}^{\text{DV}}(\mathbf{X}_{1}^{\prime},\mathbf{X}_{2}^{\prime})\right|<\epsilon_1,~\text{a.s.}
\end{align}
On the other hand, following Belghazi et al. \cite[Lemma 1]{belghazi2018mutual} for given $\epsilon_2>0$, there exists a set of neural network $\{T_{\boldsymbol{\gamma}}\}_{\boldsymbol{\gamma}\in\boldsymbol{\Gamma}}$ on some compact domain $\boldsymbol{\Gamma}$ such that
\begin{align}\label{Thm2:eps2}
     \left| \text{MI}^{\text{DV}}(\mathbf{X}_{1}^{\prime},\mathbf{X}_{2}^{\prime})-\text{MI}(\mathbf{X}_{1}^{\prime},\mathbf{X}_{2}^{\prime})\right|<\epsilon_2,~\text{a.s.}
\end{align}
Now using \eqref{Thm2:eps1} and \eqref{Thm2:eps2} in \eqref{thm2:eps12} with choosing $\epsilon_1=\epsilon_2=\epsilon/2$ completes the proof. 
\end{proof}
\section{Additional results and implementing details}
\subsection{synthetic example}
\paragraph{Stanford Bunny Dataset:}\label{app:additional-results-bunny}
We use the Stanford Bunny dataset, available at \url{https://graphics.stanford.edu/data/3Dscanrep/}, to investigate the impact of DPMINE on an additional synthetic example. The dataset consists of point clouds representing the Stanford Bunny, a renowned 3D model provided by the Stanford University Computer Graphics Laboratory. These point clouds, captured with the Cyberware 3030 MS scanner and stored in PLY files (Polygon File Format) developed at Stanford, represent spatial locations on the object's surface. Point clouds provide a 3D spatial representation of the object, enabling detailed visualization and analysis. For our study, we use three point clouds from this dataset and apply filtering with the \emph{pyoints} Python library to extract points in 3D space. 

We randomly sampled 5000 points from the available 43,188 points as the training dataset and implemented our model on them by feeding the model with noise inputs to generate 2500 random samples and with encoded real inputs to obtain reconstruction samples. Figure \ref{bunny-DecGen} illustrates the significant impact of DPMINE on the performance of the BNP VAE-GAN in data generation.


\subsection{Real examples}
\paragraph{Covid-19 dataset:}\label{app:additional-results-covid}

\begin{figure*}[!t]
\centering\begin{tabular}[h!]{|ccc|}
\hline
\multicolumn{2}{|l|}{\rotatebox{90}{BNPWMMD}}&\multicolumn{1}{r|}{\includegraphics[width=.6\textwidth,height=4cm]{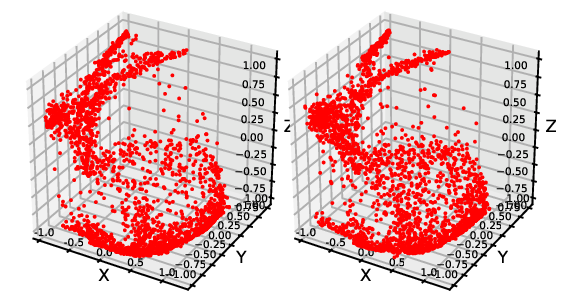}}\\
\hline
\multicolumn{1}{|c|}{\rotatebox{90}{BNPWMMD+DPMINE}}&\multicolumn{2}{c|}{\includegraphics[width=.9\textwidth,height=4cm]{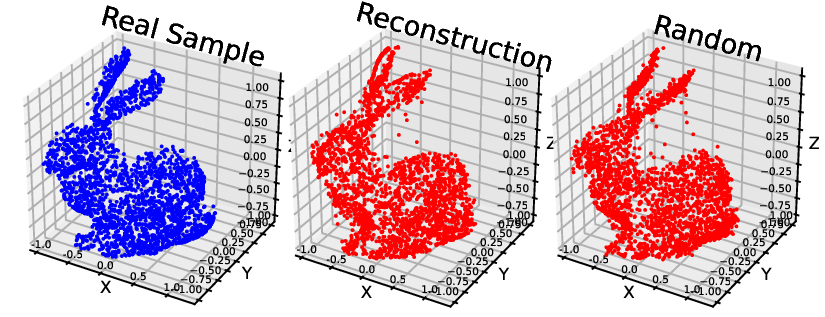}}\\
\hline
\multicolumn{2}{|l|}{\rotatebox{90}{$\alpha$-WGAN+MINE}}&\multicolumn{1}{r|}{\includegraphics[width=.6\textwidth,height=4cm]{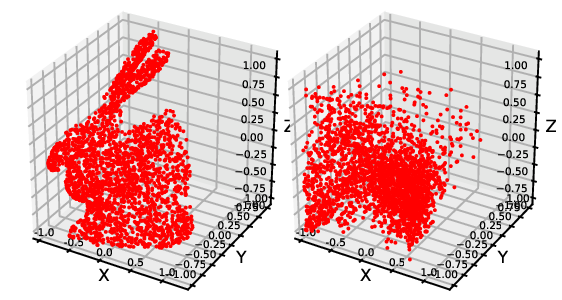}}\\
\hline
\multicolumn{2}{|l|}{\rotatebox{90}{BiGAN+MINE}}&\multicolumn{1}{r|}{\includegraphics[width=.6\textwidth,height=4cm]{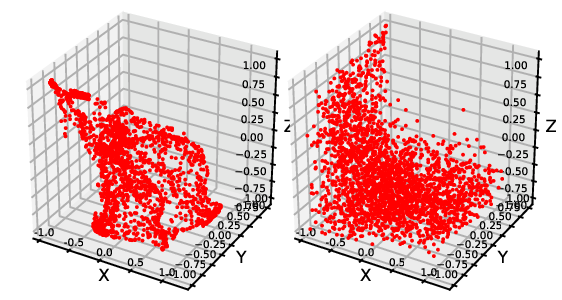}}\\
\hline
\end{tabular}
\caption{2500 samples for the Stanford Bunny using BNPWMMD+DPMINE, $\alpha$-WGAN+MINE, and BiGAN+MINE after 5000 epochs.}
\label{bunny-DecGen}
\end{figure*}

The red border in Figure \ref{random-slices} indicates the corresponding slices depicted in Figure 6 of the experimental findings discussed in the main paper. This clearly demonstrates the excellent performance of BNPWMMD+DPMINE in displaying sharp and diverse slices in a 3D randomly generated sample, surpassing the performance of other methods. Additionally, Figure \ref{slice-reconstruct} includes reconstructed samples that exhibit the highest similarity to the training dataset for the BNPWMMD+DPMINE method.

\begin{figure*}[htbp]
\centering\begin{tabular}[h!]{|c||c|c|c|}
\hline
\multicolumn{2}{|c|}{Axial}&Sagittal &Coronal\\
\hline
\multirow{1}{*}[40pt]{\rotatebox{90}{{\scriptsize Real Dataset}}}&\includegraphics[width=.29\textwidth,height=2.1cm]{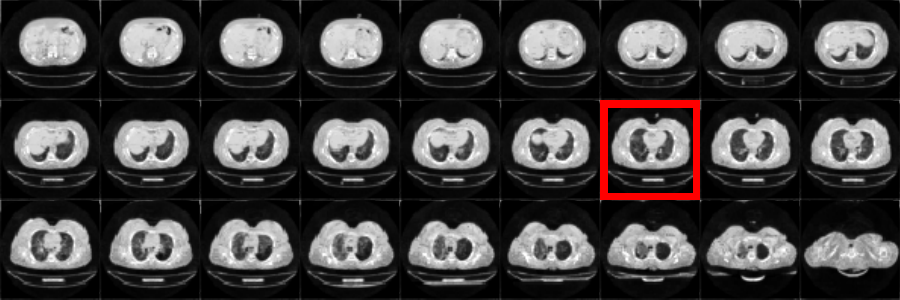}&\includegraphics[width=.29\textwidth,height=2.1cm]{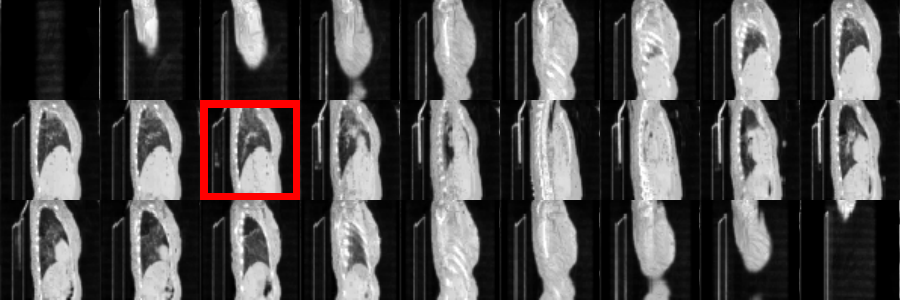}&\includegraphics[width=.29\textwidth,height=2.1cm]{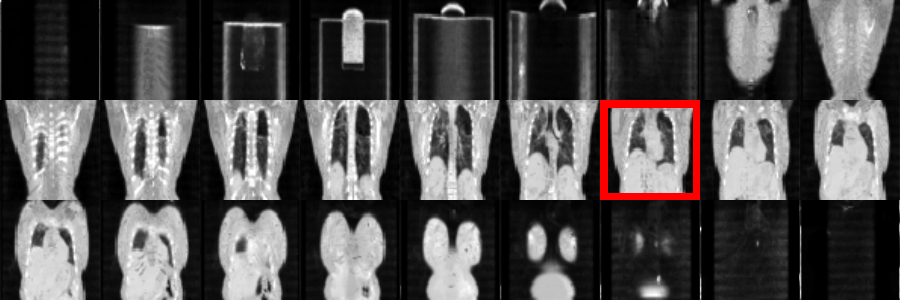}\\
\hline
\hline
\rotatebox{90}{{\scriptsize BNPWMMD+DPMINE}}&\includegraphics[width=.29\textwidth,height=2.1cm]{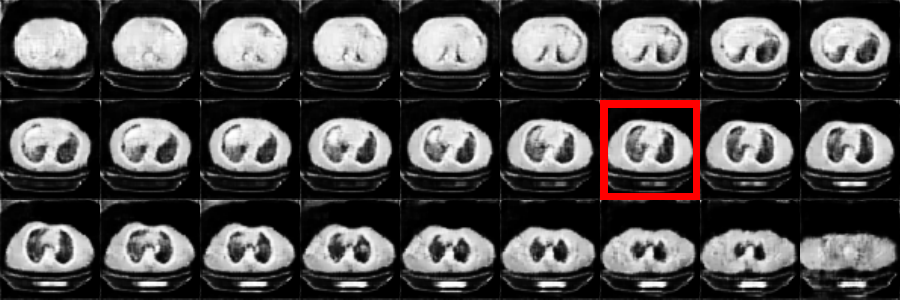}&\includegraphics[width=.29\textwidth,height=2.1cm]{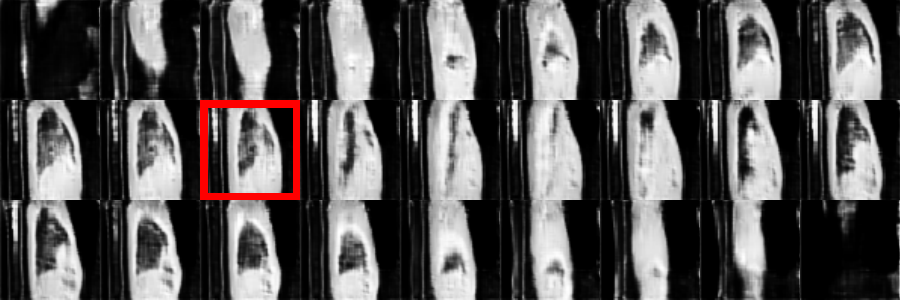}&\includegraphics[width=.29\textwidth,height=2.1cm]{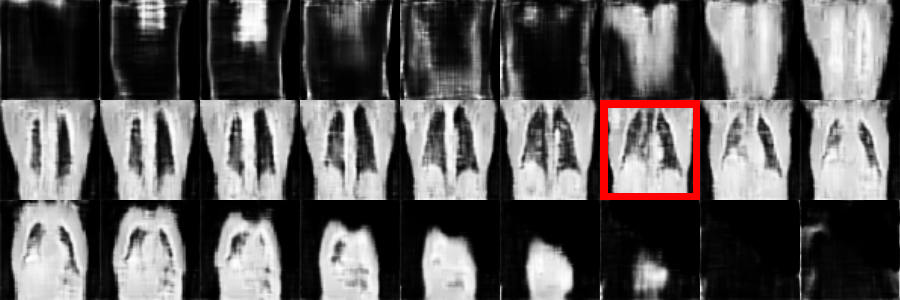}\\
\hline
\rotatebox{90}{{\scriptsize $\alpha$-WGAN+MINE}}&\includegraphics[width=.29\textwidth,height=2.1cm]{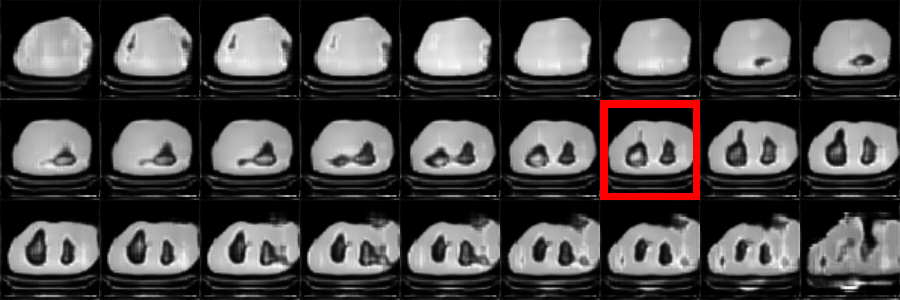}&\includegraphics[width=.29\textwidth,height=2.1cm]{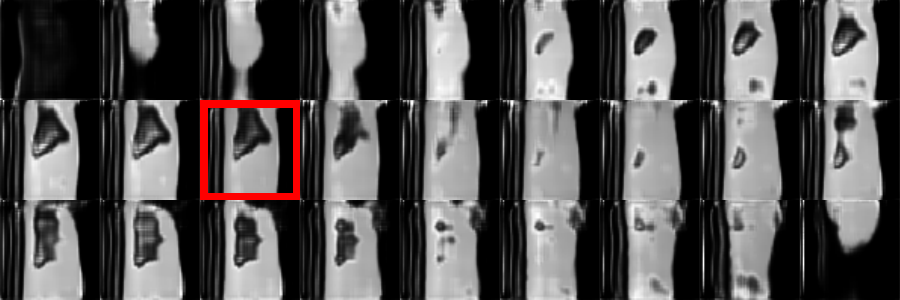}&\includegraphics[width=.29\textwidth,height=2.1cm]{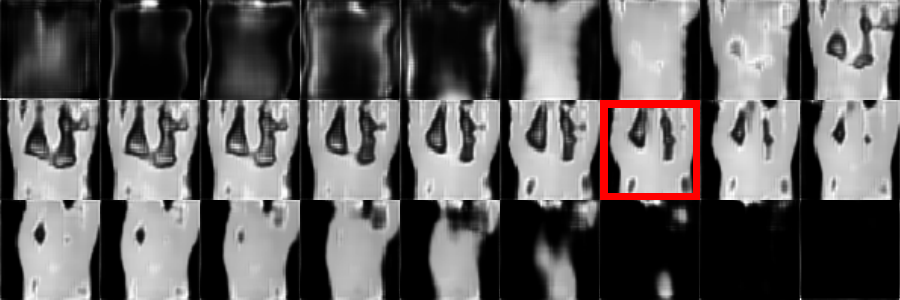}\\
\hline
\rotatebox{90}{{\scriptsize BiGAN+MINE}}&\includegraphics[width=.29\textwidth,height=2.1cm]{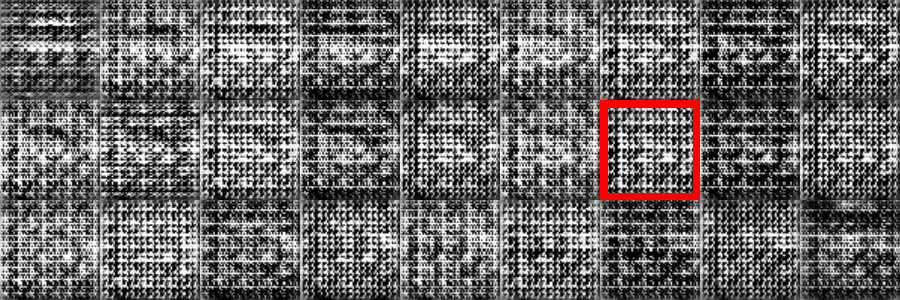}&\includegraphics[width=.29\textwidth,height=2.1cm]{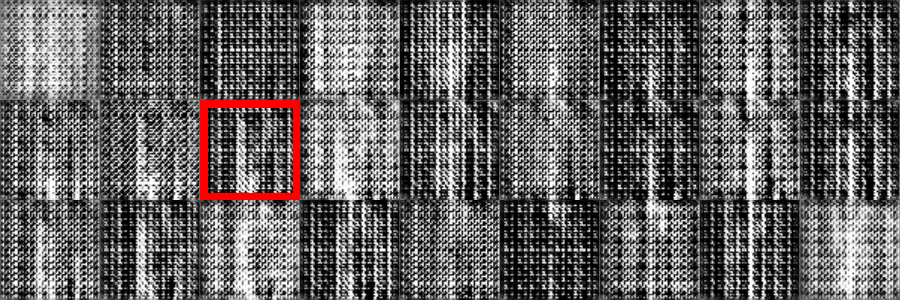}&\includegraphics[width=.29\textwidth,height=2.1cm]{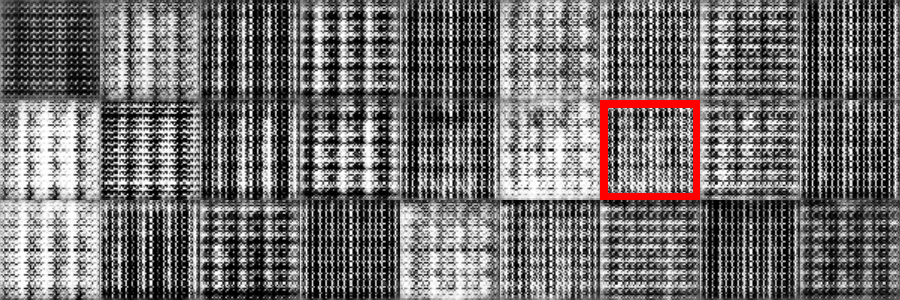}\\
\hline
\multicolumn{4}{|c|}{{\scriptsize Random Samples}}\\
\hline
\end{tabular}
\caption{27 slices of each side of a 3D sample randomly generated from BNPWMMD+DPMINE, BiGAN+MINE, and $\alpha$-WGAN+MINE after 7500 epochs for COVID-19 example.}
\label{random-slices}
\end{figure*}
\begin{figure*}[h!]
\centering\begin{tabular}[h!]{|c||c|c|c|}
\hline
\multicolumn{2}{|c|}{Axial}&Sagittal &Coronal\\
\hline
\rotatebox{90}{{\scriptsize BNPWMMD+DPMINE}}&\includegraphics[width=.29\textwidth,height=2.1cm]{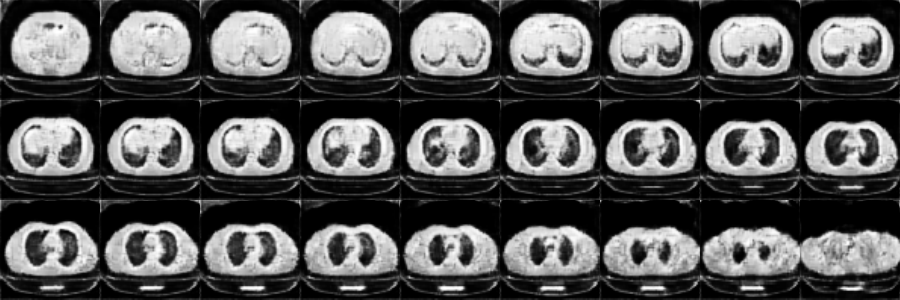}&\includegraphics[width=.29\textwidth,height=2.1cm]{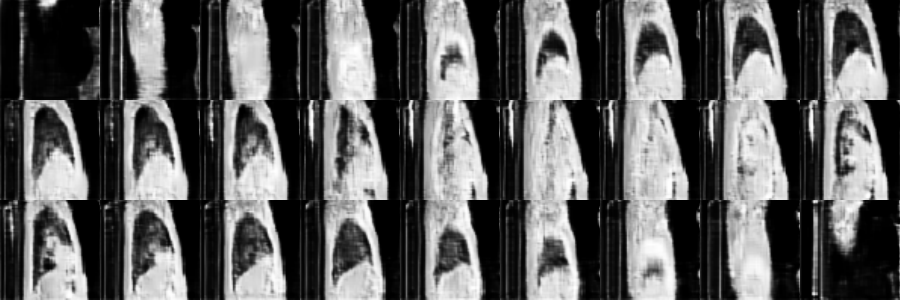}&\includegraphics[width=.29\textwidth,height=2.1cm]{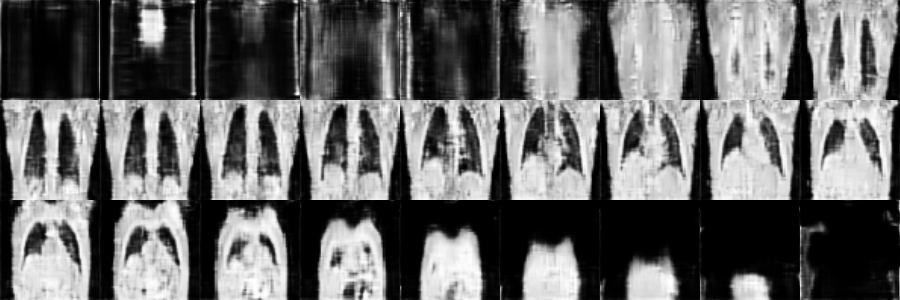}\\
\hline
\rotatebox{90}{{\scriptsize $\alpha$-WGAN+MINE}}&\includegraphics[width=.29\textwidth,height=2.1cm]{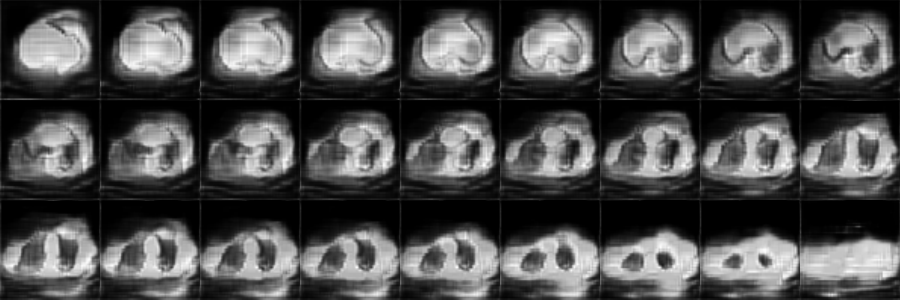}&\includegraphics[width=.29\textwidth,height=2.1cm]{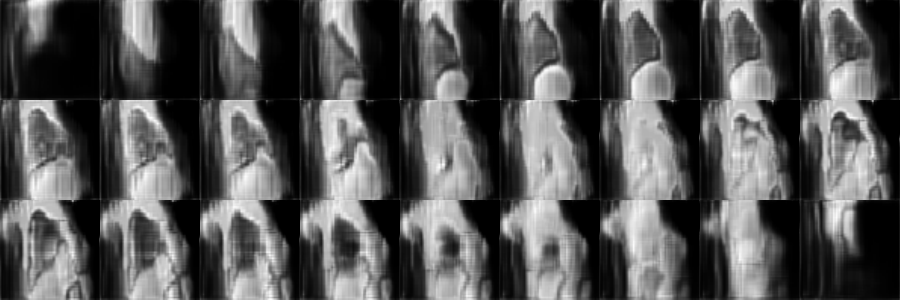}&\includegraphics[width=.29\textwidth,height=2.1cm]{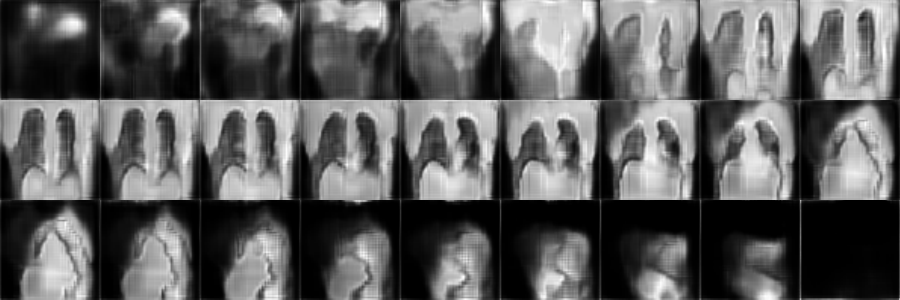}\\
\hline
\rotatebox{90}{{\scriptsize BiGAN+MINE}}&\includegraphics[width=.29\textwidth,height=2.1cm]{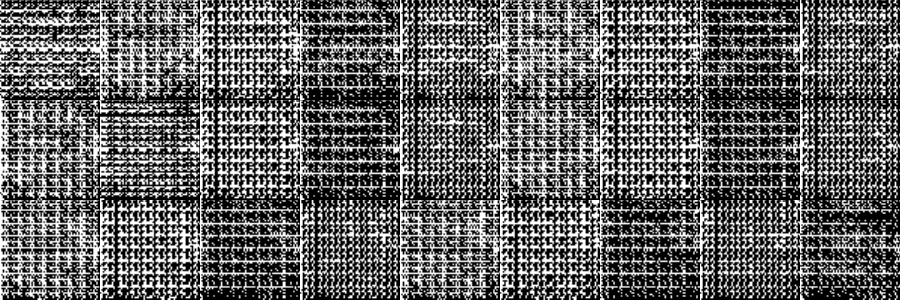}&\includegraphics[width=.29\textwidth,height=2.1cm]{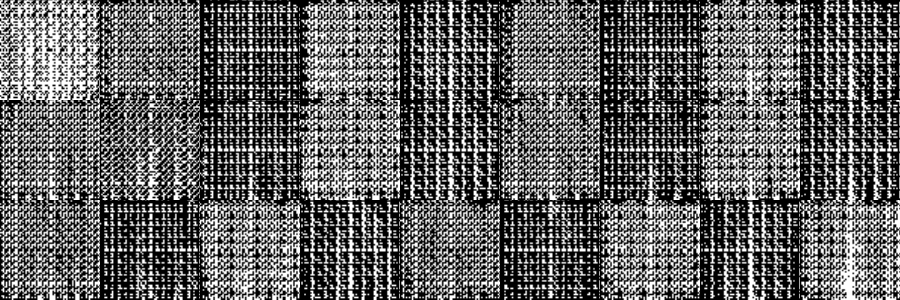}&\includegraphics[width=.29\textwidth,height=2.1cm]{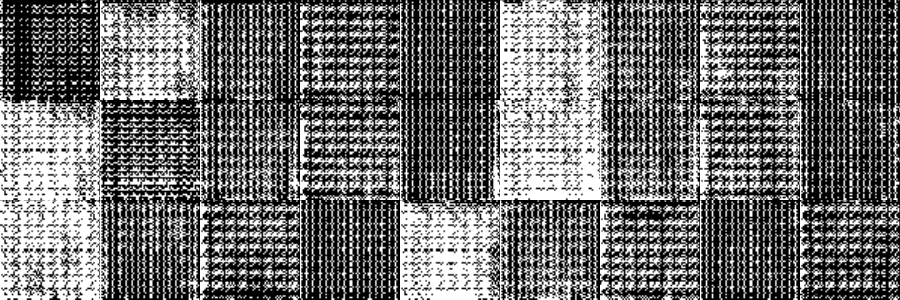}\\
\hline
\multicolumn{4}{|c|}{{\scriptsize Reconstructed Samples}}\\
\hline
\end{tabular}
\caption{27 slices of each side of a 3D sample reconstructed by BNPWMMD+DPMINE, BiGAN+MINE, and $\alpha$-WGAN+MINE after 7500 epochs for COVID-19 example.}
\label{slice-reconstruct}
\end{figure*}

\paragraph{BraTS 2018 Dataset:}
The BRATS 2018 dataset, a benchmark resource in medical imaging available at \url{https://www.med.upenn.edu/sbia/brats2018/data.html}, is employed for training models in the generation of brain tumor MRI scans. For the experiments, data from 210 subjects labeled as ``HGG" (High-Grade Glioma) are used, focusing on patients with aggressive brain tumors. Each subject's MRI data includes four distinct imaging modalities: T1-weighted (T1), T1-weighted with contrast enhancement (T1ce), T2-weighted (T2), and Fluid Attenuated Inversion Recovery (FLAIR). Notably, the FLAIR modality, which is particularly effective for highlighting edema and tumor boundaries, is used for the experiments, providing essential insights for synthetic MRI generation.

The results presented in Figures \ref{fig-cube-brat}-\ref{slice-reconstruct-brat} and Table \ref{table1Brat} reinforce our previous findings on the lung dataset, demonstrating similarly strong performance when applying the methodology to the BraTS dataset. This cross-domain evaluation further validates the model's robustness and adaptability to diverse data types.

\begin{figure}[ht]
\centering
\subfloat[{\tiny Real Dataset}]{\label{fig1}\includegraphics[width=.22\linewidth]{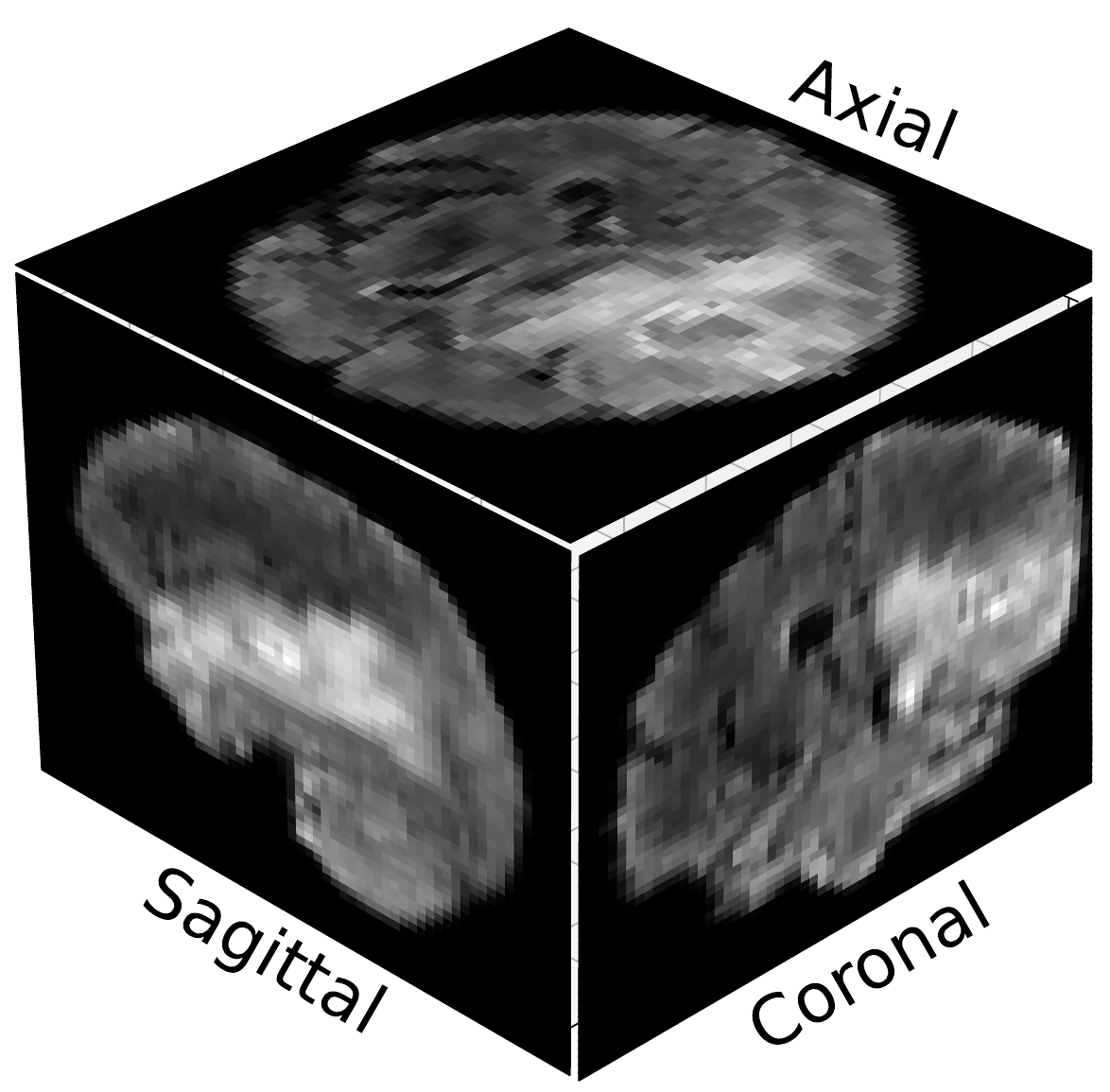}}\quad
\subfloat[{\tiny BNPWMMD+DPMINE}]{\label{fig2}\includegraphics[width=.22\linewidth]{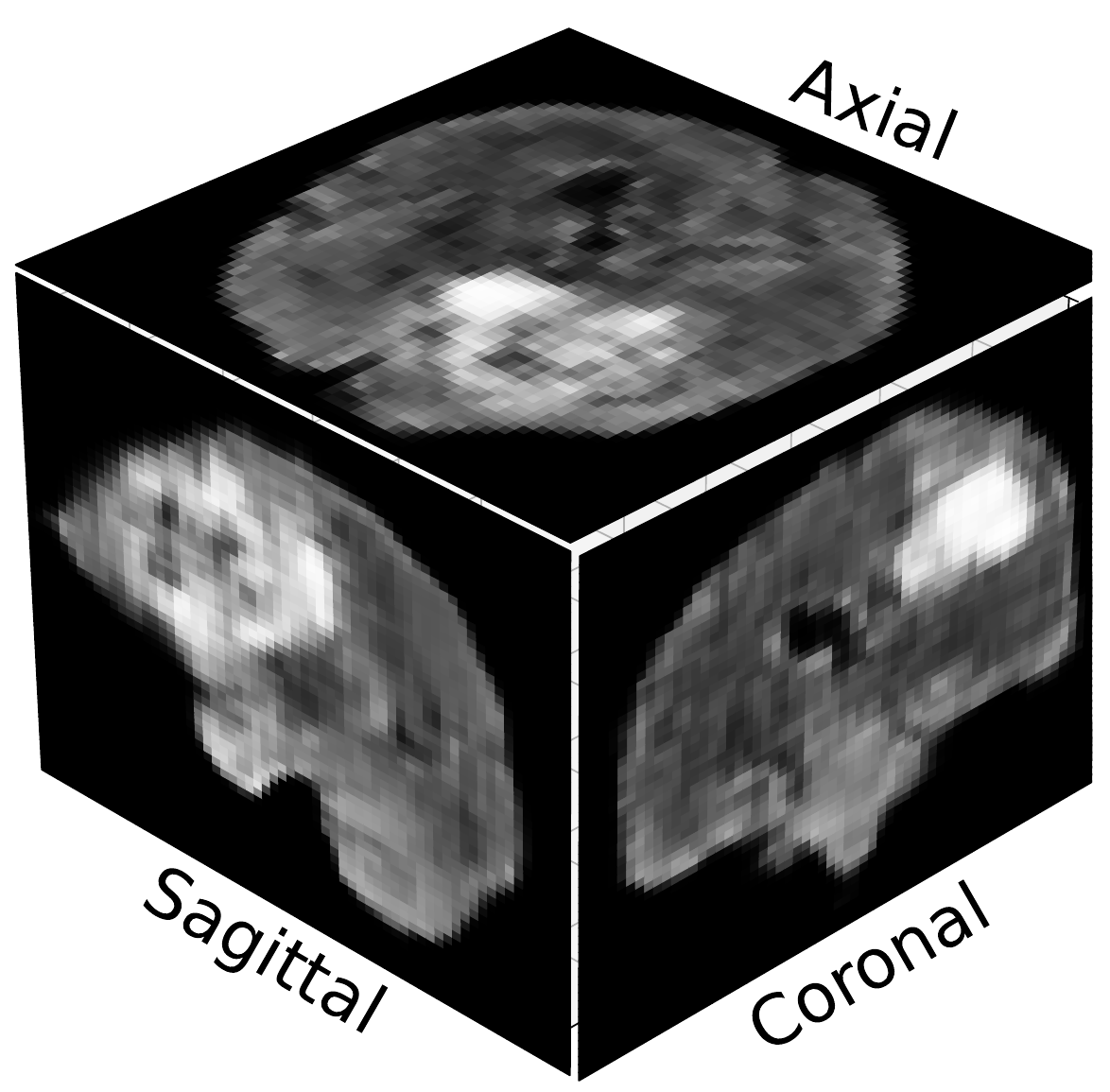}}\quad
\subfloat[{\tiny $\alpha$-WGAN+MINE}]{\label{fig3}\includegraphics[width=.22\linewidth]{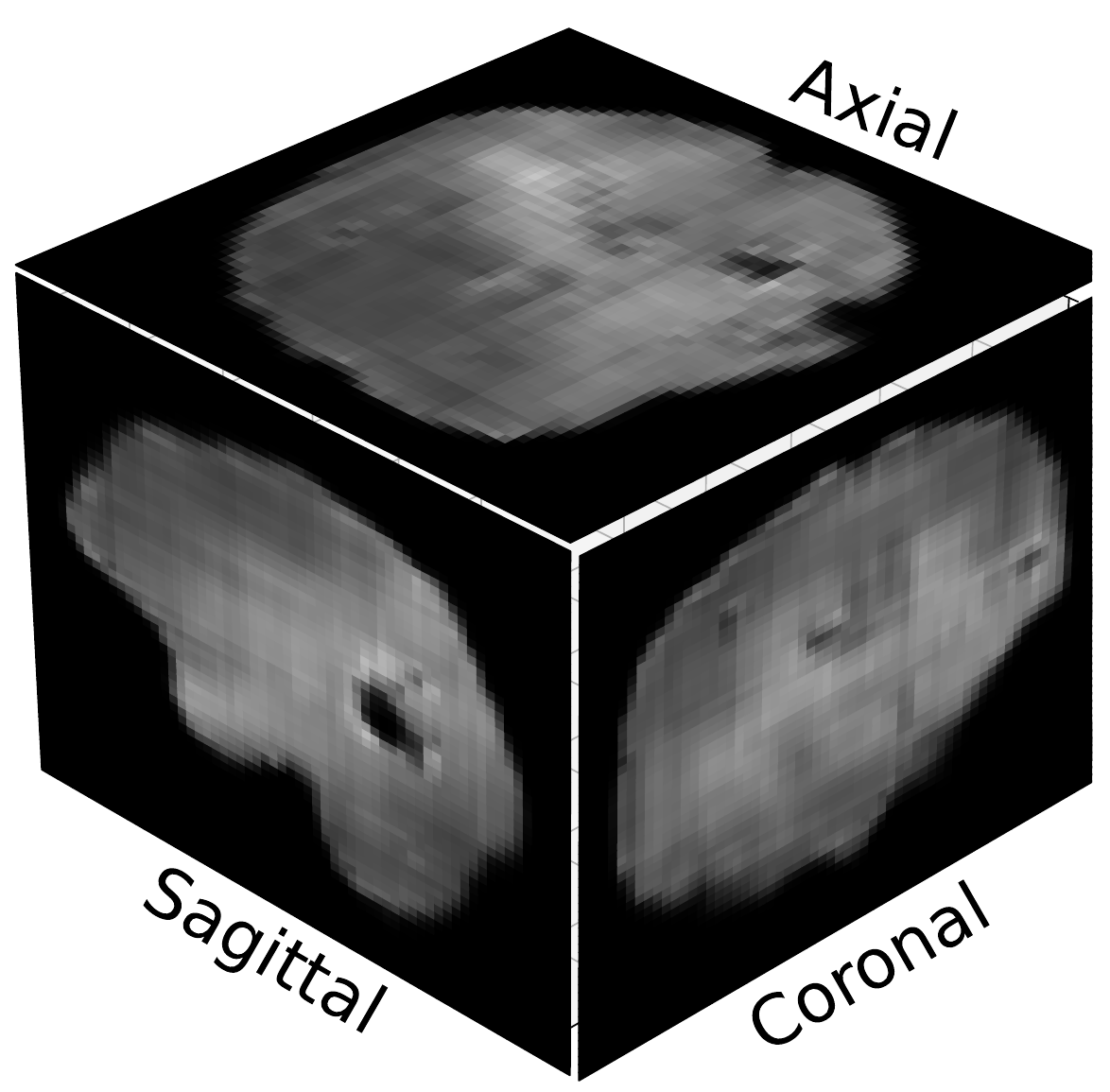}}\quad
\subfloat[{\tiny BiGAN+MINE}]{\label{fig4}\includegraphics[width=.22\linewidth]{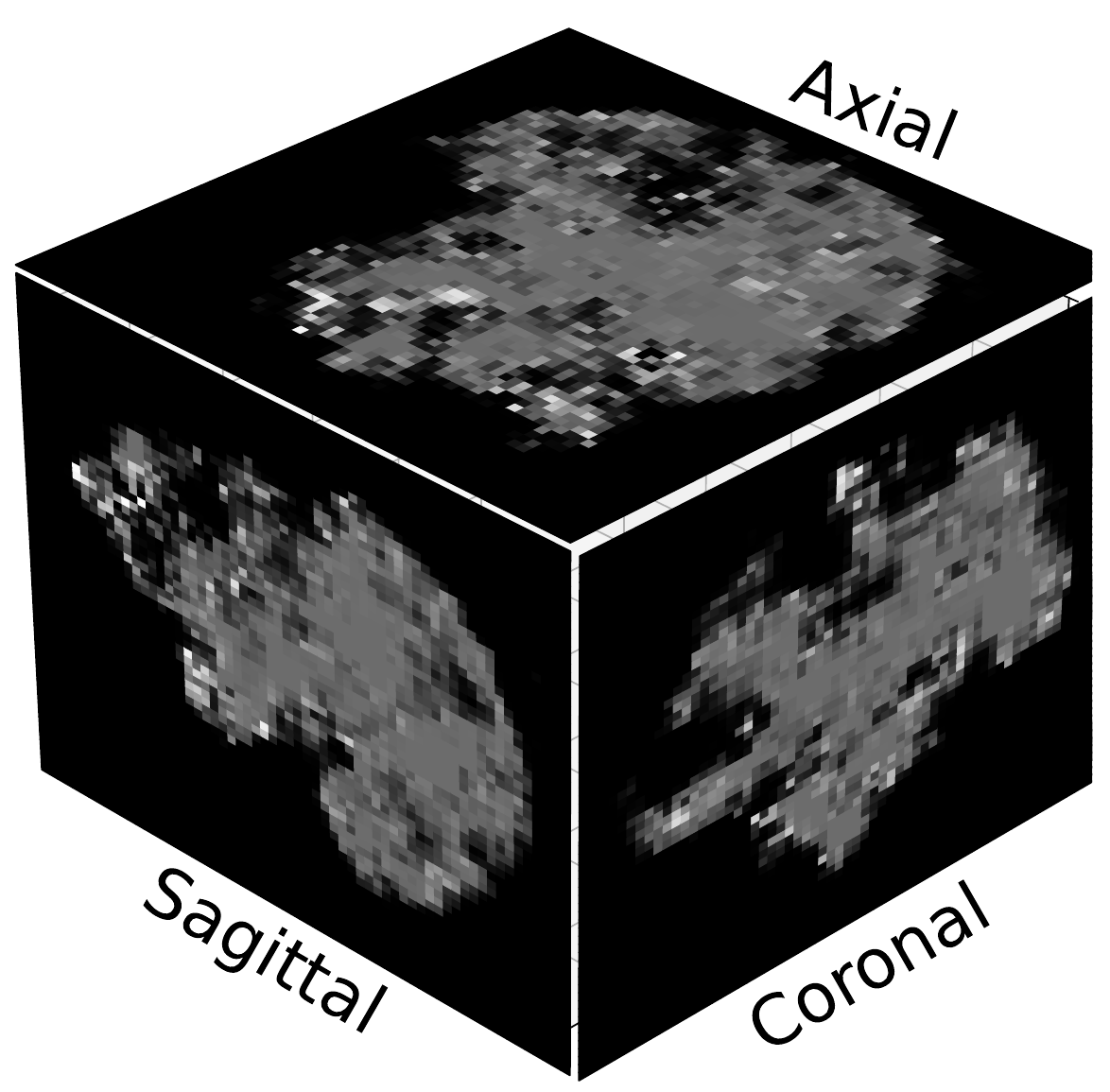}}
\caption{A 3D visualization of a real sample and a randomly generated sample in BraTS18 example.}
\label{fig-cube-brat}
\end{figure}

\begin{figure}[htbp]
  \centering
  \subfloat[{\small t-SNE}]{\includegraphics[width=.45\linewidth]{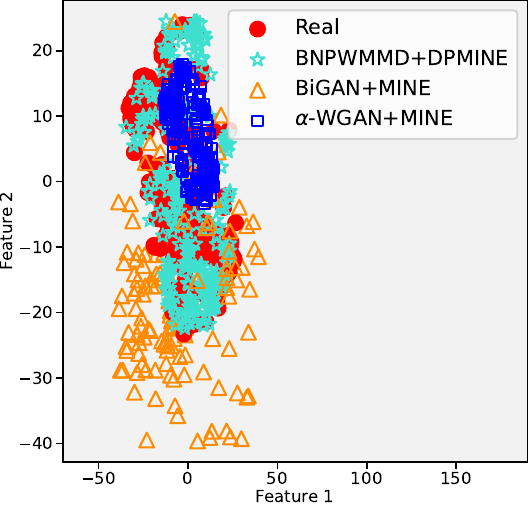}}\quad
  \subfloat[{\small Custom encoder}]{\includegraphics[width=.45\linewidth]{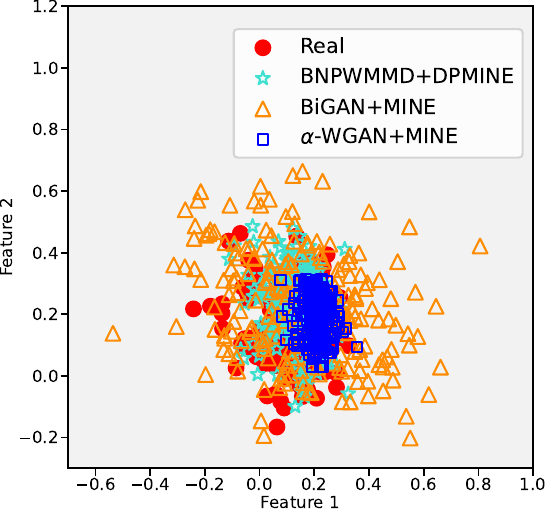}}
  \caption{Scatter plot of 2D features in BraTS18 example.}
  \label{fig1-tsne-brat}
\end{figure}

\begin{figure*}[htbp]
\centering\begin{tabular}[h!]{|c||c|c|c|}
\hline
\multicolumn{2}{|c|}{Axial}&Sagittal &Coronal\\
\hline
\multirow{1}{*}[40pt]{\rotatebox{90}{{\scriptsize Real Dataset}}}&\includegraphics[width=.29\textwidth,height=2.1cm]{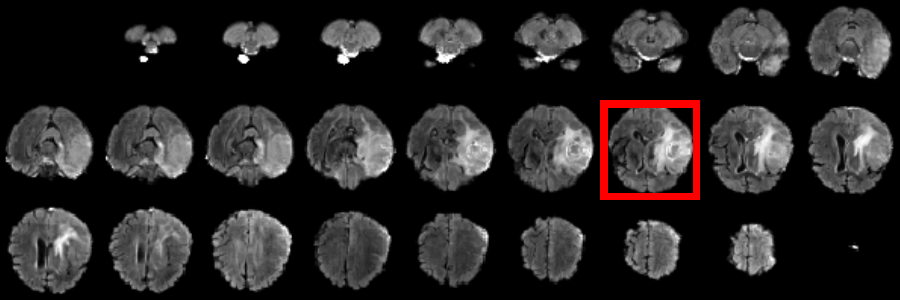}&\includegraphics[width=.29\textwidth,height=2.1cm]{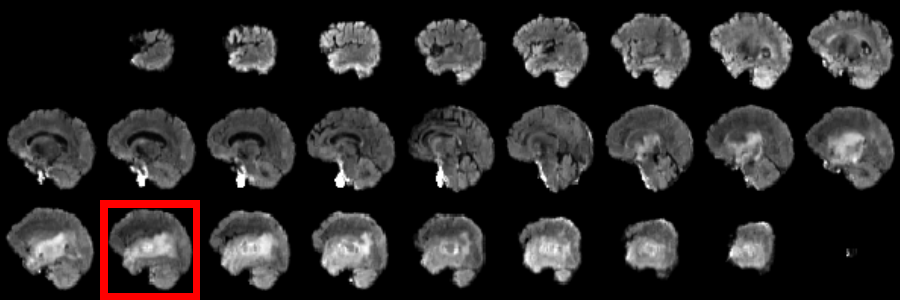}&\includegraphics[width=.29\textwidth,height=2.1cm]{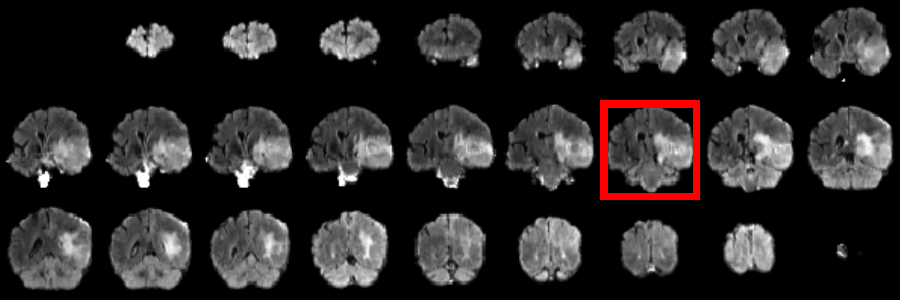}\\
\hline
\hline
\rotatebox{90}{{\scriptsize BNPWMMD+DPMINE}}&\includegraphics[width=.29\textwidth,height=2.1cm]{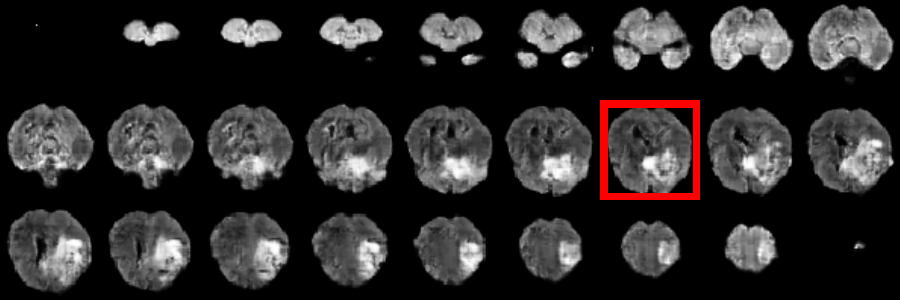}&\includegraphics[width=.29\textwidth,height=2.1cm]{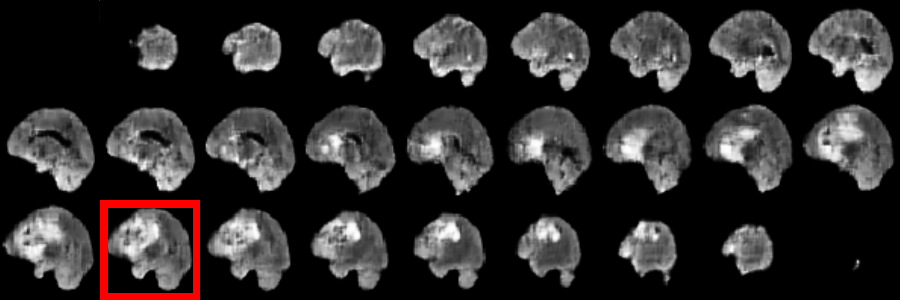}&\includegraphics[width=.29\textwidth,height=2.1cm]{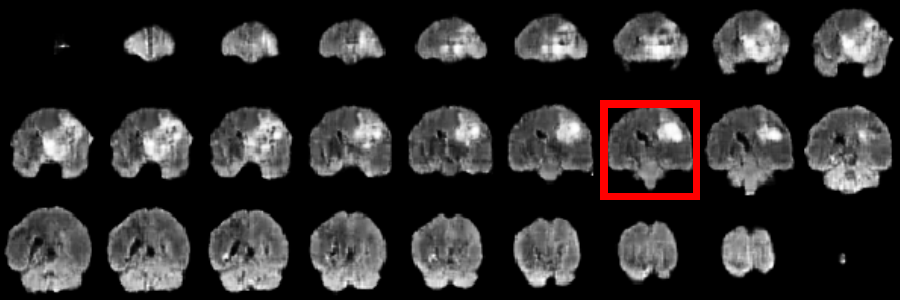}\\
\hline
\rotatebox{90}{{\scriptsize $\alpha$-WGAN+MINE}}&\includegraphics[width=.29\textwidth,height=2.1cm]{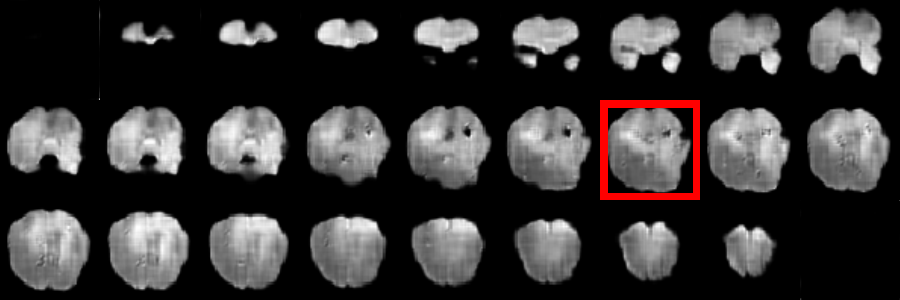}&\includegraphics[width=.29\textwidth,height=2.1cm]{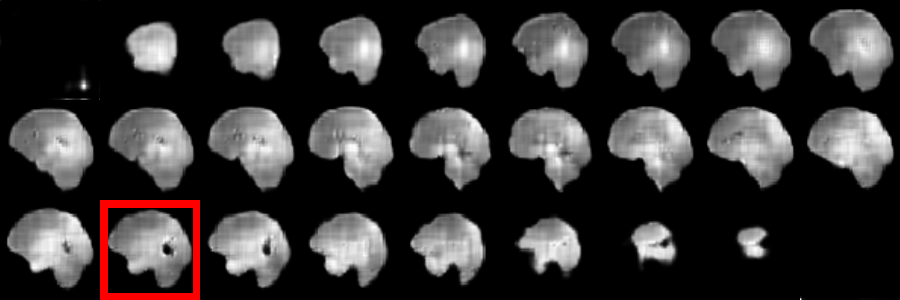}&\includegraphics[width=.29\textwidth,height=2.1cm]{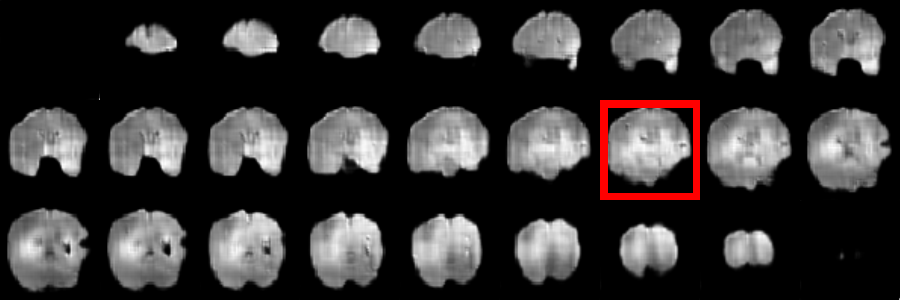}\\
\hline
\rotatebox{90}{{\scriptsize BiGAN+MINE}}&\includegraphics[width=.29\textwidth,height=2.1cm]{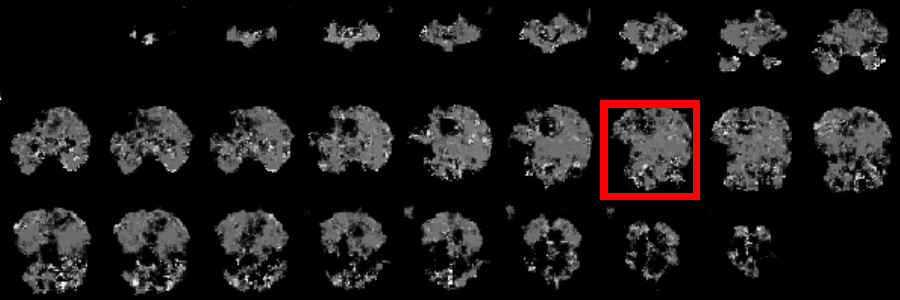}&\includegraphics[width=.29\textwidth,height=2.1cm]{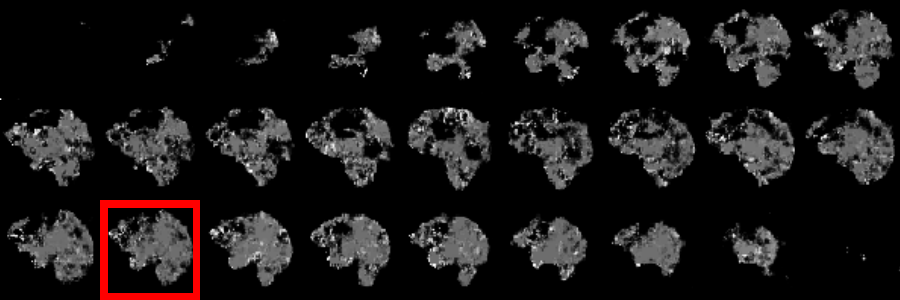}&\includegraphics[width=.29\textwidth,height=2.1cm]{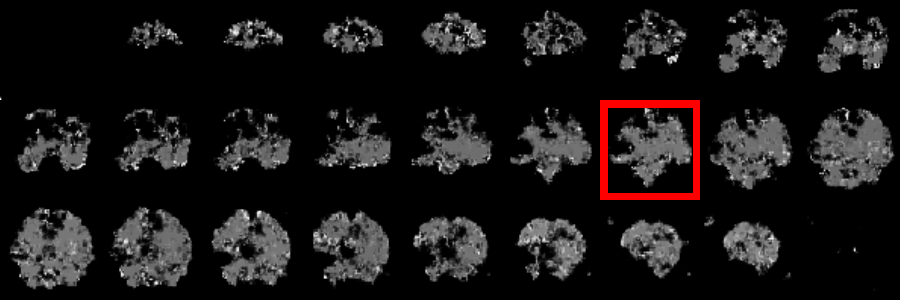}\\
\hline
\multicolumn{4}{|c|}{{\scriptsize Random Samples}}\\
\hline
\end{tabular}
\caption{27 slices of each side of a 3D sample randomly generated from BNPWMMD+DPMINE, BiGAN+MINE, and $\alpha$-WGAN+MINE after 7500 epochs for BraTS18 example.}
\label{random-slices-brat}
\end{figure*}

\begin{figure*}[h!]
\centering\begin{tabular}[h!]{|c||c|c|c|}
\hline
\multicolumn{2}{|c|}{Axial}&Sagittal &Coronal\\
\hline
\rotatebox{90}{{\scriptsize BNPWMMD+DPMINE}}&\includegraphics[width=.29\textwidth,height=2.1cm]{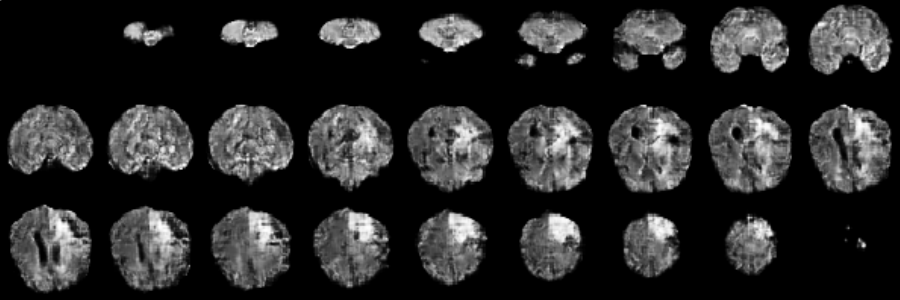}&\includegraphics[width=.29\textwidth,height=2.1cm]{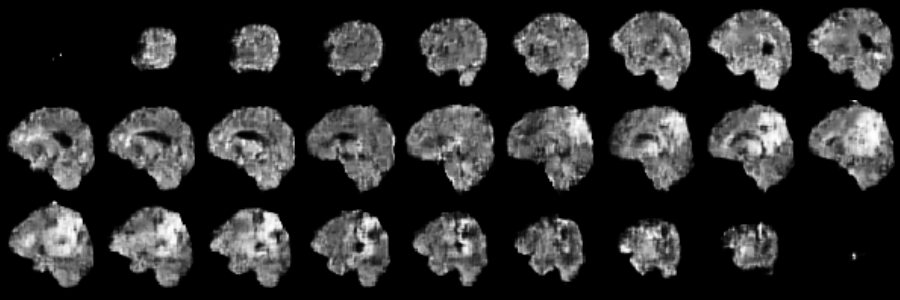}&\includegraphics[width=.29\textwidth,height=2.1cm]{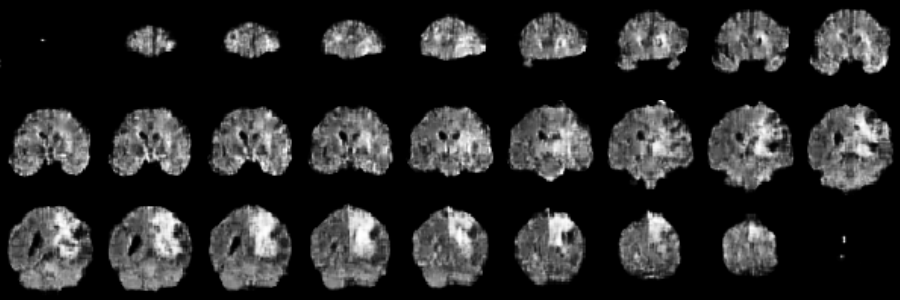}\\
\hline
\rotatebox{90}{{\scriptsize $\alpha$-WGAN+MINE}}&\includegraphics[width=.29\textwidth,height=2.1cm]{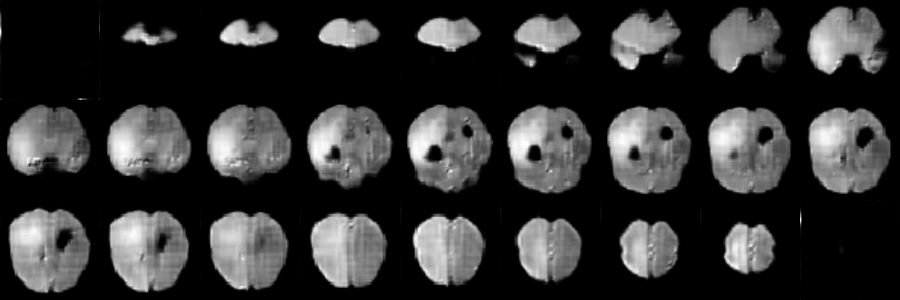}&\includegraphics[width=.29\textwidth,height=2.1cm]{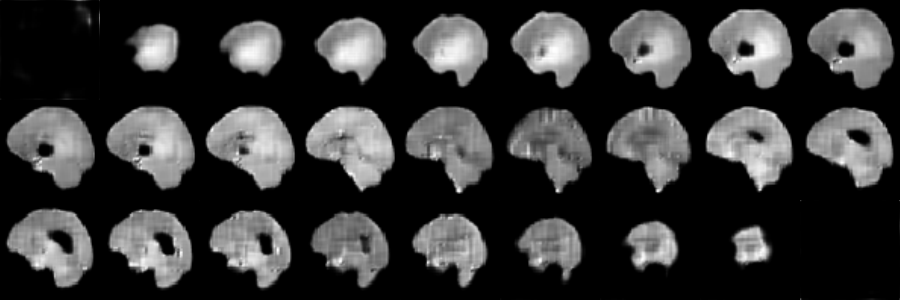}&\includegraphics[width=.29\textwidth,height=2.1cm]{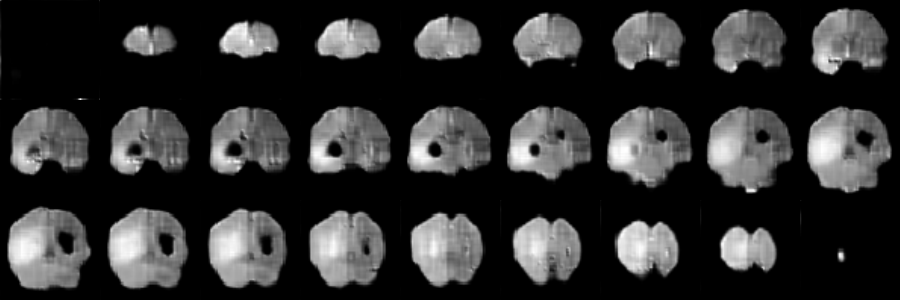}\\
\hline
\rotatebox{90}{{\scriptsize BiGAN+MINE}}&\includegraphics[width=.29\textwidth,height=2.1cm]{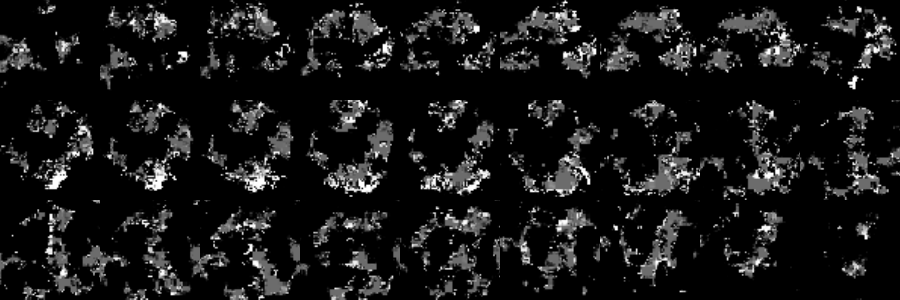}&\includegraphics[width=.29\textwidth,height=2.1cm]{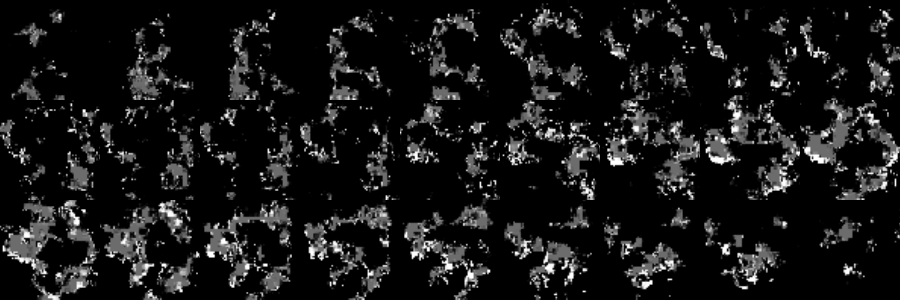}&\includegraphics[width=.29\textwidth,height=2.1cm]{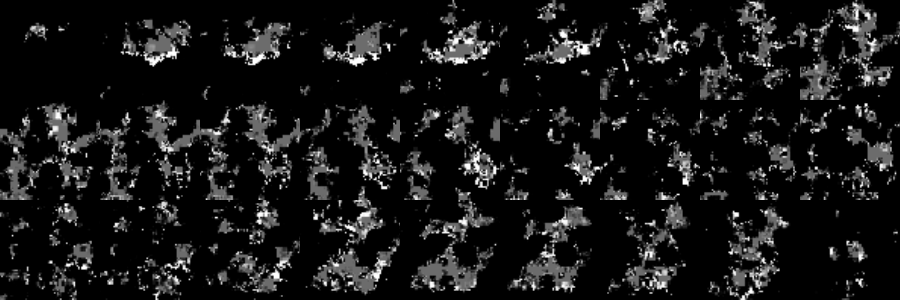}\\
\hline
\multicolumn{4}{|c|}{{\scriptsize Reconstructed Samples}}\\
\hline
\end{tabular}
\caption{27 slices of each side of a 3D sample reconstructed by BNPWMMD+DPMINE, BiGAN+MINE, and $\alpha$-WGAN+MINE after 7500 epochs for BraTS18 example.}
\label{slice-reconstruct-brat}
\end{figure*}

\begin{table}[ht]
\centering
\small
\captionof{table}{Comparison of statistical scores for BraTS18 example.}
\scalebox{0.85}{
\begin{tabular}{ l l| c |c|c }
\hline
\toprule
\multicolumn{2}{l|}{\multirow{2}{*}{Evaluator}} & BNPWMMD&$\alpha$-WGAN & BiGAN \\
&&+DPMINE&+MINE&\hspace{.25 cm}+MINE\\
\hline
\multirow{1}{*}{FID} &Custom Encoder&$\mathbf{0.00197}$ &$0.00692$ &$0.04000$\\
&t-SNE&$\mathbf{25.9317}$&$59.5596$&$5572.95$\\
\hline
\multirow{1}{*}{KID}&Custom Encoder&$\mathbf{0.00038}$ &$0.00229$ &$0.1589$\\ 
&t-SNE&$\mathbf{8.8100}$&$15.1563$&$2723.97$\\
\hline
\multirow{1}{*}{MMD} &Custom Encoder&$\mathbf{0.00075}$ &$0.00410$ &$0.02799$\\
&t-SNE&$\mathbf{0.1988}$&$0.8150$&$2.4023$\\
\hline
 \multirow{1}{*}{MS-SSIM}&&$\mathbf{0.75007}$ &$0.74799$ &$0.20349$\\
\bottomrule
\end{tabular}
}
\label{table1Brat}
\end{table}

\subsection{Implementing details}\label{app:imp-details}
\subsubsection{Calculation of evaluation scores}\label{app:scores}
Considering the real and generated features, $\boldsymbol{f}_r:=(\text{Feature1}_{1:200,r},\text{Feature2}_{1:200,r})$ and $\boldsymbol{f}_g:=(\text{Feature1}_{1:200,g},\text{Feature2}_{1:200,g})$, the FID and KID metrics are calculated using the following Eqs., respectively \cite{binkowski2018demystifying}:
\begin{small}
\begin{align*}
\text{FID}(\boldsymbol{f}_{r},\boldsymbol{f}_{g}) &= ||\boldsymbol{\mu}_{\boldsymbol{f}_{r}}-\boldsymbol{\mu}_{\boldsymbol{f}_{g}}||^2+\text{Tr}\Big(\Sigma_{\boldsymbol{f}_{r}}+\Sigma_{\boldsymbol{f}_{g}}-2\sqrt{\Sigma_{\boldsymbol{f}_{r}}\Sigma_{\boldsymbol{f}_{g}}}\Big),\\
\text{KID}(\boldsymbol{f}_{r},\boldsymbol{f}_{g}) &=\text{MMD}^2(\boldsymbol{f}_{r},\boldsymbol{f}_{g}),
\end{align*}
\end{small}
where $\boldsymbol{\mu}_{\boldsymbol{f}_r}$ and $\boldsymbol{\mu}_{\boldsymbol{f}_g}$, and $\Sigma_{\boldsymbol{f}_r}$ and $\Sigma_{\boldsymbol{g}_r}$ are the mean vector and covariance matrix of $\boldsymbol{f}_r$ and $\boldsymbol{f}_g$, respectively. The empirical MMD metric in \cite{Gretton} is used to compute $\text{MMD}^2(\cdot,\cdot)$, with a polynomial kernel $k(\boldsymbol{f}_{r},\boldsymbol{f}_{g})=(0.5\boldsymbol{f}_{r}^{\text{T}}\boldsymbol{f}_{g}+1)^{\nu}$ and degree $\nu$. The ``$\text{Tr}$'' denotes matrix trace, ``$\text{T}$'' denotes matrix transpose, and ``$||\cdot||$'' denotes Euclidean norm. We use the provided code at \url{https://torchmetrics.readthedocs.io/en/v0.8.2/image/kernel_inception_distance.html} to compute KID with $\nu=3$, and also use the provided code at \url{https://pytorch.org/ignite/generated/ignite.metrics.FID.html} to calculate FID.

Additionally, we calculate MS-SSIM using the available codes at \url{https://torchmetrics.readthedocs.io/en/v0.8.2/image/multi_scale_structural_similarity.html}.

\subsubsection{Network setting}\label{app:network}
In this section, we provide comprehensive information regarding the architectures of all networks used in the experimental results. The encoder network, generator network, and discriminator network architectures, as presented in Tables \ref{tab:encoder_network}, \ref{tab:generator_network}, and \ref{tab:discriminator_network} respectively, were implemented based on \cite{kwon2019generation} for the $\alpha$-WGAN model. The relevant codes\footnote{It is licensed under the
MIT License.} are available at \url{https://github.com/cyclomon/3dbraingen.git}. Moreover, the BNPWMMD model proposed in \cite{fazeli2023bayesian} incorporates these architectures, alongside an additional code generator outlined in Table \ref{tab:CG_network}. 
Figure \ref{DPMINE-BNPWMMD1-diagram} provides a flowchart illustrating the operation process of BNPWMMD+DPMINE.

On the other hand, we have attempted to adapt the generator and discriminator architectures given in \cite{belghazi2018mutual} for generating 3D chest CT images using BiGAN. The generator architecture can be found in Table \ref{tab:dcgan_generator_network}, and the discriminator architecture can be found in Table \ref{tab:dc_bigan_discriminator_network}. We have also considered the encoder architecture provided in Table \ref{tab:encoder_network} for BiGAN. Additionally, Table \ref{tab:T_network_architecture} presents a simple architecture used for updating the parameters $\{T_{\boldsymbol{\gamma}}\}_{\boldsymbol{\gamma}\in\boldsymbol{\Gamma}}$ in the DPMINE and MINE calculations\cite{belghazi2018mutual}. The basic codes\footnote{It is licensed under the
MIT License.} of BiGAN+MINE can be found in \url{https://github.com/gtegner/mine-pytorch.git}.

All architectures have been implemented in PyTorch using the Adam optimizer and a learning rate of $0.0002$ on an NVIDIA Tesla V100-SXM2 with 4 GPUs, each with 32GB of RAM. Our code requires 48-72 hours to provide results for the COVID-19 example. For the coil experiments, our code takes about 4-6 hours.

\begin{figure*}[!t]
\centering
\includegraphics[width=.8\textwidth,height=6.5cm]{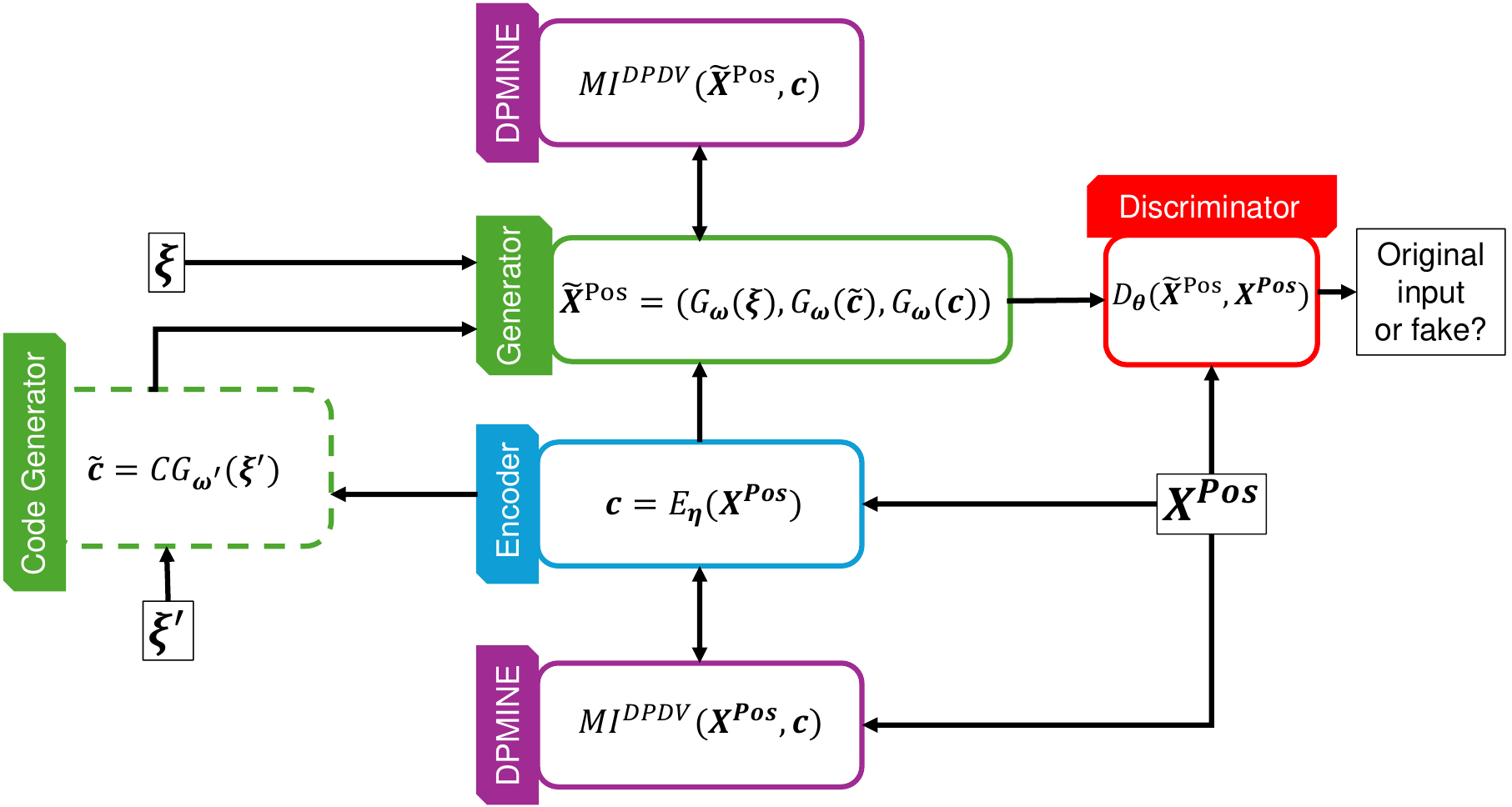}
\caption{A general flowchart to illustrate the actual operation process of BNPWMMD+DPMINE.}
\label{DPMINE-BNPWMMD1-diagram}
\end{figure*}
\begin{table*}[ht]
\centering
\caption{Encoder: A 3D Convolutional Network Architecture for the COVID-19 Dataset used in BNPWMMD-GAN and $\alpha$-WGAN.}
\label{tab:encoder_network}
\small
\setlength{\tabcolsep}{.5pt}
\begin{tabular}{|l|c|c|c|c|c|c|}
\hline
\multirow{2}{*}{\textbf{Layer}} & \textbf{Input} & \textbf{Output} & \textbf{Kernel} & \multirow{2}{*}{\textbf{Stride}} & \multirow{2}{*}{\textbf{Padding}} & \textbf{Activation }\\
&\textbf{dimension}&\textbf{dimension}&\textbf{Size}&&&\textbf{Function}
\\
\hline
\multirow{2}{*}{Convolution} & $1\times64\times64\times64$ & \multirow{2}{*}{$64\times32\times32\times32$} &\multirow{2}{*}{$ 4\times4\times4$} & \multirow{2}{*}{$2$} & \multirow{2}{*}{1} & \multirow{2}{*}{Leaky ReLU} \\
&(Data dimension)&&&&&(negative slope=0.2)\\
\hline
\multirow{2}{*}{Convolution} & \multirow{2}{*}{$64\times32\times32\times32$} & \multirow{2}{*}{$128\times16\times16\times16$} & \multirow{2}{*}{$4\times4\times4$} & \multirow{2}{*}{$2$} & \multirow{2}{*}{1} & Leaky ReLU  \\
&&&&&&(negative slope=0.2)\\
\hline
BatchNorm & $128\times16\times16\times16$ & $128\times16\times16\times16$ & - & - & - & - \\
\hline
\multirow{2}{*}{Convolution} & \multirow{2}{*}{$128\times16\times16\times16$} & \multirow{2}{*}{$256\times8\times8\times8$} & \multirow{2}{*}{$4\times4\times4$} & \multirow{2}{*}{$2$} & \multirow{2}{*}{1} & Leaky ReLU  \\
&&&&&&(negative slope=0.2)\\
\hline
BatchNorm & $256\times8\times8\times8$ & $256\times8\times8\times8$ & - & - & - & - \\
\hline
\multirow{2}{*}{Convolution} & \multirow{2}{*}{$256\times8\times8\times8$} & \multirow{2}{*}{$512\times4\times4\times4$} & \multirow{2}{*}{$4\times4\times4$} & \multirow{2}{*}{$2$} & \multirow{2}{*}{1} & Leaky ReLU  \\
&&&&&&(negative slope=0.2)\\
\hline
BatchNorm & $512\times4\times4\times4$ & $512\times4\times4\times4$ & - & - & - & - \\
\hline
Convolution & $512\times4\times4\times4$ & $1000\times1\times1\times1$ & $4\times4\times4$ & $2$ & 1 & - \\
\hline
\end{tabular}
\end{table*}

\begin{table*}[ht]
\centering
\caption{Code-Generator: A 3D Convolutional Network Architecture for the COVID-19 Dataset used in BNPWMMD-GAN and $\alpha$-WGAN.}
\label{tab:CG_network}
\small
\setlength{\tabcolsep}{.5pt}
\begin{tabular}{|l|c|c|c|c|c|c|}
\hline
\multirow{2}{*}{\textbf{Layer}} & \textbf{Input} & \textbf{Output} & \textbf{Kernel} & \multirow{2}{*}{\textbf{Stride}} & \multirow{2}{*}{\textbf{Padding}} & \textbf{Activation }\\
&\textbf{dimension}&\textbf{dimension}&\textbf{Size}&&&\textbf{Function}
\\
\hline
\multirow{3}{*}{Convolution} & $100$ & \multirow{3}{*}{$16\times2\times5\times10$} &\multirow{3}{*}{$ 3\times3\times3$} & \multirow{3}{*}{$1$} & \multirow{3}{*}{1} & \multirow{3}{*}{ReLU} \\
&(Sub-latent&&&&&\\
& dimension)&&&&&\\
\hline
BatchNorm2 & $16\times2\times5\times10$ & $16\times2\times5\times10$ & - & - & - & - \\
\hline
Max-pooling & $16\times2\times5\times10$ & $16\times2\times5\times5$ & $1\times1\times2$ & $1\times1\times2$& -  & - \\
\hline
Convolution & $16\times2\times5\times5$ & $32\times2\times5\times5$ & $3\times3\times3$ & $1$ & 1 & ReLU \\
\hline
BatchNorm2 & $32\times2\times5\times5$ & $32\times2\times5\times5$ & - & - & - & - \\
\hline
Max-pooling & $32\times2\times5\times5$ & $32\times2\times5\times2$  &$1\times1\times2$ & $1\times1\times2$& -  & - \\
\hline
Convolution & $32\times2\times5\times2$ & $64\times2\times5\times2$ & $3\times3\times3$ & $1$ & 1 & ReLU \\
\hline
BatchNorm3 & $64\times2\times5\times2$ & $64\times2\times5\times2$ & - & - & - & - \\
\hline
Max-pooling & $64\times2\times5\times2$ & $64\times2\times5\times1$  & $1\times1\times2$ & $1\times1\times2$& - & - \\
\hline
Flatten&$64\times2\times5\times1$&$640$&-&-&-&-\\
\hline
Fully-connected & $640$ & $1000$ & - & - & - & - \\
\hline
\end{tabular}
\end{table*}
\begin{table*}[ht]
\centering
\caption{Generator: A 3D Convolutional Network Architecture for the COVID-19 Dataset used in BNPWMMD-GAN and $\alpha$-WGAN.}
\label{tab:generator_network}
\small
\setlength{\tabcolsep}{1pt}
\begin{tabular}{|l|c|c|c|c|c|c|}
\hline
\multirow{2}{*}{\textbf{Layer}} & \textbf{Input} & \textbf{Output} & \textbf{Kernel} & \multirow{2}{*}{\textbf{Stride}} & \multirow{2}{*}{\textbf{Padding}} & \textbf{Activation }\\
&\textbf{dimension}&\textbf{dimension}&\textbf{Size}&&&\textbf{Function}
\\
\hline
\multirow{2}{*}{Transposed} & 1000 & \multirow{2}{*}{$512\times4\times4\times4$} & \multirow{2}{*}{$4\times4\times4$} & \multirow{2}{*}{1} & \multirow{2}{*}{0} & \multirow{2}{*}{ReLU} \\
convolution&(Latent dimension)&&&&&\\
\hline
BatchNorm & $512\times4\times4\times4$ & $512\times4\times4\times4$ & - & - & - & - \\
\hline
Upscale&$512\times4\times4\times4$&$512\times4\times4\times4$&-&-&-&-\\
\hline
Convolution & $512\times4\times4\times4$ & $256\times8\times8\times8$ & $3\times3\times3$ & 1 & 1 & ReLU \\
\hline
BatchNorm & $256\times8\times8\times8$ & $256\times8\times8\times8$ & - & - & - & - \\
\hline
Upscale&$256\times8\times8\times8$&$256\times8\times8\times8$&-&-&-&-\\
\hline
Convolution & $256\times8\times8\times8$ & $128\times16\times16\times16$ & $3\times3\times3$ & 1 & 1 & ReLU \\
\hline
BatchNorm & $128\times16\times16\times16$ & $128\times16\times16\times16$ & - & - & - & - \\
\hline
Upscale&$128\times16\times16\times16$&$128\times16\times16\times16$&-&-&-&-\\
\hline
Convolution & $128\times16\times16\times16$ & $64\times32\times32\times32$ & $3\times3\times3$ & 1 & 1 & ReLU \\
\hline
BatchNorm & $64\times32\times32\times32$ & $64\times32\times32\times32$ & - & - & - & - \\
\hline
Upscale&$64\times32\times32\times32$&$64\times32\times32\times32$&-&-&-&-\\
\hline
Convolution & $64\times32\times32\times32$ & $1\times64\times64\times64$ & $3\times3\times3$ & 1 & 1 & Tanh \\
\hline
\end{tabular}
\end{table*}

\begin{table*}[ht]
\centering
\caption{Discriminator: A 3D Convolutional Network Architecture for the COVID-19 Dataset used in BNPWMMD-GAN and $\alpha$-WGAN.}
\label{tab:discriminator_network}
\small
\setlength{\tabcolsep}{.5pt}
\begin{tabular}{|l|c|c|c|c|c|c|}
\hline
\multirow{2}{*}{\textbf{Layer}} & \textbf{Input} & \textbf{Output} & \textbf{Kernel} & \multirow{2}{*}{\textbf{Stride}} & \multirow{2}{*}{\textbf{Padding}} & \textbf{Activation }\\
&\textbf{dimension}&\textbf{dimension}&\textbf{Size}&&&\textbf{Function}
\\
\hline
\multirow{2}{*}{Convolution} & $1\times64\times64\times64$ & \multirow{2}{*}{$64\times32\times32\times32$} &\multirow{2}{*}{$ 4\times4\times4$} & \multirow{2}{*}{$2$} & \multirow{2}{*}{1} & \multirow{2}{*}{Leaky ReLU } \\
&(Data dimension)&&&&&(negative slope=0.2)\\
\hline
\multirow{2}{*}{Convolution} & \multirow{2}{*}{$64\times32\times32\times32$} & \multirow{2}{*}{$128\times16\times16\times16$} & \multirow{2}{*}{$4\times4\times4$} & \multirow{2}{*}{$2$} & \multirow{2}{*}{1} & Leaky ReLU \\
&&&&&&(negative slope=0.2)\\
\hline
BatchNorm & $128\times16\times16\times16$ & $128\times16\times16\times16$ & - & - & - & - \\
\hline
\multirow{2}{*}{Convolution} & \multirow{2}{*}{$128\times16\times16\times16$} & \multirow{2}{*}{$256\times8\times8\times8$} & \multirow{2}{*}{$4\times4\times4$} & \multirow{2}{*}{$2$} & \multirow{2}{*}{1} & Leaky ReLU \\
&&&&&&(negative slope=0.2)\\
\hline
BatchNorm & $256\times8\times8\times8$ & $256\times8\times8\times8$ & - & - & - & - \\
\hline
\multirow{2}{*}{Convolution} & \multirow{2}{*}{$256\times8\times8\times8$} & \multirow{2}{*}{$512\times4\times4\times4$} & \multirow{2}{*}{$4\times4\times4$} & \multirow{2}{*}{$2$} & \multirow{2}{*}{1} & Leaky ReLU \\
&&&&&&(negative slope=0.2)\\
\hline
BatchNorm & $512\times4\times4\times4$ & $512\times4\times4\times4$ & - & - & - & - \\
\hline
Convolution & $512\times4\times4\times4$ & $1\times1\times1\times1$ & $4\times4\times4$ & $2$ & 1 & Sigmoid \\
\hline
\end{tabular}
\end{table*}
\begin{table*}[ht]
\centering
\caption{Generator: A 3D Convolutional Network Architecture for the COVID-19 Dataset used in BiGAN.}
\label{tab:dcgan_generator_network}
\small
\setlength{\tabcolsep}{1pt}
\begin{tabular}{|l|c|c|c|c|c|c|}
\hline
\multirow{2}{*}{\textbf{Layer}} & \textbf{Input} & \textbf{Output} & \textbf{Kernel} & \multirow{2}{*}{\textbf{Stride}} & \multirow{2}{*}{\textbf{Padding}} & \textbf{Activation }\\
&\textbf{dimension}&\textbf{dimension}&\textbf{Size}&&&\textbf{Function}
\\
\hline
\multirow{2}{*}{Fully-connected} & 1000 & \multirow{2}{*}{$4\times4\times4\times512$} & \multirow{2}{*}{-} & \multirow{2}{*}{-} & \multirow{2}{*}{-} & \multirow{2}{*}{-} \\
&(latent dimension)&&&&&\\
\hline
Transposed  & \multirow{2}{*}{$512\times4\times4\times4$} & \multirow{2}{*}{$256\times8\times8\times8$} & \multirow{2}{*}{$4\times4\times4$} & \multirow{2}{*}{2} & \multirow{2}{*}{1} & \multirow{2}{*}{ReLU} \\
convolution&&&&&&\\
\hline
Transposed  & \multirow{2}{*}{$256\times8\times8\times8$} & \multirow{2}{*}{$128\times16\times16\times16$} & \multirow{2}{*}{$4\times4\times4$} & \multirow{2}{*}{2} & \multirow{2}{*}{1} & \multirow{2}{*}{ReLU} \\
convolution&&&&&&\\
\hline
Transposed  & \multirow{2}{*}{$128\times16\times16\times16$} & \multirow{2}{*}{$64\times32\times32\times32$} & \multirow{2}{*}{$4\times4\times4$} & \multirow{2}{*}{2} & \multirow{2}{*}{1} & \multirow{2}{*}{ReLU} \\
convolution&&&&&&\\
\hline
Transposed  & \multirow{2}{*}{$64\times32\times32\times32$} & \multirow{2}{*}{$1\times64\times64\times64$} & \multirow{2}{*}{$4\times4\times4$} & \multirow{2}{*}{2} & \multirow{2}{*}{1} & \multirow{2}{*}{Tanh} \\
convolution&&&&&&\\
\hline
\end{tabular}
\end{table*}

\begin{table*}[ht]
\centering
\caption{Discriminator: A 3D Convolutional Network Architecture for the COVID-19 Dataset used in BiGAN.}
\label{tab:dc_bigan_discriminator_network}
\small
\setlength{\tabcolsep}{1pt}
\begin{tabular}{|l|c|c|c|c|c|c|}
\hline
\multirow{2}{*}{\textbf{Layer}} & \textbf{Input} & \textbf{Output} & \textbf{Kernel} & \multirow{2}{*}{\textbf{Stride}} & \multirow{2}{*}{\textbf{Padding}} & \textbf{Activation }\\
&\textbf{dimension}&\textbf{dimension}&\textbf{Size}&&&\textbf{Function}
\\
\hline
\multirow{2}{*}{Convolution} & $1\times64\times64\times64$ & \multirow{2}{*}{$64\times32\times32\times32$} & \multirow{2}{*}{$5\times5\times5$} & \multirow{2}{*}{2} & \multirow{2}{*}{2} & \multirow{2}{*}{LeakyReLU} \\
&(Data dimension)&&&&&\\
\hline
Convolution & $64\times32\times32\times32$ & $128\times16\times16\times16$ & $5\times5\times5$ & 2 & 2 & LeakyReLU \\
\hline
Convolution & $128\times16\times16\times16$ & $256\times8\times8\times8$ & $5\times5\times5$ & 2 & 2 & LeakyReLU \\
\hline
Flatten & $256\times8\times8\times8$ & $131072$ & - & - & - & - \\
\hline
Concatenate  & $1000+131072$ & $132072$ & - & - & - & - \\
latent-variable&&&&&&\\
\hline
Fully-connected & $132072$ & $1024$ & - & - & - & - \\
\hline
Fully-connected & $1024$ & $1$ & - & - & - & Sigmoid \\
\hline
\end{tabular}
\end{table*}


\begin{table*}[ht]
\centering
\caption{$T_{\boldsymbol{\gamma}}$-network Architecture used in DPMINE and MINE computation.}
\label{tab:T_network_architecture}
\small
\setlength{\tabcolsep}{1pt}
\begin{tabular}{|l|c|c|c|}
\hline
\multirow{2}{*}{\textbf{Layer}} & \multirow{2}{*}{\textbf{Input Dimension}} & \multirow{2}{*}{\textbf{Output Dimension}} & \textbf{Activation} \\
&&&\textbf{ Function}
\\
\hline
\multirow{2}{*}{concatenation}
 & $1000+64^3$ & \multirow{2}{*}{$263144$} & \multirow{2}{*}{-} \\
 &(Latent dimension + Data dimension)&&\\
\hline
Linear & $263144$ & $400$ & ReLU \\
\hline
Linear & $400$ & $400$ & ReLU \\
\hline
Linear & $400$ & $400$ & ReLU \\
\hline
Linear & $400$ & $1$ & ReLU \\
\hline
\end{tabular}
\end{table*}

\section{Limitations}\label{app:limitation}
The BNP generative model proposed in this paper assumes that the training data follows the IID assumption. However, this assumption restricts the model's applicability to certain privacy-preserving techniques, such as federated learning, which often rely on non-IID assumptions at the local device level. While this limitation is acknowledged, it is beyond the scope of the current work to address it. In future research, we plan to extend the proposed BNP model to incorporate non-IID assumptions and apply it to federated learning. This extension aims to mitigate issues such as slow convergence and unfair predictions that may arise when applying a global model to unseen data. 

\section{Broad impacts and safeguards}\label{app:Broad-Safe}
One potential positive impact of this paper is the advancement of artificial intelligence (AI) techniques in healthcare. The proposed model offers a solution to the challenge of accessing high-quality, diverse medical datasets for clinical decision-making. This can lead to improved accuracy and effectiveness in diagnosing and treating diseases, ultimately benefiting patients and healthcare providers. Additionally, the incorporation of a BNP procedure helps uncover underlying structures in the data and reduce overfitting, enhancing the reliability of the generated samples.

However, the use of AI in healthcare also raises concerns about data privacy and security. It is important to prioritize patient privacy and data protection by implementing strict regulations and protocols to ensure that patient data is anonymized and protected. Additionally, there is a risk of bias in the generated samples, which could lead to disparities in healthcare outcomes if the models are not properly validated and tested on diverse populations. Therefore, transparency and accountability in the development and deployment of these AI models are necessary. Rigorous testing and validation on diverse datasets should be conducted to mitigate bias and ensure fairness. Collaboration between healthcare professionals, AI researchers, and policymakers is essential to establish guidelines and regulations that prioritize patient privacy, data protection, and equitable healthcare outcomes.

{\noindent \em Remainder omitted in this sample. See http://www.jmlr.org/papers/ for full paper.}

\vskip 0.2in
\bibliography{Bibliography-MM-MC}

\end{document}